\documentclass[twoside,11pt]{article}

\usepackage{amsmath}
\usepackage{amssymb}
\usepackage{amsthm}
\usepackage{indentfirst}
\usepackage[hidelinks]{hyperref}
\usepackage{paralist}
\usepackage{geometry}
\usepackage{setspace,lipsum}
\usepackage[title]{appendix}
\usepackage{color}
\usepackage{multimedia}
\usepackage{multirow}
\usepackage{bbm}
\usepackage{graphics}
\usepackage{graphicx}
\usepackage{booktabs}
\usepackage{romannum}
\usepackage{grffile}
\usepackage{float}
\usepackage{url}
\usepackage{mathabx}
\usepackage{supertabular}
\usepackage{rotating}
\usepackage{pdflscape}
\usepackage{verbatim}
\usepackage{listings}
\usepackage{xcolor}
\usepackage{multicol}
\usepackage[shortlabels]{enumitem} 
\usepackage{caption,subcaption}
\usepackage{epstopdf}
\usepackage{algorithm}%
\usepackage{algpseudocode}
\usepackage{algorithmicx}%
\usepackage[round]{natbib}
\usepackage{authblk}
\usepackage{etoolbox}
\graphicspath{{figure/}}
\newcommand{\bm}[1]{\boldsymbol{#1}}
\newcommand{\dd}{\mathrm{d}}

\newcommand{\E}{\mathbb{E}}
\newcommand{\p}{\mathbb{P}}
\newcommand{\q}{\mathbb{Q}}

\newcommand{\f}{\mathcal{F}}

\newtheorem{assumption}{Assumption}
\newtheorem{eg}{Example}
\newtheorem{proposition}{Proposition}
\newtheorem{theorem}{Theorem}
\newtheorem{lemma}{Lemma}
\newtheorem{definition}{Definition}

\begin{document}
\pagenumbering{arabic}
\providecommand{\keywords}[1]
{
\small	
\textbf{\textit{Keywords:}} #1
}
\providecommand{\jel}[1]
{
\small	
\textbf{\textit{JEL Classification:}} #1
}

\title{Continuous-time Risk-sensitive Reinforcement Learning via Quadratic Variation Penalty}

\author{Yanwei Jia\thanks{Department of Systems Engineering and Engineering Management, The Chinese University of Hong Kong, Shatin, New Territories, Hong Kong. Email: yanweijia@cuhk.edu.hk.} }

\maketitle
\begin{abstract}
\singlespacing
This paper studies continuous-time risk-sensitive reinforcement learning (RL) under the entropy-regularized, exploratory diffusion process formulation with the exponential-form objective. The risk-sensitive objective arises either as the agent's risk attitude or as a distributionally robust approach against the model uncertainty. Owing to the martingale perspective in Jia and Zhou (J Mach Learn Res 24(161):1–61, 2023), the risk-sensitive RL problem is shown to be equivalent to ensuring the martingale property of a process involving both the value function and the q-function, augmented by an additional penalty term: the quadratic variation of the value process, capturing the variability of the value-to-go along the trajectory. This characterization allows for the straightforward adaptation of existing RL algorithms developed for non-risk-sensitive scenarios to incorporate risk sensitivity by adding the realized variance of the value process. Additionally, I highlight that the conventional policy gradient representation is inadequate for risk-sensitive problems due to the nonlinear nature of quadratic variation; however, q-learning offers a solution and extends to infinite horizon settings. Finally, I prove the convergence of the proposed algorithm for Merton's investment problem and quantify the impact of temperature parameter on the behavior of the learning procedure. I also conduct simulation experiments to demonstrate how risk-sensitive RL improves the finite-sample performance in the linear-quadratic control problem. 
\end{abstract}

\keywords{risk-sensitive control, continuous-time reinforcement learning, exponential martingale, quadratic variation penalty, q-learning}

\section{Introduction}
\label{sec1}
Continuous-time reinforcement learning (RL) has drawn considerable attention owing to its practical significance in modeling systems that necessitate or benefit from high-frequency or real-time interaction with the environment, such as in financial trading, real-time response systems, and robotics. Furthermore, its theoretical development serves to bridge the gap between the conventional stochastic control and the discrete-time, Markov decision process (MDP)-based RL theory.

Existing literature on RL largely focuses on optimizing the expectation of a flow of time-additive rewards. Here, the instantaneous reward can be interpreted as a utility within the time-additive von Neumann–Morgenstern expected utility framework. Though utility functions can reflect the agent's attitude towards risk in a one-off setting, economic literature has long recognized that such an additive form inadequately captures intertemporal preferences on the uncertainty \citep{epstein1989substitution}. Moreover, empirically measuring and incorporating an agent's risk preference via a simple additive utility function is notoriously difficult.  

An alternative approach is the risk-sensitive control (which can be traced back to \citep{jacobson1973optimal}), and it later becomes popular, particularly in financial asset management, e.g., in \citet{bielecki1999risk}. In contrast to solely considering the expectation, risk-sensitive objective accounts for the whole distribution of the accumulated reward. 
Moreover, the risk-sensitive objective function in the exponential form is well-known to be closely related to the robustness within a family of distributions measured by the Kullback–Leibler (KL) divergence, which is also known as the robust control problems \citep{hansen2001robust}. Such uncertainty on the distribution of a random variable (not just its realization whose uncertainty can be statistically quantified) is often regarded as the Knightian uncertainty or ambiguity, which often occurs due to the lack of knowledge or historical data \citep{leroy1987knight}. From this perspective, an RL agent often encounters the similar situation, in which the agent lacks information about the environment, and hence, has difficulty formulating a probabilistic model to quantify the associated risk. Therefore, it is natural to consider the risk sensitivity in the RL setting. 

In this paper, I consider the risk-sensitive objective in the exponential form that is used in \citet{bielecki1999risk} and study this problem from RL perspective, i.e., in a data-driven and model-free (up to a controlled diffusion process) approach. Specifically, I adopt the entropy-regularized continuous-time RL framework proposed in \citet{wang2020reinforcement} and aim to find a stochastic policy that maximizes the entropy-regularized risk-sensitive objective. It is noticeable that, in this paper, the entropy regularization term is added inside the exponential form. This is motivated by regarding the regularizer as an extra source of reward for exploration, and hence, it should be treated similarly as the reward. Such choice is also made in \cite{enders2024risk}, who numerically document the improvement over the counterpart without entropy regularization. However, the benefits of having entropy regularization have not been theoretically investigated, even for simple cases. 

The primary contribution of this paper lies in the establishment of the q-learning theory for the continuous-time risk-sensitive RL problems. The conventional Q-function for the discrete-time, MDP-based risk-sensitive RL often relies on nonlinear recursive relations, such as ``exponential Bellman equation'' \citep{fei2021exponential} or the ``distributional robust Bellman equation'' \citep{wang2023finite}. However, in continuous-time scenarios, the complicated structure gets simplified and clarified. This simplification stems from two key observations: firstly, a risk-sensitive control problem can be equivalently transformed into its non-risk-sensitive counterpart augmented with a quadratic variation (QV) penalty \citep{skiadas2003robust}; secondly, the conventional Q-function has to be properly decomposed and rescaled into the value function and q-function in continuous time \citep{jia2022q}, and the latter does not have nonlinear effect. In particular, the definition and characterization of the risk-sensitive q-function is almost parallel to that for the non-risk-sensitive q-function established in \citet{jia2022q}, and they only differ by an extra term involving the QV of the value-to-go process (the value function applying on the state variables). Consequently, explicit computation of exponential forms becomes unnecessary, and the derived martingale condition is linear in the risk-sensitive q-function, albeit still nonlinear in the value function. This linearity facilitates the application of algorithms aimed at enforcing the martingale property of a process, as discussed in \citet{jia2022policy}, to risk-sensitive RL problems.

As a side product, I show that in the risk-sensitive problems, the relation between q-learning and policy gradient established in \cite{schulman2017equivalence} and \cite{jia2022q} for the non-risk-sensitive RL does not hold anymore. The reason why this relation fails is because the associated Bellman-type (Feynman-Kac-type) equation is no longer linear in the value function \citep{nagai1996bellman}. This observation adds to one of the advantages of q-learning over policy gradient methods. In addition, I show how the risk-sensitive q-learning theory for finite-horizon episodic tasks can be extended to infinite-horizon ergodic tasks, thereby encompassing the original problem formulation in \citet{bielecki1999risk}. 

The second contribution is to analyze the proposed RL algorithm in Merton's investment problem with power utility, which can be viewed as the risk-sensitive objective of the log-return of the portfolio. This problem has been solved only in theory, in which the model for stock price is given and known. By contrast, I study how an agent can learn the policy with no prior knowledge. Specifically, I investigate the role of the temperature parameter in the entropy-regularized RL.  While many have suggested intuitively that the temperature parameter governs the tradeoff between exploitation and exploration, its algorithmic impact has yet to be formally studied. To the author's best knowledge, there are no clear guidelines on how to choose the temperature parameter endogenously, except for some theoretical results on the convergence or asymptotic expansion of the exploratory problem to the classical problem in the non-risk-sensitive setting, e.g., \cite{wang2020reinforcement,tang2021exploratory,dai2023recursive}. These papers all consider the difference between the total reward under the optimal deterministic/stochastic policy by ignoring the estimation error of optimal policy incurred in a specific learning procedure. Therefore, from their analysis, it seems that randomization and entropy regularization always cause inefficiency but have no clear benefits.

I highlight that the virtue of entropy regularization lies in the algorithmic aspect via boosting the estimation accuracy of the optimal policy. Within the framework of stochastic approximation algorithms, this study reveals that the temperature parameter functions analogously to the learning rate: Higher temperatures correspond to faster learning because data collected spans larger space and contains more information (interpreted as exploration) but also entail higher noise in the data and lower expected reward (due to a lack of exploitation). The convergence and convergence rate are determined by the combined schedule of the step size and the temperature parameter in each iteration. I give two possible conditions to ensure the proposed learning algorithm converges at the optimal rate that matches the state-of-the-art in terms of the number of episodes.

Furthermore, I conduct another numerical study to examine the choice of the risk sensitivity coefficient when the sample size is finite in a off-policy learning task for the linear-quadratic problem (with both drift and volatility control). The results confirm the intuition about the inherent robustness of the risk-sensitive RL. It is demonstrated that employing risk-sensitive RL with finite datasets yields reduced estimation errors for the optimal policy compared to non-risk-sensitive approaches. Moreover, the optimal choice of the risk sensitivity coefficient diminishes as the sample size increases, suggesting a delicate relationship between risk sensitivity and data availability.
\subsection*{Related literature}

Previous literature on risk-sensitive control problems in continuous time studies their theoretical properties and economic implications, for example, the well-posedness \citep{fleming1995risk,nagai1996bellman,dupuis2000robust,fleming2002risk,menaldi2005remarks}, and how risk sensitivity explains the equity premium puzzle \citep{maenhout2004robust,glasserman2013robust}. However, in the conventional paradigm, agents have risk-sensitive objectives, but they know a benchmark distribution (or a \textit{model}) and also know the forms of the reward functions. RL perspective contrasts the conventional settings in not knowing a benchmark distribution (i.e., being \textit{model-free}), and even not the forms of reward functions. The important observation on the link between a risk-sensitive problem and the QV penalty is made by \citet{skiadas2003robust}. Thus, a risk-sensitive RL problem can also be transformed to an ordinary, non-risk-sensitive RL problem plus an extra QV penalty on the value function along the trajectory. Therefore, it becomes a recursive utility maximization problem \citep{duffie1992stochastic}, and hence, the martingale optimality principle still applies to characterize the optimal policy.

The continuous-time RL framework adopted in this paper is developed\footnote{It is worthwhile pointing out that prior literature on continuous-time RL has been largely restricted to deterministic system, such as \cite{doya2000reinforcement,lee2021policy,kim2021hamilton} and many others. To the author's knowledge, \cite{wang2020reinforcement} is the first paper to approach stochastic control problems from an RL perspective rather than a computational method.} in \citet{wang2020reinforcement}, who introduce stochastic policies and entropy regularization to the standard stochastic control problems to formulate RL in continuous-time, and to embed RL into relaxed control problems for theoretical analysis. Moving from the theoretical framework to general principles to devise data-driven RL algorithms, \citet{jia2022policy} propose the unifying martingale perspective in the policy evaluation stage and suggest the loss function and orthogonality conditions to approximate the martingale properties. \citet{jia2022policypg} further derive the representation for the policy gradient for stochastic policies. Moreover, \citet{jia2022q} clarify the notion of ``q-function'' in the continuous-time system, and reveal the relation between the q-function in RL and Hamiltonian in control literature. In particular, \citet{jia2022q} show that the q-function and the value function can be learned jointly via martingale conditions, thus, the principles proposed in \cite{jia2022policy} can be applied to devise q-learning algorithms. Based on this framework, there have been extensive studies and applications, such as \cite{GuoXZ2020,frikha2023actor,wei2023continuous} for mean-field interactions, \cite{dai2023learning} for time-inconsistency, \cite{wang2020continuous,szpruch2022optimal} for linear-quadratic controls, and \cite{wang2023reinforcement} for optimal execution. However, all the above has been done only for the ordinary, non-risk-sensitive RL. This paper focuses on the risk-sensitive objective, to see how the martingale perspective can be adopted and applied here and how to design algorithms based on it.


The discrete-time counterparts (by replacing integrals with sums, and replacing controlled diffusion processes by MDPs, but without entropy regularization) are also popular in discrete-time robust RL, such as \cite{borkar2002q,fei2020risk,fei2021exponential,wang2023finite,wenhaoregret}. On the other hand, the entropy-regularized RL, such as soft-Q-learning \citep{schulman2017equivalence} and soft-actor-critic \citep{haarnoja2018soft}, and also has been integrated to continuous-time RL by \cite{wang2020reinforcement,GuoXZ2020,jia2022policypg,jia2022q}, and many others. The combination of risk-sensitive RL and entropy regularization has been adopted in discrete-time setting in \cite{enders2024risk}, but is novel to continuous-time setting. In the discrete-time, the analysis has been largely restricted to tabular cases \citep{borkar2002q,fei2020risk,fei2021exponential,wang2023finite}, and their approaches are via ``exponential Bellman equation'' or ``distributional robust Bellman equation'', both of which involve nonlinear recursive relation. The best regret bound in terms of the number of episodes is at the square-root order in general, ignoring the logarithm factor. But these works do not include entropy regularizations. By considering the problem in continuous-time, the resulting martingale condition is still nonlinear in the value function while becomes linear in q-function. Numerically, it avoids computation of exponential forms and reduces the sampling noises. Intuitively speaking, this simplification is due to the validity of the normal approximation for the state transition. In the Merton's investment problem, the same convergence rate in terms of the number of episodes, matching the state-of-the-art, can also be proved when the underlying stock price follows the geometric Brownian motion.
	
The other example considered in this paper is the off-policy linear-quadratic control. The linear-quadratic control problems have been widely studied in the RL context, in both discrete-time (e.g., \cite{hambly2021policy,fazel2018global,wang2021global,abbasi2011regret}) and continuous-time (e.g., \cite{basei2020logarithmic,guo2023reinforcement,giegrich2024convergence,szpruch2022optimal}). All these literature studies non-risk-sensitive problems. The only exception is \cite{wenhaoregret}, who also study risk-sensitive problem but focus on online, on-policy learning by developing a model-based, least-square-type algorithm. The example considered in this paper has three distinctions compared to previous literature: first, the volatility of the noise in the system is controlled; second, it is the first paper considering the continuous-time risk-sensitive RL, and an entropy regularizer is included in the objective function; third, the learning is conducted off-policy. The formulation and the risk-sensitive q-learning theory introduced in this paper gives full flexibility and account for all these unconventional features compared to what has been studied in the literature, and provides the foundation for devising model-free learning algorithms.

\subsection*{Structure of the paper}
The rest of the paper is organized as follows. Section \ref{sec:setup} describes the problem setup and motivates this formulation. Section \ref{sec:risk sensitive q learning} is devoted to introducing the definition of q-function for risk-sensitive problems and showing the martingale characterization of the optimal q-function and value function. The difference between q-learning and policy gradient, and the extension to ergodic problems are presented in Section \ref{sec:extensions}. I demonstrate the performance of the risk-sensitive q-learning algorithm on two applications in Section \ref{sec:numerical}. In particular, the theoretical guarantee for Merton's investment problem is shown in Section \ref{subsec:meton convergence results}. Finally, Section \ref{sec:conclusion} concludes. All proofs are in the Appendix.  
\section{Problem Formulation}
\label{sec:setup}

Throughout this paper, by convention all vectors are {\it column} vectors unless otherwise specified, and $\mathbb{R}^k$ is the space of all $k$-dimensional vectors (hence $k\times 1$ matrices). 
Given two matrices $A$ and $B$ of the same size, denote by $A \circ B$ their inner product, by $|A|$ the Eculidean/Frobenius norm of $A$, 
by $A^\top$ the $A$'s transpose. 
I denote by $\mathcal{N}(\mu,\sigma^2)$ the probability density function of the multivariate normal distribution with mean vector $\mu$ and covariance matrix $\sigma^2$. Finally,
for any stochastic process $X=\{X_{s},$ $s\geq 0\}$, I denote by $\{\f^X_s\}_{s\geq 0}$ the natural filtration generated by $X$ and by $\langle X \rangle(t)$ its \textit{quadratic variation} (QV) between $[0,t]$. I use a bold-faced letter $\bm\pi$ to denote a \textit{policy}, and a plain letter $\pi \approx 3.14$ to denote the mathematical constant.

\subsection{Classical model-based formulation}
\label{sec:classical formulation}
For readers' convenience, I first review the classical, model-based risk-sensitive stochastic control formulation.

Let $d,n$ be given positive integers, $T>0$, 
and $b: [0,T]\times \mathbb{R}^d\times \mathcal{A} \mapsto \mathbb{R}^d$ and $\sigma:
[0,T]\times \mathbb{R}^d\times \mathcal{A}\mapsto \mathbb{R}^{d\times n}$ be given functions, where $\mathcal{A}\subset\mathbb{R}^m$ is the action/control set. The classical stochastic control problem is to control the {\it state} (or {\it feature}) dynamics governed by  a stochastic differential equation (SDE), defined on a filtered probability space $\left( \Omega ,\mathcal{F},\mathbb{P}^W; \{\mathcal{F}_s^W\}_{s\geq0}\right) $ along with a standard  
$n$-dimensional Brownian motion  $W=\{W_{s},$ $s\geq 0\}$:
\begin{equation}
	\label{eq:model classical}
	\dd X_s^{\bm{a}} = b(s,X_s^{\bm{a}},\bm a_s)\dd s + \sigma(s,X_s^{\bm{a}},\bm a_s) \dd W_s,\
	s\in [0,T],
\end{equation}
where $\bm a_s$ stands for the agent's action at time $s$. The agent's total reward is $\int_0^T r(s, X_s^{\bm a}, a_s)\dd s + h(X_T^{\bm a})$, where $r:[0,T]\times \mathbb{R}^d \times \mathcal{A} \to \mathbb{R}$ is the expected instantaneous rate of the reward, and $h: \mathbb{R}^d \to \mathbb{R}^n$ is the lump sum reward at the end of the planning period.

I make the standard assumptions in the ordinary, non-risk-sensitive stochastic control problems \citep{YZbook}. Note that these conditions are regularity conditions and growth conditions imposed on the state process and the reward function, and cannot be verified without knowing the model. Moreover, they alone do not entail the well-posedness of the problem, and have to be combined with suitably defined control (policy). 
\begin{assumption}
	\label{ass:dynamic}
	The following conditions for the state dynamics and reward functions hold true:
	
	\begin{enumerate}
		\item[(i)] $b,\sigma,r,h$ are all continuous functions in their respective arguments;
		\item[(ii)] $b,\sigma$ are uniformly Lipschitz continuous in $x$, i.e., for $\varphi \in\{ b,\ \sigma\}$, there exists a constant $C>0$ such that
		\[ |\varphi(t,x,a) - \varphi(t,x',a)| \leq C|x-x'|,\;\;\forall (t,a)\in [0,T] \times \mathcal{A},\;\forall x,x'\in \mathbb{R}^d; \]
		\item[(iii)] $b,\sigma$ have linear growth in $x$, i.e., for $\varphi \in\{ b,\ \sigma\}$, there exists a constant $C>0$ such that
		\[|\varphi(t,x,a)| \leq C(1+|x|) ,\;\;\forall (t,x,a)\in [0,T] \times \mathbb{R}^d\times \mathcal{A};\]
		\item[(iv)] $r$ and $h$ have polynomial growth  in $(x,a)$ and $x$ respectively,  i.e., 
		there exist constants $C>0$ and $\mu\geq 1$ such that
		\[
		|r(t,x,a)| \leq C(1+|x|^{\mu} + |a|^{\mu}) ,\;\;|h(x)| \leq C(1+|x|^{\mu}),\; \forall (t,x,a)\in [0,T] \times \mathbb{R}^d \times \mathcal{A}.\]
	\end{enumerate}
\end{assumption}

The risk-sensitive objective studied in literature\footnote{The long-run average of version of it will be discussed later in Section \ref{subsec:ergodic}.} (e.g., \cite{bielecki1999risk,davis2014risk}, and many others) refers to 
\begin{equation}
	\label{eq:objective}
	\frac{1}{\epsilon}\log\E^{\p^W}\left[ e^{\epsilon \left[\int_0^T r(s,X_s^{\bm{a}},\bm a_s)\dd s + h(X_T^{\bm a}) \right]} \right],
\end{equation}
for some \textit{risk sensitivity coefficient} $\epsilon\neq 0$. 

There are several motivations for this objective function and explanations for why \eqref{eq:objective} reflects the sensitivity of risk. First of all, the risk sensitivity is in contrast to simply considering the expectation of the total reward $\E\left[ \int_0^T r(s, X_s^{\bm a}, a_s)\dd s + h(X_T^{\bm a}) \right]$, and \eqref{eq:objective} is the cumulant-generating function divided by the risk sensitivity coefficient, which reflects the properties of the full distribution of the total reward, rather than simply its mean. In a special case, it is well-recognized that when $\epsilon\to 0$, \eqref{eq:objective} can be expanded as 
\begin{equation}
	\label{eq:risk sensitive expansion}
	\begin{aligned}
		& \frac{1}{\epsilon}\log\E^{\p^W}\left[ e^{\epsilon \left[\int_0^T r(s,X_s^{\bm{a}},\bm a_s)\dd s + h(X_T^{\bm a}) \right]} \right] \\
		\approx & \E^{\p^W}\left[ \int_0^T r(s,X_s^{\bm{a}},\bm a_s)\dd s + h(X_T^{\bm a})\right] + \frac{\epsilon}{2}\operatorname{Var}^{\p^W}\left[\int_0^T r(s,X_s^{\bm{a}},\bm a_s)\dd s + h(X_T^{\bm a}) \right]  .
	\end{aligned}
\end{equation}
Hence, when $\epsilon < (>) 0$, it penalizes (incentivizes) the variability of the total reward. 

Second, the form of the objective function \eqref{eq:objective} can be viewed as applying an exponential utility function on the total reward, where $-\epsilon$ measures the absolute risk aversion. Whereas, the risk sensitivity should not be confused with the commonly used notion of risk-aversion in the additive utility functions in economics literature, in which the risk-aversion may be represented via the concavity of the function $r$ and $h$.

Third, the risk sensitivity also stems from the ambiguity and robustness in the sense of \cite{hansen2011robustness}, or sometimes called the control under model misspecification \citep{hansen2001robust} or robust control \citep{dupuis2000robust}. From this perspective, the agent is still maximizing the expected total reward whereas the agent is uncertain about the true distribution, hence the agent forms ``multi-prior'' in the sense of \cite{gilboa1989maxmin} and considers the worse distribution:
\begin{equation}
	\label{eq:dro kl}
	\min_{\q \in \mathcal{Q}}\E^{\q}\left[ \int_0^T r(s, X_s^{\bm a}, a_s)\dd s + h(X_T^{\bm a}) \right] ,
\end{equation}
where $\mathcal Q$ is a collection of probability distributions specified by the agent. When it is specified as 
\[ \mathcal Q = \{ \q: D_{KL}(\q, \p^W) \leq \delta \}, \]
where $D_{KL}(\q,\p^W)$ is the Kullback–Leibler divergence from $\p^W$ to $\q$, defined as $\int \log\frac{\dd \q}{\dd \p^W} \dd \q$, and $\delta > 0$ is the radius of this set, then \eqref{eq:dro kl} can be reformulated as \eqref{eq:objective} due to the well-known Donsker and Varadhan's variational formula \citep{donsker1983asymptotic}, for suitable $\epsilon \leq 0$, which is the Lagrange multiplier to relax the constraint of $\q\in \mathcal Q$.

From the above justification, note that in \eqref{eq:objective}, $\epsilon < 0$ reflects the extra uncertainty aversion and may be plausible in most applications, however, in this paper, there is no such restriction. Admittedly, there are other forms of risk-aware objective, e.g., CVaR \citep{chow2015risk}, certainty-equivalency \citep{xu2023regret}, general utility functions \citep{wu2023risk}, and many others, I would like to focus on solving the problem \eqref{eq:objective}, given its rich interpretation from various aspects mentioned above, as well as its tractability that will be introduced as follows. Another important issue that I do not include in my discussion is the determination of the risk sensitivity coefficient $\epsilon$. On one hand, the risk sensitivity coefficient can be interpreted as part of the agent's preference and should be given exogenously. On the other hand, the formulation under the robust control framework like \eqref{eq:dro kl} regards $\epsilon$ as the Lagrange multiplier that needs to be determined endogenously (e.g., the approach in \cite{wang2023finite}), while such formulation requires an exogeneously specified uncertainty set characterized by an additional coefficient $\delta$. To the author's best knowledge, the data-driven way to determine $\epsilon$ and/or $\delta$ has only been studied in special, static setting, e.g., \citet{blanchet2022distributionally}, and remains an open question in general, dynamic settings.

The classical model-based problem is to find the optimal control (policy) to maximize \eqref{eq:objective} subject to the state dynamics \eqref{eq:model classical}, with the knowledge of the forms of functions $b,\sigma,r,h$, for a given risk sensitivity coefficient $\epsilon\neq 0$. To solve such problems, there are many standard approaches, and the predominant one is via dynamic programming to deduce Hamilton-Jacobi-Bellman (HJB)-type of equations.\footnote{They are slightly different from the standard HJB equations in the ordinary stochastic control problems because the objective function is not additive in the risk-sensitive control, but their forms are similar, see, e.g., \cite{nagai1996bellman}.} I will not repeat on these approaches but present an alternative martingale characterization of the optimal control and the optimal value function to problem \eqref{eq:objective} to conclude this subsection. 

\begin{lemma}
	\label{lemma:classical martingale}
	Suppose that there is a continuous function $V^*(t,x;\epsilon)$, satisfying $V^*(T,x;\epsilon) = h(x)$ and $\E\left[ e^{\frac{\epsilon^2}{2} \langle {V^*}^{\bm a}\rangle(T) }  \right] < \infty$ for any admissible $\bm a$. If there exists a control $\bm a^* $, such that for any initial condition $(t,x)$,
	\begin{enumerate}
		\item[(i)] $ \int_t^s \left\{ r(u, X_{u}^{\bm a^*}, a_{u}^*)\dd u + \frac{\epsilon}{2} \dd \langle  V^{*^{\bm a^*}} \rangle(u)\right\} + V^*(s, X_s^{\bm a^*};\epsilon) $ is an $( \{\f_s^W\}_{s\geq t},\p^W)$- (local) martingale; and
		\item[(ii)] for any $\bm a$, $ \int_t^s \left\{ r(u, X_{u}^{\bm a}, a_{u})\dd u + \frac{\epsilon}{2} \dd \langle  V^{*^{\bm a}} \rangle(u)\right\}  + V^*(s, X_s^{\bm a};\epsilon) $ is an $(\{\f_s^W\}_{s\geq t},\p^W)$- (local) supermartingale,
	\end{enumerate}
	where $\langle V^{*^{\bm a}} \rangle$ is the QV of the process $V^{*}(s, X_s^{\bm a};\epsilon)$. Then $\bm a^*$ is the optimal solution to \eqref{eq:objective} and $V^*(t,x;\epsilon)$ is the optimal value function.
\end{lemma}

Typically, such martingale-based characterization of the optimality is slightly weaker than or sometimes equivalent to the HJB-type of characterization, because the martingale process is often involved in the standard verification argument. Moreover, the benefit of martingale optimality principle generalizes to many other contexts, and is the key concept that motivates the seminal works in reinforcement learning by \citet{jia2022policy,jia2022policypg,jia2022q}. A particular feature in Lemma \ref{lemma:classical martingale} is that the exponential form is removed but a new term, the QV of value function process, is introduced. This is due to the fundamental connection between an \textit{exponential martingale} and the QV process incurred in applying It\^o's lemma. Hence, QV can be reviewed as an extra penalty term on the variability of the value process to induce the policy to be less risk sensitive. The rigorous statement can be found in Appendix \ref{appendix:proof}.

Notably, two recently proposed approaches in the discrete-time risk-sensitive problems are via attacking either the ``exponential Bellman equation'' \citep{fei2021exponential} or the ``distributional robust Bellman equation'' \citep{wang2023finite}. It remains unclear what the continuous-time counterpart to them is because they both involve nonlinear recursive equations. In contrast, QV is also a nonlinear functional of the value function in Lemma \ref{lemma:classical martingale}, however, such dependency is in a much more explicit form that stands for the variability of the value function along the trajectory.


\subsection{Reinforcement learning formulation}
\label{sec:filtration discuss}
I now present the exploratory formulation in the sense of \citet{wang2020reinforcement} of the problem to be studied in this paper, which is regarded as the mathematical model for analyzing continuous-time RL. Mathematically, it means to generalize the classical problem to allow \textit{relaxed control}, where the actions are randomized generated from a probability distribution. Intuitively, what an agent can do is to try a (randomly generated) sequence of actions ${\bm a} = \{a_s,t\leq s \leq T\}$\footnote{Arguably, it is not possible to sample continuously in practice, I keep to this continuous-time sequence for the ease of the development of algorithms. In principle, one can use simple process to approximate it \citep{szpruch2022optimal}.}, observe the resulting state process $X^{\bm a} = \{X_s^{\bm a},t\leq s \leq T\}$, and the realized reward $\{r(s, X_s^{\bm a}, a_s),t\leq s \leq T\}$. Meanwhile, the agent gradually adjusts actions based on these observations.\footnote{This procedure applies to both the offline and online settings.  In the former, the agent can repeatedly try different sequences of actions over the same time period $[0,T]$ and record the corresponding state processes and payoffs. In the latter, the agent updates the actions as she goes, based on all the up-to-date historical observations.} An agent can only do trial-and-error in the RL context because the agent has no knowledge about the environment (i.e. the functions $b,\sigma,r,h$) and cannot even form an HJB equation or any mathematical tool in the classical stochastic control methodology.

To be consistent with the common practice of RL literature, I restrict to the feedback relaxed control, known as the \textit{stochastic policy}. In particular, let $\bm{\pi}:(t,x)\in [0,T] \times \mathbb{R}^d \mapsto \bm{\pi}(\cdot|t,x)\in \mathcal{P}(\mathcal{A})$ be a given (feedback) policy, where $\mathcal{P}(\mathcal{A})$ is a suitable collection of probability density functions (or probability mass function for finite space).
At each time $s$, an action $a_s$ is generated or sampled from the distribution $\bm{\pi}(\cdot|s,X_s)$. 

Fix a stochastic policy $\bm{\pi}$ and an initial time--state pair $(t,x)$.
Consider the following SDE
\begin{equation}
	\label{eq:model pi}
	\dd X_s^{\bm{\pi}} = b(s,X_s^{\bm{\pi}},a_s^{\bm{\pi}})\dd s + \sigma(s,X_s^{\bm{\pi}},a_s^{\bm{\pi}}) \dd W_s,\
	s\in [t,T]; \;\;X_{t}^{\bm{\pi}} = x  ,
\end{equation}
defined on $\left( \Omega ,\mathcal{F},\mathbb{P}; \{\mathcal{F}_s\}_{s\geq0}\right) $, where $a^{\bm{\pi}} = \{a_s^{\bm{\pi}},t\leq s \leq T\}$ is an $\{\mathcal{F}_s\}_{s\geq0}$-progressively measurable action process  generated from $\bm{\pi}$. The solution to (\ref{eq:model pi}), $X^{\bm{\pi}} = \{X_s^{\bm{\pi}},t\leq s \leq T\}$,  is the sample
state processes corresponding to $a^{\bm{\pi}}$. ($X^{\bm{\pi}}$ also depends on the {\it specific} copy $a^{\bm{\pi}}$ sampled from $\bm{\pi}$)

Note that \eqref{eq:model pi} should be interpreted as the sampled process in the learning procedure, where both $\bm a^{\bm\pi}$ and $X^{\bm \pi}$ processes can be generated and observed. In particular, marginally, $X^{\bm \pi}$ has the same distribution as the solution to a different SDE \citep{wang2020reinforcement}:
\begin{equation}
	\label{eq:model relaxed}
	\dd \tilde X_s = \tilde{b}\big( s,\tilde X_s,\bm{\pi}(\cdot|s, \tilde X_s) \big)\dd t + \tilde{\sigma}\big( s,\tilde X_s,\bm{\pi}(\cdot|s, \tilde X_s) \big) \dd W_s,\;s\in[t,T];\;\;\ \tilde X_{t} = x,
\end{equation}
where
\[ \tilde{b}\big(s,x,\pi(\cdot)\big) = \int_{\mathcal{A}} b(s,x,a) \pi(a)\dd a,\ \; \tilde{\sigma}\tilde{\sigma}^\top\big(s,x,\pi(\cdot)\big) = \int_{\mathcal{A}} \sigma\sigma^\top(s,x,a) \pi(a)\dd a.\] 
When the expectation is taken with respect to only $X^{\bm\pi}$ under $\p$, it can be equivalently written as the expectation respect to $\tilde X^{\bm\pi}$ under $\p^W$. When the expectation involves $a^{\bm\pi}$ under $\p$, then it can be integrated out by conditioning on $X^{\bm\pi}$ using the policy $\bm\pi$.\footnote{In \citet{wang2020reinforcement}, this result is derived using the law-of-large-number type of argument to integrate over the randomization of $a$. More discussions on the relation between \eqref{eq:model pi} and \eqref{eq:model relaxed} can be found in \citet{jia2022policypg,jia2022q}.}

Based on the same spirit as in \citet{wang2020reinforcement}, I add an entropy regularizer to the reward function to encourage  exploration (represented by the stochastic policy), hence the risk-sensitive objective \eqref{eq:objective} becomes\footnote{When $\lambda = 0$, \eqref{eq:objective relaxed action} reduces to the classical risk-sensitive control problem \eqref{eq:objective}. For the general non-risk-sensitive problems, the entropy-regularized problem approximates the classical problem by the order of $O(-\lambda\log\lambda)$, see \cite{tang2021exploratory}. However, whether the same order applies to the risk-sensitive problem still remains an open question. For the special case of Merton's investment problem studied in Section \ref{sec:numerical}, the connection between the entropy-regularized problem and the classical problem becomes clear because the closed-form solution to the entropy-regularized problem can be found in Appendix \ref{sec:merton true solution}. In this case, the gap is indeed $\frac{\lambda}{2}\log\frac{2\pi\lambda}{\gamma\sigma^2}$.} 
\begin{equation}
	\label{eq:objective relaxed action}
	\begin{aligned}
		J(t,x;\bm{\pi};\epsilon) = &	\frac{1}{\epsilon}\log\E^{\p}\left[e^{\epsilon \left[ \int_t^T \left[ r(s,X_s^{\bm{\pi}},a_s^{\bm \pi}) - \lambda \log\bm\pi(a_s^{\bm\pi}|s,X_s^{\bm\pi}) \right] \dd s + h(X_T^{\bm \pi}) \right]} \Big|X_t^{\bm \pi}= x\right],
	\end{aligned}
\end{equation}
where $\E^{\p}$ is the expectation with respect to both the Brownian motion and the action randomization.

Note that in \eqref{eq:objective relaxed action}, the entropy regularizer is added inside the exponential form. This is because the regularizer can be regarded as a fictitious reward for taking random actions, and hence, should be added to the actual reward $r$. Moreover, $\lambda \geq 0$ is a given weighting parameter on exploration, also known as the {\it temperature} parameter. It is assumed to be known or chosen by the agent. In this paper, I first leave it as exogenous, and revisit it when conducting analysis for a specific algorithm for a simple problem in Section \ref{subsec:merton} to understand better how it affects the learning performance. In short, the intuition is, as $\lambda$ rises, there exists a tradeoff: It incentivizes more exploration, and hence, accelerates the convergence, however it incurs larger noise in sampling, which reduces efficiency and exploitation.

The following gives the precise definition of admissible policies.

\begin{definition}
	\label{ass:admissible}
	A policy $\bm{\pi}=\bm{\pi}(\cdot|\cdot,\cdot)$ is called {\it admissible} if
	\begin{enumerate}
		\item[(i)] $\bm{\pi}(\cdot|t,x)\in \mathcal{P}(\mathcal{A})$, $\operatorname{supp}\bm\pi(\cdot|t,x) = \mathcal{A}$ for every $(t,x)\in [0,T]\times \mathbb{R}^d$, and $\bm{\pi}(a|t,x):(t,x,a)\in [0,T] \times \mathbb{R}^d\times \mathcal{A}\mapsto
		\mathbb{R}$ is measurable;
		\item[(ii)] the SDE \eqref{eq:model relaxed} admits a unique weak solution (in the sense of distribution) for any initial $(t,x)\in [0,T] \times \mathbb{R}^d$;
		\item[(iii)] 
		$\int_{\mathcal{A}} |r(t,x,a) - \gamma \log\bm{\pi}(a|t,x) | \bm{\pi}(a|t,x)\dd a \leq C(1+|x|^{\mu})$,  $\forall (t,x)$ where $C>0$ and $\mu\geq 1$ are constants;
		\item[(iv)] $\bm{\pi}(a|t,x)$ is continuous in $(t,x)$  and uniformly Lipschitz continuous in $x$ in the total variation distance, i.e., for each fixed $a$, $ \int_{\mathcal{A}} |\bm{\pi}(a|t,x) - \bm{\pi}(a|t',x')|\dd a \to 0$ as $(t',x')\to (t,x)$, and
		there is a constant $C>0$ independent of $(t,a)$ such that
		\[ \int_{\mathcal{A}} |\bm{\pi}(a|t,x) - \bm{\pi}(a|t,x')|\dd a \leq C|x-x'|,\;\;\forall x,x'\in \mathbb{R}^d. \]
	\end{enumerate}
	The collection of admissible policies is denoted by $\bm\Pi$.
\end{definition}

I conclude this subsection by stating the counterpart to Lemma \ref{lemma:classical martingale} for the exploratory problem. 

\begin{theorem}
	\label{thm:exploratory martingale}
	Suppose that there is a function $J^*(t,x;\epsilon) \in C^{1,2}\big([0,T)\times \mathbb{R}^d \big) \cap C\big([0,T]\times \mathbb{R}^d \big)$, satisfying $J^*(T,x;\epsilon) = h(x)$ and $\E\left[ e^{\frac{\epsilon^2}{2} \langle {J^*}^{\bm \pi}\rangle(T) }  \right] < \infty$ for any $\bm \pi\in \bm\Pi$. Consider a policy defined as
	\begin{equation}
		\label{eq:pi star in v star}
		\bm\pi^*(a|t,x)  = \frac{\exp\left\{ \frac{1}{\lambda}\left[ \mathcal{L}^a J^*(t,x) + r(t,x,a)  + \frac{\epsilon}{2}|\sigma(t,x,a)^\top \frac{\partial J^*}{\partial x}(t,x)|^2   \right] \right\} }{\int_{\mathcal A}  \exp\left\{ \frac{1}{\lambda}\left[ \mathcal{L}^a J^*(t,x) + r(t,x,a)  + \frac{\epsilon}{2}|\sigma(t,x,a)^\top \frac{\partial J^*}{\partial x}(t,x)|^2   \right] \right\} \dd a } ,
	\end{equation}
	where $\mathcal{L}^a$ is the {\it infinitesimal generator} associated with the diffusion process \eqref{eq:model classical}:
	\[ \mathcal{L}^a \varphi (t,x): = \frac{\partial \varphi}{\partial t}(t,x) + b\big( t,x, a\big) \circ \frac{\partial \varphi}{\partial x}(t,x) + \frac{1}{2}\sigma^2\big( t,x,a \big) \circ \frac{\partial^2 \varphi}{\partial x^2}(t,x),\;\;a\in \mathcal{A}. \]
	Here, $\frac{\partial \varphi}{\partial x} \in \mathbb{R}^d$ is the gradient, and $\frac{\partial^2 \varphi}{\partial x^2}\in \mathbb{R}^{d\times d}$ is the Hessian.
	
	If for any initial condition $(t,x)$,
	\[ \int_t^s \left\{ \left[ r(u, X_{u}^{\bm \pi^*}, a_{u}^{\bm \pi^*}) - \lambda\log\bm\pi^*(a_{u}^{\bm\pi^*}|u, X_{u}^{\bm\pi^*}) \right] \dd u + \frac{\epsilon}{2} \dd \langle  J^{*^{\bm \pi^*}} \rangle(u)\right\}+ J^*(s, X_s^{\bm \pi^*};\epsilon) \]
	is an $(\{\f_s^X \}_{s\geq t},\p)$- (local) martingale; where $\langle J^{*^{\bm \pi}} \rangle$ is the QV of the process $J^{*}(s, X_s^{\bm \pi};\epsilon)$. Then $\bm \pi^*$ is the optimal solution to \eqref{eq:objective relaxed action} and $J^*(t,x;\epsilon)$ is the optimal value function.
\end{theorem}

It is interesting to notice that the requirement $(ii)$ in Lemma \ref{lemma:classical martingale} now reduces to one martingale condition under the exploratory formulation, and the ``supermartingale'' requirement is gone. In fact, due to the entropy-regularization, such  ``supermartingale'' condition is replaced by the requirement of the entropy-maximizing policy, which takes a more explicit form of Gibbs measure \eqref{eq:pi star in v star}, also known as the Boltzmann exploration scheme. Recall that in the typical martingale optimality principle as in Lemma \ref{lemma:classical martingale}, the ``martingale'' condition $(i)$ is an equality condition that needs to be satisfied only for the optimal policy, however, the ``supermartingale'' condition $(ii)$ needs to be satisfied by all polices, which includes an infinitely many constraints. In contrast, the entropy-regularization resolves such difficulty by solving the optimal policy explicitly. 

\section{Risk-sensitive q-Learning}
\label{sec:risk sensitive q learning}
\subsection{Definition of q-function}
The notion of the Q-function in the discrete-time risk-sensitive RL problem is less standard and less tractable to work with, and hence, different variatants of Q-function have been proposed, e.g., in \citet{fei2021exponential}. I would like to introduce the notion of (little) ``q-function'' for the continuous-time risk-sensitive RL problem based on a different motivation -- by recalling the fundamental relation between the q-function and the ordinary, non-risk-sensitive stochastic control problems revealed in \citet{jia2022q}. In short, the q-function should be the properly scaled \textit{advantage function} by taking the limit of time-discretization to zero.

More precisely, for a given policy $\bm\pi$, and $(t,x,a)\in [0,T)\times \mathbb{R}^d\times \mathcal{A}$, consider a  ``perturbed" policy of $\bm{\pi}$ as follows: It takes the action $a\in \mathcal{A}$ on $[t,t+\Delta t)$ where $\Delta t>0$, and then follows $\bm{\pi}$ on $[t+\Delta t,T]$. Then the advantage function is  
\begin{equation}
	\label{eq:deriving q function}
	\begin{aligned}
		& \frac{1}{\epsilon}\log\E^{\p}\left[e^{\epsilon \left[\int_t^{t+\Delta t}r(s,X_s^a,a)\dd s + \int_{t+\Delta t}^T \left[ r(s,X_s^{\bm{\pi}},a_s^{\bm \pi}) - \lambda \log\pi(a_s^{\bm\pi}|s,X_s^{\bm\pi}) \right] \dd s + h(X_T^{\bm \pi}) \right]} \Big|X_t = x\right]\\
		& - J(t,x;\bm\pi;\epsilon) \\
		= & \frac{1}{\epsilon}\log\E^{\p}\left[e^{\epsilon \left[\int_t^{t+\Delta t}r(s,X_s^a,a)\dd s + J(t+\Delta t, X_{t+\Delta t}^{a};\bm\pi;\epsilon)  - J(t,x;\bm\pi;\epsilon) \right]} \Big|X_t^a = x\right] \\
		= & \frac{1}{\epsilon}\log\E\Bigg[1 + \epsilon \int_t^{t+\Delta t} e^{\epsilon \left[\int_t^{s}r(u,X_{u}^a,a)\dd u + J(s, X_{s}^{a};\bm\pi;\epsilon)  - J(t,x;\bm\pi;\epsilon) \right]} \\
		&\times  \left[ r(s,X_s^a,a) + \mathcal{L}^a J(s,X_s^a;\bm\pi;\epsilon) +  \frac{\epsilon}{2} |\sigma(s,X_s^a,a)^\top\frac{\partial J}{\partial x}(s,X_s^a;\bm\pi;\epsilon)|^2  \right]\dd s    \Big| X_t^a = x\Bigg] \\
		= & \left[ r(t,x,a) + \mathcal{L}^a J(t,x;\bm\pi;\epsilon) +  \frac{\epsilon}{2} |\sigma(t,x,a)^\top\frac{\partial J}{\partial x}(t,x;\bm\pi;\epsilon)|^2 \right] \Delta t+ o(\Delta t).
	\end{aligned}
\end{equation}
After properly being scaled by $\Delta t$ and taking limit, the expansion in \eqref{eq:deriving q function} motivates the following definition for the q-function.

\begin{definition}
	\label{def:def q function}
	The q-function associated with a given policy $\bm\pi\in \bm\Pi$ is defined as
	\begin{equation}
		\label{eq:q rate}
		\begin{aligned}
			q(t,x,a;\bm{\pi}) = & \mathcal{L}^a J(t,x;\bm\pi;\epsilon) + r(t,x,a) +  \frac{\epsilon}{2} |\sigma(t,x,a)^\top\frac{\partial J}{\partial x}(t,x;\bm\pi;\epsilon)|^2,\\
			& \;\;\;\;\;(t,x,a)\in [0,T]\times \mathbb{R}^d\times \mathcal{A}.
		\end{aligned}
	\end{equation}
\end{definition}
Note that the form of q-function coincides with the optimal policy in \eqref{eq:pi star in v star}, which leads to the optimal q-function:
\begin{equation}
	\label{eq:optimal q}
	q^*(t,x,a;\epsilon) = \mathcal{L}^a J^*(t,x;\epsilon) + r(t,x,a) +  \frac{\epsilon}{2} \left| \sigma(t,x,a)^\top\frac{\partial J^*}{\partial x}(t,x;\epsilon) \right|^2. 
\end{equation}

\subsection{Martingale characterization of the optimal q-function}
In most applications, only the optimal policy is of the interest. Hence, I will focus on the characterization of the optimal q-function in this subsection, and how such martingale characterization can facilitate q-learning algorithm in both \textit{on-policy} and \textit{off-policy} settings.\footnote{It is possible to establish parallel results as in \citet{jia2022q} for the q-function associated with any policy, but these results are omitted in this paper. On-policy and off-policy stand for two different learning settings. On-policy learning updates the policy currently in use to generate data, whereas off-policy learning aims to estimate a policy of interest (called \textit{target policy}), typically the optimal policy, using the data collected from a different, possibly sub-optimal policy (called \textit{behavior policy}). For example, learning playing Go from existing human players is off-policy, whereas learning via self-play is likely on-policy.}

First of all, I present a necessary condition for the optimal q-function and optimal policy. This is a parallel result to \citet[Proposition 8]{jia2022q}.
\begin{proposition}\label{prop:qstar1}
	It holds that
	\begin{equation}
		\label{eq:optimal q hjb}
		\int_{\mathcal A} \exp\{ \frac{1}{\lambda} q^*(t,x,a;\epsilon) \}\dd a = 1,
	\end{equation}
	for all $(t,x)$, and consequently
	the optimal policy $\bm\pi^*$ is  
	\begin{equation}
		\label{eq:optimal pi and q}
		\bm\pi^*(a|t,x) = \exp\{ \frac{1}{\lambda} q^*(t,x,a;\epsilon) \} .
	\end{equation}
\end{proposition}

Given \eqref{eq:optimal q hjb} and the terminal condition for the value function as constraints, the optimal value function and the optimal q-function can be characterized via martingale conditions. 
\begin{theorem}
	\label{thm:q optimal}
	Let a function $\widehat{J^*}\in C^{1,2}\big([0,T)\times \mathbb{R}^d \big) \cap C\big([0,T]\times \mathbb{R}^d \big)$ with polynomial growth in its all derivatives and a continuous function $\widehat{q^*}:[0,T]\times \mathbb{R}^d\times \mathcal{A}\to \mathbb{R}$ be given, satisfying
	\begin{equation}
		\label{eq:q hjb2 optimal}
		\widehat{J^*}(T,x) = h(x),\;\;\; \int_{\mathcal{A}} \exp\{ \frac{1}{\lambda} \widehat{q^*}(t,x,a) \} \dd a =1,\;\;\forall (t,x)\in[0,T]\times\mathbb{R}^d.
	\end{equation}
	Then 
	\begin{enumerate}
		\item[(i)] If $\widehat{J^*}$ and $\widehat{q^*}$ are respectively the optimal value function and the optimal q-function, then for any $\bm\pi\in \bm\Pi$ and all $(t,x)\in[0,T]\times\mathbb{R}^d$, the following process
		\begin{equation}
			\label{eq:martingale with q function2 optimal}
			\widehat{J^*}(s,{X}_s^{\bm\pi}) + \int_t^s \left\{ \left[ r(u,{X}_{u}^{\bm\pi},a^{\bm\pi}_{u}) - \widehat{q^*}(u,{X}_{u}^{\bm\pi},a^{\bm\pi}_{u})  \right]\dd u + \frac{\epsilon}{2}\dd \langle \widehat{J^*}^{\bm\pi} \rangle(u)\right\}
		\end{equation}
		is an $(\{\f_s\}_{s\geq t},\p)$-martingale, where $\{{X}_s^{\bm\pi}, t\leq s\leq T\}$ is the solution to (\ref{eq:model pi}) under $\bm\pi$ with ${X}_t^{\bm\pi}=x$.
		\item[(ii)] Suppose stronger regularity condition $\E\left[ e^{\frac{\epsilon^2}{2} \langle \widehat{J^*}^{\bm \pi}\rangle(T) }  \right] < \infty$ for any $\bm \pi\in \bm\Pi$ holds. If there exists one $\bm\pi\in \bm\Pi$ such that for all $(t,x)\in[0,T]\times\mathbb{R}^d$,  \eqref{eq:martingale with q function2 optimal} is an $(\{\f_s\}_{s\geq t},\p)$-martingale, then $\widehat{J^*}$ and $\widehat{q^*}$ are respectively the optimal value function and the optimal q-function. Moreover, in this case, $\widehat{\bm\pi^*}(a|t,x) = \exp\{  \frac{1}{\lambda}\widehat{q^*}(t,x,a) \}$ is the optimal policy. 
	\end{enumerate}
\end{theorem}

Theorem \ref{thm:q optimal} is the parallel result to \citet[Theorem 9]{jia2022q}, where the only difference is that the integral term in \eqref{eq:martingale with q function2 optimal} now involves an extra QV term. 
Because the extra QV term is always increasing, when $\epsilon < 0$ (corresponding to the typical situation), the process ${J^*}(s,{X}_s^{\bm\pi}) + \int_t^s \left[ r(u,{X}_{u}^{\bm\pi},a^{\bm\pi}_{u}) - {q^*}(u,{X}_{u}^{\bm\pi},a^{\bm\pi}_{u})  \right]\dd u$ is a sub-martingale. Its expectation is increasing because the planning period is shorten and hence becomes less uncertain. In contrast to the non-risk-sensitive counterpart, such a process ought to be a martingale. This reflects the intertemporal preference on uncertainty that the agent is no longer indifferent between when the uncertainty is resolved as time goes by.  

Note that, the conditions required in ($ii$) are imposed jointly on the policy $\boldsymbol{\pi}$, the candidate value function and q-function $\widehat{J^*},\widehat{q^*}$; by ($i$), if the candidate value function and q-function are the optimal value function and q-function, then this policy $\boldsymbol{\pi}$ can be any arbitrary policy. Therefore, as long as the original problem \eqref{eq:objective relaxed action} admits an optimal policy, then we do not need to impose additional restrictions.

\subsection{Risk-sensitive q-Learning}
Note that based on the martingale characterization of the optimal q-function, it is possible to devise various algorithms that either aim to minimize a suitable loss function, or to solve a system of moment conditions, as evidenced in \cite{jia2022q}. Then one may use either stochastic gradient decent or stochastic approximation algorithm to numerically find the solutions. 

I highlight that since \eqref{eq:q hjb2 optimal} includes a constraint on the integration of the approximated q-function, and hence, is difficult to compute in general. I refer to the discussion in \cite{jia2022q} on the case in which such integration cannot be explicitly calculated and another approximation is required. In the following, I restrict to the case in which such integration is easy to compute, e.g., when q-function is a quadratic function in the action $a$. Algorithm \ref{algo:offline episodic general} describes a simple risk-sensitive q-learning algorithm based on the theoretical results in the previous subsections. This algorithm sets up moment conditions and uses stochastic approximation to iterate the parameters to be learned. Note that the constraint on the terminal value of the approximated value function in \eqref{eq:q hjb2 optimal} is not critical, because the terminal payoff function at any sample point $h(X_T^{\bm\pi})$ can be observed directly, and only the terminal values of the approximated value function at those sample points are concerned in the algorithm. Moreover, the q-learning algorithm can be applied to both on-policy and off-policy settings. 

\begin{algorithm}[htbp]
	\caption{Offline--Episodic Risk-sensitive q-Learning Algorithm}
	\textbf{Inputs}: initial state $x_0$,  horizon $T$, time step $\Delta t$, number of episodes $N$, number of mesh grids $K$, initial learning rates $\alpha_{\theta},\alpha_{\psi}$ and a learning rate schedule function $l(\cdot)$ (a function of the number of episodes), functional forms  of parameterized  value function $J^{\theta}(\cdot,\cdot)$ and  q-function $q^{\psi}(\cdot,\cdot,\cdot)$ satisfying \eqref{eq:q hjb2 optimal}, functional forms of test functions $\bm{\xi}(t,x_{\cdot \wedge t},a_{\cdot \wedge t})$ and $\bm{\zeta}(t,x_{\cdot \wedge t},a_{\cdot \wedge t})$, temperature parameter $\lambda$, and risk sensitivity parameter $\epsilon$.

	\textbf{Required program (on-policy)}: environment simulator $(x',r) = \textit{Environment}_{\Delta t}(t,x,a)$ that takes current time--state pair $(t,x)$ and action $a$ as inputs and generates state $x'$ at time $t+\Delta t$ and  instantaneous reward $r$ at time $t$ as outputs. Policy $\bm\pi^{\psi}(a|t,x) = \exp\{  \frac{1}{\lambda}q^{\psi}(t,x,a) \}$.
	
	\textbf{Required program (off-policy)}: observations $ \{a_{t_k}, r_{t_{k}}, x_{t_{k+1}}\}_{k = 0,\cdots, K-1}\cup \{ x_{t_K}, h(x_{t_K})\} = \textit{Observation}(\Delta t)$ including the observed actions, rewards, and state trajectories under the given behavior policy  at the sampling time grids with step size $\Delta t$.

	\textbf{Learning procedure}:
	\begin{algorithmic}
		\State Initialize $\theta,\psi$.
		\For{episode $j=1$ to $N$} \State{Initialize $k = 0$. Observe  initial state $x_0$ and store $x_{t_k} \leftarrow  x_0$.
			
			\Comment{\textbf{On-policy case}
				
				\While{$k < K$}
				
				Generate action $a_{t_k}\sim \bm{\pi}^{\psi}(\cdot|t_k,x_{t_k})$.
				
				Apply $a_{t_k}$ to environment simulator $(x,r) = Environment_{\Delta t}(t_k, x_{t_k}, a_{t_k})$, and observe new state $x$ and reward $r$ as outputs. Store $x_{t_{k+1}} \leftarrow x$ and $r_{t_k} \leftarrow r$.
				
				Update $k \leftarrow k + 1$.
				
				\EndWhile
				
			}
			
			\Comment{\textbf{Off-policy case}
				
				Obtain one observation $\{a_{t_k}, r_{t_{k}}, x_{t_{k+1}}\}_{k = 0,\cdots, K-1}\cup \{ x_{t_K}, h(x_{t_K})\} = \textit{Observation}(\Delta t)$.
				
			}
			
			\Comment{\textbf{After getting a trajectory}}
			
			For every $i = 0,1,\cdots,K-1$, compute and store test functions $\xi_{t_i} = \bm{\xi}(t_i, x_{t_0},\cdots, x_{t_i},a_{t_0},\cdots, a_{t_i})$, $\zeta_{t_i} = \bm{\zeta}(t_i, x_{t_0},\cdots, x_{t_i},a_{t_0},\cdots, a_{t_i})$.	
			
			Compute
			\[ \delta_{t_i} = J^{\theta}(t_{i+1},x_{t_{i+1}}) - J^{\theta}(t_{i},x_{t_{i}}) + r_{t_i}\Delta t -q^{\psi}(t_{i},x_{t_{i}},a_{t_i})\Delta t + \frac{\epsilon}{2}\left[ J^{\theta}(t_{i+1},x_{t_{i+1}}) - J^{\theta}(t_{i},x_{t_{i}})   \right]^2  \]
			for every $i = 0,1,\cdots,K-1$, and define
			\[ \Delta \theta = \sum_{i=0}^{K-1} \xi_{t_i} \delta_{t_i},\ \Delta \psi =   \sum_{i=0}^{K-1}\zeta_{t_i}\delta_{t_i} . \]

			Update $\theta$ and $\psi$ by
			\[ \theta \leftarrow \theta + l(j)\alpha_{\theta} \Delta \theta .\]
			\[ \psi \leftarrow \psi + l(j)\alpha_{\psi} \Delta \psi .  \]
			
		}
		\EndFor
	\end{algorithmic}
	\label{algo:offline episodic general}
\end{algorithm}

\section{Extensions}
\label{sec:extensions}

\subsection{Contrast with policy gradient}
As an alternative method to solve RL problems, the policy gradient is one of the most commonly used algorithms. In the continuous-time, exploratory diffusion process setting, \citet{jia2022policypg} obtain the representation of policy gradient for the non-risk-sensitive problems by studying the associated Feynman-Kac-type partial differential equation (PDE) for the value function, such that the value function can be identified as an implicit function of the policy. 

In the above discussion, I have turned the risk-sensitive RL problems to a martingale problem that only differs from the non-risk-sensitive counterpart by a QV penalty term. Does it mean the policy gradient representation can be achieved in parallel by simply adding an extra QV penalty term as well? 

In this subsection, I follow the same derivation for the policy gradient in \citet{jia2022policypg}, and show that such representation is \textit{not valid} for the risk sensitive RL problems. Thus, it highlights the benefit of considering q-function and sheds light upon the relation between policy gradient and q-learning. 

Consider a policy $\bm\pi^{\phi}$, parameterized by $\phi$. From the definition \eqref{eq:objective relaxed action}, notice that 
\[ e^{\epsilon \left[ J(t,X_t^{\bm\pi^{\phi}};\bm\pi^{\phi};\epsilon) + \int_0^t [r(s,X_s^{\bm\pi^{\phi}},a_s^{\bm\pi^{\phi}}) - \lambda\log\bm\pi^{\phi}(a_s^{\bm\pi^{\phi}}|s,X_s^{\bm\pi^{\phi}})]\dd s \right]} \]
is a martingale. Applying It\^o's lemma, I obtain a PDE characterization of the value function $J(\cdot,\cdot;\bm\pi^{\phi};\epsilon)$, which satisfies
\begin{equation}
	\label{eq:feynman kac risk sensitive}
	\begin{aligned}
		\int_{\mathcal A} & \bigg[ \mathcal{L}^a J(t,x;\bm\pi^{\phi};\epsilon) + r(t,x,a) -\lambda\log \bm\pi^{\phi}(a|t,x) \\
		& +  \frac{\epsilon}{2} |\sigma(t,x,a)^\top\frac{\partial J}{\partial x}(t,x;\bm\pi^{\phi};\epsilon)|^2 \bigg] \bm\pi^{\phi}(a|t,x)\dd a = 0,
	\end{aligned}
\end{equation}
for all $(t,x)\in [0,T]\times \mathbb{R}^d$, with terminal condition $J(T,x;\bm\pi;\epsilon) = h(x)$. 

Note that, unlike the usual Feynman-Kac-type PDE (with $\epsilon=0$), which is a linear parabolic equation, \eqref{eq:feynman kac risk sensitive} contains a nonlinear term $|\frac{\partial J}{\partial x}(t,x;\bm\pi^{\phi};\epsilon)|^2$ that arises from the QV of the value process. Therefore, differentiating with respect to $\phi$ in \eqref{eq:feynman kac risk sensitive} (as in \cite{jia2022policypg}) does not yield a linear PDE for $G(t,x;\bm\pi^{\phi};\epsilon) = \frac{\partial }{\partial \phi}J(t,x;\bm\pi^{\phi};\epsilon)$. More precisely, $G(t,x;\bm\pi^{\phi};\epsilon)$ satisfies
\[ \begin{aligned}
	0 = & \int_{\mathcal A}\Bigg\{ \frac{\partial \log\bm\pi^{\phi}(a|t,x)}{\partial \phi} \bigg[ \mathcal{L}^a J(t,x;\bm\pi^{\phi};\epsilon) + r(t,x,a) -\lambda\log \bm\pi^{\phi}(a|t,x) \\
	& +  \frac{\epsilon}{2} |\sigma(t,x,a)^\top\frac{\partial J}{\partial x}(t,x;\bm\pi^{\phi};\epsilon)|^2 \bigg]  \\
	& + \left[ \mathcal{L}^a G(t,x;\bm\pi^{\phi};\epsilon) +  \epsilon \frac{\partial G}{\partial x}(t,x;\bm\pi^{\phi};\epsilon) \sigma(t,x,a)\sigma(t,x,a)^\top\frac{\partial J}{\partial x}(t,x;\bm\pi^{\phi};\epsilon) \right] \Bigg\}\bm\pi^{\phi}(a|t,x)\dd a .
\end{aligned}  \]
Hence the gradient can be represented by
\begin{equation}
	\label{eq:pg risk sensitive}
	\begin{aligned}
		G(t,x;\bm\pi^{\phi};\epsilon) = \E\Bigg[\int_t^T & \frac{\partial \log\bm\pi^{\phi}(a_s^{\bm\pi^{\phi}}|s,X_s^{\bm\pi^{\phi}})}{\partial \phi}\bigg[ \dd J^{\bm\pi^{\phi}}_s + r(s,X_s^{\bm\pi^{\phi}},a_s^{\bm\pi^{\phi}})\dd s \\
		- &\lambda\log \bm\pi^{\phi}(a_s^{\bm\pi^{\phi}}|s,X_s^{\bm\pi^{\phi}})\dd s + \frac{\epsilon}{2}\dd \langle J^{\bm\pi^{\phi}} \rangle(s)  \bigg] \\
		+ & \epsilon \dd \langle J^{\bm\pi^{\phi}}, G^{\bm\pi^{\phi}} \rangle(s) \mid X_t^{\bm\pi^{\phi}} = x \Bigg] .
	\end{aligned}
\end{equation} 

In the above expression, note that $\dd J^{\bm\pi^{\phi}}_s + r(s,X_s^{\bm\pi^{\phi}},a_s^{\bm\pi^{\phi}})\dd s -\lambda\log \bm\pi^{\phi}(a_s^{\bm\pi^{\phi}}|s,X_s^{\bm\pi^{\phi}})\dd s$ is the usual ``temporal difference'' (TD) term, and it is augmented by the QV $\frac{\epsilon}{2}\dd \langle J^{\bm\pi^{\phi}} \rangle(s)$ due to the QV penalty. It is remarkable that the policy gradient is \textit{no longer} the gradient of the log-likelihood of the policy multiplied by the (augmented) TD error, instead, an extra term arises because of the nonlinear nature. Furthermore, this extra term $\epsilon \dd \langle J^{\bm\pi^{\phi}}, G^{\bm\pi^{\phi}} \rangle(s)$ is the cross-variation between the value process and its gradient process, and remains uncomputable directly. Therefore, it is difficult to obtain an unbiased estimate of the policy gradient in the risk-sensitive RL. In contrast, q-learning does not have this problem due to the universal martingale property. Example \ref{eg:pg} provides a simple illustration of the bias caused by ignoring this extra term in estimating the policy gradient, and such bias is proportional to the risk sensitivity coefficient $\epsilon$. 
\begin{eg}
	\label{eg:pg}
	Consider the following state process:
	\[ \dd X_t^{\bm\pi^{\phi}} = a_t^{\bm\pi^{\phi}}\dd t + \dd W_t,\ X_t^{\bm\pi^{\phi}} = x,\  \ a_t^{\bm\pi^{\phi}}|X_t^{\bm\pi^{\phi}} \sim \bm\pi^{\phi} = \mathcal{N}(-\phi X_t^{\bm\pi^{\phi}},1) , \]
	with $\phi > 0$. Then the distribution of $X_t^{\bm\pi^{\phi}}$ is the same as an Ornstein–Uhlenbeck process $\tilde X_t$ that satisfies 
	\[  \dd \tilde X_t = -\phi \tilde X_t\dd t + \dd W_t,\ \tilde X_t = x .\]
	Therefore, the marginal distribution of $\tilde X_T$, or equivalently, $X^{\bm\pi^{\phi}}_T$ is $\mathcal N\left(x e^{-\phi(T-t)}, \frac{1}{2\phi}\left(1 - e^{-2\phi(T-t)}\right)  \right)$. For simplicity, I ignore the running rewards and the entropy terms and consider the risk-sensitive value function:
	\[ J(t,x;\bm\pi^{\phi};\epsilon) = \frac{1}{\epsilon}\log\E\left[ e^{\epsilon X_T^{\bm\pi^{\phi}}}  \mid X_t^{\bm\pi^{\phi}} = x\right] = xe^{-\phi(T-t)} + \frac{\epsilon}{4\phi}\left(1 -  e^{-2\phi(T-t)} \right) . \]
	By direct calculation, the true policy gradient with respect to $\phi$ is
	\[ \frac{\partial }{\partial \phi} J(0,x;\bm\pi^{\phi};\epsilon) = -x e^{-\phi T}T + \frac{\epsilon T}{2\phi}e^{-2\phi T} - \frac{\epsilon}{4\phi^2}(1 - e^{-2\phi T}) . \]
	However, if the extra term $\langle J^{\bm\pi^{\phi}}, G^{\bm\pi^{\phi}} \rangle$ is ignored, then \eqref{eq:pg risk sensitive} becomes
	\[\begin{aligned}
		& \E\left[\int_0^T \frac{\partial \bm\pi^{\phi}(a_s^{\bm\pi^{\phi}} | X_s^{\bm\pi^{\phi}}) }{\partial \phi} \left[\dd J(s,X_s^{\bm\pi^{\phi}};\bm\pi^{\phi};\epsilon) + \frac{\epsilon}{2}\dd \langle J^{\bm\pi^{\phi}} \rangle(s)   \right]  \mid X_0^{\bm\pi^{\phi}} = x  \right]\\
		= & \E\left[\int_0^T -(a_s^{\bm\pi^{\phi}} +\phi X_s^{\bm\pi^{\phi}}) X_s^{\bm\pi^{\phi}} e^{-\phi(T-s)}(a_s^{\bm\pi^{\phi}} +\phi X_s^{\bm\pi^{\phi}})\dd s   \mid X_0^{\bm\pi^{\phi}} = x \right] \\
		= & -e^{-\phi T}\int_0^T e^{\phi s}  \E\left[ X_s^{\bm\pi^{\phi}}  \mid X_0^{\bm\pi^{\phi}} = x \right]\dd s = -xe^{-\phi T} T . 
	\end{aligned}  \]
	It is clearly different from $\frac{\partial }{\partial \phi} J(0,x;\bm\pi^{\phi};\epsilon)$.
\end{eg}

\subsection{Ergodic tasks}
\label{subsec:ergodic}
Next, I extend the q-learning to ergodic tasks. That is, the risk-sensitive objective is 
\[ \liminf_{T\to \infty} \frac{1}{\epsilon T}\log\E^{\p}\left[e^{\epsilon \left[ \int_t^T \left[ r(X_s^{\bm{\pi}},a_s^{\bm \pi}) - \lambda \log\bm\pi(a_s^{\bm\pi}|X_s^{\bm\pi}) \right] \dd s \right]} \Big|X_t^{\bm \pi}= x\right] . \]

The q-function in such ergodic task can be similarly defined as in Definition \ref{def:def q function} except that the associated value function $J(\cdot;\bm\pi;\epsilon)$ becomes time-invariant, and is only unique up to a constant. In addition, there is a number $\beta(\bm\pi;\epsilon)$ that is associated with the policy. This number stands for the objective function value (the long-term average) and it is also part of the solutions. Moreover, since the ergodic tasks focus on the long-term average performance, we restrict the SDE \eqref{eq:model relaxed} to be time-homogeneous and also restrict the admissible policies to time-invariant ones.

		\begin{definition}
			\label{ass:admissible stationary}
			A time-invariant policy $\bm{\pi}=\bm{\pi}(\cdot|\cdot)$ is called {\it admissible} if
			\begin{enumerate}
				\item[(i)] $\bm{\pi}(\cdot|x)\in \mathcal{P}(\mathcal{A})$, $\operatorname{supp}\bm\pi(\cdot|x) = \mathcal{A}$ for every $x\in \mathbb{R}^d$, and $\bm{\pi}(a|x):(x,a)\in  \mathbb{R}^d\times \mathcal{A}\mapsto
				\mathbb{R}$ is measurable;
				\item[(ii)] the SDE \eqref{eq:model relaxed} admits a unique weak solution (in the sense of distribution) for any initial $(t,x)\in \mathbb R_+ \times \mathbb{R}^d$;
				\item[(iii)] 
				$\int_{\mathcal{A}} |r(x,a) - \gamma \log\bm{\pi}(a|x) | \bm{\pi}(a|x)\dd a \leq C(1+|x|^{\mu})$,  $\forall (t,x)$ where $C>0$ and $\mu\geq 1$ are constants;
				\item[(iv)] $\bm{\pi}(a|x)$ is continuous in $x$  and uniformly Lipschitz continuous in $x$ in the total variation distance, i.e., for each fixed $a$, $ \int_{\mathcal{A}} |\bm{\pi}(a|x) - \bm{\pi}(a|x')|\dd a \to 0$ as $ x' \to x$, and
				there is a constant $C>0$ independent of $a$ such that
				\[ \int_{\mathcal{A}} |\bm{\pi}(a| x) - \bm{\pi}(a| x')|\dd a \leq C|x-x'|,\;\;\forall x,x'\in \mathbb{R}^d. \]
			\end{enumerate}
			The collection of admissible time-invariant policies is denoted by $\bm\Pi_s$.
		\end{definition}

I state the parallel results to Theorem \ref{thm:q optimal} for the ergodic tasks about the characterization of the optimal value function, q-function, and the optimal value in Theorem \ref{thm:q optimal ergodic}.
\begin{theorem}
	\label{thm:q optimal ergodic}
	Let a function $\widehat{J^*}\in C^{2}\big(\mathbb{R}^d \big)$ with polynomial growth in its all derivatives, a continuous function $\widehat{q^*}: \mathbb{R}^d\times \mathcal{A}\to \mathbb{R}$, and a constant $\widehat{\beta^*}$ be given, satisfying
	\begin{equation}
		\label{eq:q hjb2 optimal ergodic}
		\lim_{T\to \infty}\E\left[ e^{\frac{\epsilon^2}{2} \langle \widehat{J^*}^{\bm \pi}\rangle(T) }  \right] < \infty, \text{ for any } \bm \pi\in \bm\Pi_s,\  \int_{\mathcal{A}} \exp\{ \frac{1}{\lambda} \widehat{q^*}(x,a) \} \dd a =1,\;\;\forall x\in \mathbb{R}^d.
	\end{equation}
	Moreover, assume there exists $\delta > 0$ such that $\lim_{T\to \infty}\frac{1}{T}\log\E^{\hat\p}\left[ e^{-\epsilon (1+\delta) \widehat{J^*}(X_T^{\bm \pi}) }  \right] = 0$ under any probability measure $\hat\p$ that is equivalent to the original probability $\p$, for all $\bm\pi\in \bm\Pi_s$.
	
	If there exists one $\bm\pi\in \bm\Pi_s$ such that for all initial state $x\in \mathbb{R}^d$, the following process
	\begin{equation}
		\label{eq:martingale with q function2 optimal ergodic}
		\widehat{J^*}({X}_t^{\bm\pi}) + \int_0^t \left\{ \left[ r({X}_{u}^{\bm\pi},a^{\bm\pi}_{u}) - \widehat{q^*}({X}_{u}^{\bm\pi},a^{\bm\pi}_{u}) - \widehat{\beta^*}  \right]\dd u + \frac{\epsilon}{2}\dd \langle \widehat{J^*}^{\bm\pi} \rangle(u)\right\}
	\end{equation}
	is an $(\{\f_t\}_{t\geq 0},\p)$-martingale. Then $\widehat{J^*}$, $\widehat{q^*}$, $\widehat{\beta^*}$ are respectively the optimal value function, the optimal q-function, and the optimal value. Moreover, in this case, $\widehat{\bm\pi^*}(a|x):= \exp\{  \frac{1}{\lambda}\widehat{q^*}(x,a) \}$ is the optimal policy.
\end{theorem}

As a consequence of Theorem \ref{thm:q optimal ergodic}, one may similarly make use of the martingale conditions to approximate the optimal q-function by stochastic approximation algorithm. Notice that for ergodic tasks, they are learned typically based on a single trajectory, and hence the algorithm is online. Algorithm \ref{algo:online ergodic} describes such an algorithm that can be applied both on-policy and off-policy.

\begin{algorithm}[htbp]
	\caption{Ergodic Risk-sensitive q-Learning Algorithm}
	\textbf{Inputs}: initial state $x_0$, time step $\Delta t$, initial learning rates $\alpha_{\theta},\alpha_{\psi},\alpha_{\beta}$ and learning rate schedule function $l(\cdot)$ (a function of time), functional forms of the parameterized value function $J^{\theta}(\cdot)$ and q-function $q^{\psi}(\cdot,\cdot)$ satisfying \eqref{eq:q hjb2 optimal ergodic}, functional forms of  test functions $\bm{\xi}(t,x_{\cdot \wedge t},a_{\cdot \wedge t})$, $\bm{\zeta}(t,x_{\cdot \wedge t},a_{\cdot \wedge t})$, temperature parameter $\lambda$, and risk sensitivity parameter $\epsilon$.

	\textbf{Required program (on-policy)}: an environment simulator $(x',r) = \textit{Environment}_{\Delta t}(x,a)$ that takes initial state $x$ and action $a$ as inputs and generates a new state $x'$ at $\Delta t$ and an instantaneous reward $r$ as outputs. Policy $\bm\pi^{\psi}(a|x) = \exp\{  \frac{1}{\lambda}q^{\psi}(x,a) \}$.
	
	\textbf{Required program (off-policy)}: observations $ \{a, r, x'\} = \textit{Observation}( x;\Delta t)$ including  the observed actions, rewards, and state when the current state is $x$ under the given behavior policy at the sampling time grids with step size $\Delta t$.

	\textbf{Learning procedure}:
	\begin{algorithmic}
		\State Initialize $\theta,\psi,\beta$. Initialize $k = 0$. Observe the initial state $x_0$ and store $x_{t_k} \leftarrow  x_0$.
		\Loop \State{
			\Comment{\textbf{On-policy case}
				
				Generate action $a\sim \bm{\pi}^{\psi}(\cdot|x)$.
				
				Apply $a$ to  environment simulator $(x',r) = Environment_{\Delta t}(x, a)$, and observe new state $x'$ and reward $r$ as outputs. Store $x_{t_{k+1}} \leftarrow x'$.
				
			}
			
			\Comment{\textbf{Off-policy case}
				
				Obtain one observation $a_{t_k}, r_{t_k}, x_{t_{k+1}} = \textit{Observation}(x_{t_k};\Delta t)$.

			}
			\Comment{\textbf{After obtaining a pair of samples}}
			
			Compute test functions $\xi_{t_k} = \bm{\xi}(t_k,x_{t_0},\cdots, x_{t_k},a_{t_0},\cdots, a_{t_k})$, $\zeta_{t_k} = \bm{\zeta}(t_k, x_{t_0},\cdots, x_{t_k},a_{t_0},\cdots, a_{t_k})$.		
			
			Compute
			\[\begin{aligned}
				& \delta = J^{\theta}(x') - J^{\theta}(x) + r\Delta t -q^{\psi}(x,a)\Delta t - \beta \Delta t + \frac{\epsilon}{2}\left[ J^{\theta}(x') - J^{\theta}(x)  \right]^2, \\
				& \Delta \theta = \xi_{t_k} \delta, \\
				& \Delta \beta = \delta, \\
				& \Delta \psi =\zeta_{t_k} \delta.
			\end{aligned}  \]
			
			Update $\theta$, $\beta$ and $\psi$ by
			\[ \theta \leftarrow \theta + l(k\Delta t)\alpha_{\theta} \Delta \theta,\]
			\[ \beta \leftarrow \beta + l(k\Delta t)\alpha_{\beta} \Delta \beta,\]
			\[ \psi \leftarrow \psi + l(k\Delta t)\alpha_{\psi}   \Delta \psi.  \]
			
			Update $x\leftarrow x'$ and $k \leftarrow k + 1$.
		}
		\EndLoop
	\end{algorithmic}
	\label{algo:online ergodic}
\end{algorithm}

\section{Applications}
\label{sec:numerical}
I illustrate our methods on two applications with synthetic data sets. To avoid the confusion between the time-index of a process and the number of iteration used in an algorithm, I write $X(t)$ to denote the time-$t$ value of a process $X$ in this section.
\subsection{Merton's investment problem with power utility}
\label{subsec:merton}
I consider the well-known Merton's investment problem (without consumption) with power utility \citep{merton1969lifetime}. The problem is formulated as follows: Consider a market with one (for simplicity) risky asset (stock) and one risk-free asset (bond). The price of the stock follows a geometric Brownian motion
\[ \frac{\dd S(t)}{S(t)} = \mu\dd t+ \sigma\dd W(t) , \]
and the bond price follows $\frac{\dd S^0(t)}{S^0(t)} = r\dd t$. 

There is an agent with investment horizon $[0,T]$ who determines the investment allocation between the stock and the bond. The proportion of wealth allocated to the stock is denoted by $a(t)$, and I use $X(t)$ to represent the value of a self-financing portfolio. Then $X(t)$ satisfies the wealth equation:
\begin{equation}
	\label{eq:state dynamics X}
	\begin{aligned}
		\dd X(t) = & X(t) a(t)\frac{\dd S(t)}{S(t)} + X(t)(1 - a(t)) \frac{\dd S^0(t)}{S^0(t)}\\
		= & [r + (\mu - r)a(t)] X(t)  \dd t + X(t) a(t)\sigma \dd W(t) .    
	\end{aligned}
\end{equation}
The objective function of the agent is to maximize the expected ``bequest utility'' on the terminal wealth: $\E\left[ U\left( X(T) \right)\right]$, where the utility function $U$ is assumed to take the form of a power function $U(x) = \frac{x^{1-\gamma}}{1-\gamma}$, with $\gamma > 0,\gamma \neq 1$. $\gamma$ here stands for the relative risk aversion coefficient. It is well-known that the optimal portfolio choice is to maintain a constant proportion of wealth in the risky asset with $a^* = \frac{\mu-r}{\gamma \sigma^2}$.

The problem for an RL agent is, to find the optimal policy to this problem without the knowledge of $\mu,\sigma$. ($\gamma$ is assumed to be known by the RL agent) At the first glance, the problem belongs to standard, non-risk-sensitive RL, and prior studies have studied this problem directly. However, they have been largely restricted to log-utility, e.g., \cite{dai2023learning,jiang2022reinforcement} due to the tractability of the associated exploratory stochastic control problems, which correspond to the limit of $\gamma \to 1$. Even for general power utilities, the associated exploratory stochastic control problems have to be specifically design to guarantee well-posedness and tractability, see discussions in \cite{dai2023recursive}.

However, this problem can be naturally embedded into a risk-sensitive RL framework by utilizing the transformation: $U(x) = \frac{1}{1-\gamma}e^{(1-\gamma)\log x}$, so that a power utility can be reviewed as the risk-sensitive counterpart of the log-returns of the portfolio, where $0\neq 1-\gamma < 1$ is the risk sensitivity coefficient. This connection has been noticed in the mean-variance analysis for the log-returns in \citet{dai2021dynamic} due to the close relation between a risk-sensitive problem and the mean-variance problem and their applications in portfolio management \citep{bielecki1999risk}.

Therefore, our algorithm aims to solve
\begin{equation}
	\label{eq:risk sensitive rl crra}
	\max_{\bm\pi}\frac{1}{1-\gamma}\log\E\left[ e^{(1-\gamma)\left[ \int_0^T -\lambda\log\bm\pi\left( a(t) | t, X(t)  \right)\dd t + \log X(T)  \right]} \right],
\end{equation} 
where the wealth process $X$ follows \eqref{eq:state dynamics X}, and the portfolio choices $a(t)$ is generated from $a(t)|t,X(t) \sim \bm\pi\left( \cdot|t, X(t) \right)$.

\subsubsection{Description of the algorithm}

As the ingredient of the q-learning algorithms, I approximate the (optimal) q-function and (optimal) value function by
\begin{equation}
	\label{eq:parameterization form}
	q(t,x,a;\bm\psi) = -\frac{(a - \psi_1)^2}{2\psi_2} - \frac{\lambda}{2}\log2\pi\lambda - \frac{\lambda}{2}\log\psi_2,\ V(t,x;\theta) = \log x + \theta (T-t) ,
\end{equation}
where $\bm\psi\in \mathbb{R}\times \mathbb{R}_+,\ \theta\in \mathbb R$ are parameters to be learned. Note that the parameterization forms \eqref{eq:parameterization form} are motivated by the ground truth solution, which corresponds to $\psi_1^* = \frac{\mu-r}{\gamma\sigma^2}$ and $\psi_2^* = \frac{1}{\gamma \sigma^2}$. See Appendix \ref{sec:merton true solution} for details. Since function $q$ does not depend on $(t,x)$, I simplify the notation as $q(a;\bm\psi):=q(t,x,a;\bm\psi)$. Moreover, $q(\cdot;\bm\psi)$ satisfies $\int_{\mathbb R} \exp\{ \frac{1}{\lambda} q(a;\bm\psi)  \} \dd a = 1$ is one critical constraint in the q-learning algorithm. Furthermore, given the learned q-function, it suggests the policy should be 
\[ \bm\pi(\cdot|\bm\psi) =  \exp\{ \frac{1}{\lambda} q(\cdot;\bm\psi) \} = \mathcal{N}( \psi_1, \lambda \psi_2) . \]

Suppose at the beginning of the $i$-th iteration ($i\geq 0$), I already have some guesses for the parameters, denoted by $\theta_i,\bm\psi_i$ which correspond to a policy $\bm\pi_i = \mathcal{N}(\psi_{1,i},\lambda\psi_{2,i})$. Then the q-learning algorithm aims to update them iteratively by incorporating the information implied in new observations. 
The increments for $\theta,\psi_1,\psi_2$ are
\begin{equation}
	\label{eq:update theta}
	\frac{2}{T^2}	\int_0^T \frac{\partial V}{\partial \theta}\left(t, X_i(t);\theta_i \right) \left[ \dd V\left(t , X_i(t);\theta_i  \right) - q\left(a_i(t) ;\bm\psi_i\right)\dd t + \frac{1-\gamma}{2}\dd \langle V\left(\cdot , X_i(\cdot);\theta_i  \right) \rangle  \right],
\end{equation}
\begin{equation}
	\label{eq:update psi 1}
	\frac{1}{T}	\int_0^T \frac{\partial q}{\partial \psi_1}\left(a_i(t);\bm\psi_i \right) \left[ \dd V\left(t , X_i(t);\theta_i  \right) - q\left(a_i(t) ;\bm\psi_i\right)\dd t + \frac{1-\gamma}{2}\dd \langle V\left(\cdot , X_i(\cdot);\theta_i  \right) \rangle  \right],
\end{equation}
and
\begin{equation}
	\label{eq:update psi 2}
	-\frac{1}{T}\int_0^T \frac{\partial q}{\partial \psi_2^{-1}}\left(a_i(t);\bm\psi_i \right) \left[ \dd V\left(t , X_i(t);\theta_i  \right) - q\left(a_i(t) ;\bm\psi_i\right)\dd t + \frac{1-\gamma}{2}\dd \langle V\left(\cdot , X_i(\cdot);\theta_i  \right) \rangle  \right],
\end{equation}
respectively, where $\{ a_i(t), X_i(t) \}_{0\leq t\leq T}$ are sample trajectories satisfying \eqref{eq:state dynamics X} under policy $\bm\pi_i$, that is, $a_i(t) \sim \bm\pi_i$. This setting is known as the \textit{on-policy} learning because the agent can indeed determine which policy is used to acquire new data.  Across different $i$'s, I assume the trajectories are independent. 

Besides the increments in \eqref{eq:update theta}, \eqref{eq:update psi 1}, and \eqref{eq:update psi 2}, I also need to incorporate a projection step to control the change rate of these parameters and also to ensure certain parameters fall into a suitable range (e.g., $\psi_2 > 0$). In particular, denote by $\{b_i\}_{i\geq 0},\{c_i\}_{i\geq 0}$ as two increasing, divergent sequences, by $\{a_{\theta,i}\}_{i\geq 0},\{a_{\psi,i}\}_{i\geq 0}$ as the step size sequences for updating $\theta$ and $\bm\psi$, respectively, and denote by $\Pi_B(\cdot)$ as the projection onto a closed-convex set $B$. Then the new parameters $\theta_{i+1},\bm\psi_{i+1}$ are defined via
\begin{equation}
	\label{eq:projection}
	\begin{aligned}
		\theta_{i+1} = & \Pi_{[-c_{i+1}, c_{i+1}]}\left( \theta_i + a_{\theta,i} \times \text{ sample (average) of \eqref{eq:update theta}} \right),\\
		\psi_{1,i+1} = & \Pi_{[-c_{i+1}, c_{i+1}]}\left( \psi_{1,i} + a_{\psi,i} \times \text{ sample (average) of \eqref{eq:update psi 1}} \right),\\
		\psi_{2,i+1} = & \Pi_{[b_{i+1}^{-1}, c_{i+1}]}\left(  \psi_{2,i} + a_{\psi,i} \times \text{ sample (average) of \eqref{eq:update psi 2}} \right) .
	\end{aligned}
\end{equation}
The full iterative procedure is summarized by Algorithm \ref{algo:offline episodic} by discretely sampling the action sequences and using finite sums to approximate the integrals in \eqref{eq:update theta}, \eqref{eq:update psi 1}, and \eqref{eq:update psi 2}. 

\subsubsection{Convergence of the algorithm}
\label{subsec:meton convergence results}
The analysis of this algorithm falls into the scope of the general stochastic approximation algorithms (cf. \cite{kushner2003stochastic}). The implementation in Algorithm  \ref{algo:offline episodic} essentially suggests that the time discretization is only necessary for computing integrals in \eqref{eq:update theta}, \eqref{eq:update psi 1}, and \eqref{eq:update psi 2} whose error is governed by the standard numerical analysis. Hence our interest is mainly to analyze the expectation and variance (conditioned on $\theta_i,\bm\psi_i$) of \eqref{eq:update theta}, \eqref{eq:update psi 1}, and \eqref{eq:update psi 2}, respectively, as in stochastic approximation algorithms. However, in our examples, the variance of these increments is not uniformly bounded, but growing. Hence the projection is necessary to ensure the variance does not grow too fast. The ideas of projection are borrowed from \citet{andradottir1995stochastic}, and the convergence rate for the recursion is borrowed from \citet{broadie2011general}.   

		\begin{theorem}
			\label{thm:convergence of merton}
			Suppose that the temperature parameter $\lambda$ is a positive constant along the learning procedure \eqref{eq:projection}, with time discretization step size $\{\Delta t_i\}_{i\geq 0}$. Suppose there are positive sequences $\{a_{\psi,i}\}_{i\geq 0}$, $\{b_i\}_{i\geq 0},\{c_i\}_{i\geq 0}$ satisfying
			\begin{enumerate}
				\item[(i)] $b_{i}\uparrow \infty,c_i\uparrow \infty,\Delta t_i \downarrow 0$;
				\item[(ii)] $\sum_{i=1}^{\infty}a_{\psi,i}b_i^{-1} = \infty$;
				\item[(iii)] $\sum_{i=1}^{\infty}(a_{\psi,i}^2  b_i^{4}  c_i^8  +  a_{\psi,i} \Delta t_i c_i^3) < \infty$.
			\end{enumerate}
			Then the learning procedure described by \eqref{eq:projection} will converge almost surely. Moreover, with suitable choices of $a_{\psi,n},\Delta t_n \sim \frac{1}{n}$, $b_n,c_n \sim \log n$, the mean squared error (MSE) of the parameters in the policy converges to zero at the rate $\E[|\bm \psi_{n} - \bm\psi^*|^2] =\tilde{O}(\frac{1}{n})$, where $\tilde{O}(\cdot)$ means the big O ignoring the logarithm factor.
		\end{theorem}

The error analyzed in Theorem \ref{thm:convergence of merton} is related to the statistical efficiency of the learned parameters, and the order $\tilde{O}(\frac{1}{n})$ almost matches the typical optimal convergence rate in any data-driven methods except for some logarithm factor. The extra logarithm factor is necessary to overcome the unbounded variance in the sampled process and the possible degeneracy of the q-function when the associated variance tends to zero. From its proof, one can see that if a suitable range for the parameters is known ex ante, then such logarithm factor can be avoided.

Next, I consider the error of the learned policy under a different metric: the suboptimal gap in terms of its performance. I measure the performance gap using the notion of \textit{equivalent relative wealth loss (ERWL)} defined in the following.

Under the policy $\bm\pi(\cdot|\bm\psi)= \mathcal{N}( \psi_1, \lambda \psi_2)$, I denote the risk-sensitive objective function \eqref{eq:risk sensitive rl crra} less the entropy regularization by
\begin{equation}
	\label{eq:risk sensitive objective merton explicit}
	\begin{aligned}
		J(x,\psi_1,\lambda \psi_2) = & \frac{1}{1-\gamma}\log\E\left[ e^{(1-\gamma) \log X(T) }\Big| X(0) = x \right] \\
		= & \log x + \left[ r + (\mu-r)\psi_1 - \frac{\gamma \sigma^2}{2}(\psi_1^2 + \lambda \psi_2) \right]T ,
	\end{aligned} 
\end{equation}
and denote the same quantity under the optimal (control) policy by $J^*(x)$.

Recall that in Merton's problem, objective is to maximize the bequest utility on the terminal wealth. Under the policy $\bm\pi(\cdot|\bm\psi)= \mathcal{N}( \psi_1, \lambda \psi_2)$, the associated value function at time 0 can be written as  
\[ \E\left[ U(X(T)) \Big| X(0) = x \right] = \frac{1}{1-\gamma} \log\left((1-\gamma) J(x,\psi_1,\lambda \psi_2) \right) . \]

I denote the ERWL associated with the policy $\bm\pi(\cdot|\bm\psi)= \mathcal{N}( \psi_1, \lambda \psi_2)$ by $\Delta(\psi_1,\lambda \psi_2)$, which is defined as the solution to 
\[ \frac{1}{1-\gamma} \log\left((1-\gamma) J(x,\psi_1,\lambda\psi_2) \right) = \frac{1}{1-\gamma} \log\left((1-\gamma) J^*(x(1-\Delta)) \right) . \]

The next theorem characterizes how fast the accumulated ERWL up to the first $N$ episodes grows in $N$, depending on whether a deterministic policy can be executed. If ERWL is diminishing, then the accumulated ERWL will grow in $N$ at a sublinear rate. The slower it grows, the more efficient the algorithm is. In particular, I demonstrate the regret that concerns the learned deterministic policy $\mathcal{N}(\psi_{1,i},0)$ as well as the learned stochastic policy $\mathcal{N}(\psi_{i,1},\lambda_i \psi_{i,2})$. Note that compared to the deterministic policy, a stochastic policy incurs more losses because it lacks exploitation, which results in larger regret. Here, two metrics concern the same learning algorithm, which has to rely on the stochastic policy in the learning process to ensure the convergence. Moreover, Theorem \ref{thm:regret merton} also illustrates the different requirements on tuning the learning rate and temperature sequences to achieve the desired regret orders. 

		\begin{theorem}
			\label{thm:regret merton}
			The ERWL is upper bounded by the suboptimal gap in the risk-sensitive objective, i.e., $\Delta(\psi_1,\lambda \psi_2) \leq J^*(x) -  J(x,\psi_1,\lambda \psi_2)$, which is independent of $x$. Moreover, suppose in the $i$-th episode, $\lambda_i$ is used in \eqref{eq:parameterization form}. 
			\begin{enumerate}
				\item[(a)] If the deterministic policy $\mathcal{N}(\psi_{1,i},0)$ is executed along the learning procedure \eqref{eq:projection}, then with suitable choices of sequences $a_{\psi,n},\Delta t_n \sim \frac{1}{n}$, $b_n,c_n \sim \log n$, $\lambda_n = \lambda$, the expected accumulated ERWL satisfies
				\[ \E\left[ \sum_{i=1}^N  \Delta(\psi_{1,i},0) \right] = O\left( (\log N)^4 \right) .\]
				\item[(b)] If the stochastic policy $\mathcal{N}(\psi_{i,1},\lambda_i \psi_{i,2})$ is executed along the learning procedure \eqref{eq:projection}, then with suitable choices of sequences $a_{\psi,n},\Delta t_n \sim \frac{1}{\sqrt{n}}$, $b_n = O(1)$, $c_n \sim \log n$, $\lambda_n \sim \frac{1}{\sqrt{n}}$, the expected accumulated ERWL satisfies
				\[ \E\left[ \sum_{i=1}^N  \Delta(\psi_{1,i},\lambda_i \psi_{2,i}) \right] = \tilde{O}(\sqrt{N}) .\]
			\end{enumerate}
		\end{theorem}

The first part in Theorem \ref{thm:regret merton} says that the ERWL is upper bounded by the suboptimal gap in terms of the risk-sensitive objective, which is often known as the \textit{regret} in the RL literature. Note that a stochastic policy introduces more uncertainty into the state process, and hence, is detrimental to the performance metric. If one has to execute a stochastic policy, to have diminishing ERWL, one has to tune $\lambda$ appropriately and enforce that the variance of the policy tends to zero. Based on the expression \eqref{eq:risk sensitive objective merton explicit}, one can see that the ERWL depends on the squared error of $(\psi_1 - \psi_1^*)^2$, but only linearly in $\lambda$. This difference implies the different order in the accumulated ERWL. The square-root order for the stochastic policy case matches the optimal regret bound for tabular, episodic Q-learning for MDP in \citet{jin2018q} and risk-sensitive Q-learning for MDP in \citet{fei2020risk}.

\subsubsection{Numerical results}
In the numerical experiments, the model configurations are $\sigma=0.3$, $\mu=0.1$, $r=0.02$, $\gamma=2$, $T=1$, $x_0=1$, the time-discretization size is $\Delta t=0.01$. The number of episodes in each simulation run is $10^5$. 

I use two sets of tuning parameters that correspond to two situations in Theorem \ref{thm:regret merton} to illustrate. In the first setting, the temperature parameter is a fixed constant, and we take it as $\lambda=3,1,0.3$, and the step size $a_{\psi,n}$ decays as $n^{-1}$ and $b_n,c_n$ grows at the rate $\log n$. In the second setting, the temperature parameter decays as $\lambda n^{-1/2}$, and the step size $a_{\psi,n}$ decays as $n^{-1/2}$, $b_n$ is a fixed small constant, and $c_n$ grows at the rate $\log n$. I repeat the simulation runs for 1000 times to estimate the mean squared error or the mean equivalent relative wealth of the learning algorithm.

Figure \ref{fig:merton 1} illustrates the performance of the algorithm under the first tuning parameters setup. This setup satisfies the conditions in Theorem \ref{thm:convergence of merton} and Theorem \ref{thm:regret merton}~(a). The left two panels show the convergence rate of MSE almost match the theoretical results. The accumulated ERWL on the right panel (in the log-log scale) does not grow linearly. In particular, when $\lambda$ is too small ($\lambda=0.3$), the algorithm incurs larger error in the long run; whereas, the rate of convergence of the mean squared error of the learned $\psi_{1,n}$ is the same. The reason that the mean squared error of the learned $\psi_{2,n}$ does not exhibit the same convergence rate when $\lambda=0.3$ is either because the initial learning rate is not large enough or the number of episodes is still large enough to visualize the order.

Figure \ref{fig:merton 2} illustrates the performance of the algorithm under the first tuning parameters setup. This setup satisfies the conditions in Theorem \ref{thm:regret merton}~(b). It is interesting to notice that here that the MSE of $\psi_{1,n} - \psi_1^*$ also decays almost at the rate of $n^{-1}$, however, in the proof, this rate is only $\tilde{O}(n^{-1/2})$.\footnote{Proving a faster convergence rate of MSE is possible by a more careful analysis on the variance of the increment term, which would also decays in $\lambda_n^2$.} A faster convergence rate in the mean of the learned policy cannot further improve the order of the accumulated ERWL because the variance in the learned policy cannot be further reduced. When the temperature parameter also decays in the number of episode, the value of initial temperature also has similar impact as in Figure \ref{fig:merton 1}. Moreover, in this case, the accumulated ERWL when initial temperature is lower also tends to be smaller. This is because in this case, ERWL is calculated for the stochastic policy; and lower temperature implies lower variance in the policy, and hence improves its performance. 

\begin{figure}[htbp]
	\centering
	\begin{subfigure}{0.9\textwidth}
		\centering
		\includegraphics[width=0.55\textwidth]{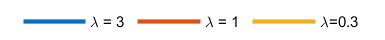}
	\end{subfigure}
	\begin{subfigure}{0.32\textwidth}
		\includegraphics[width=1\textwidth]{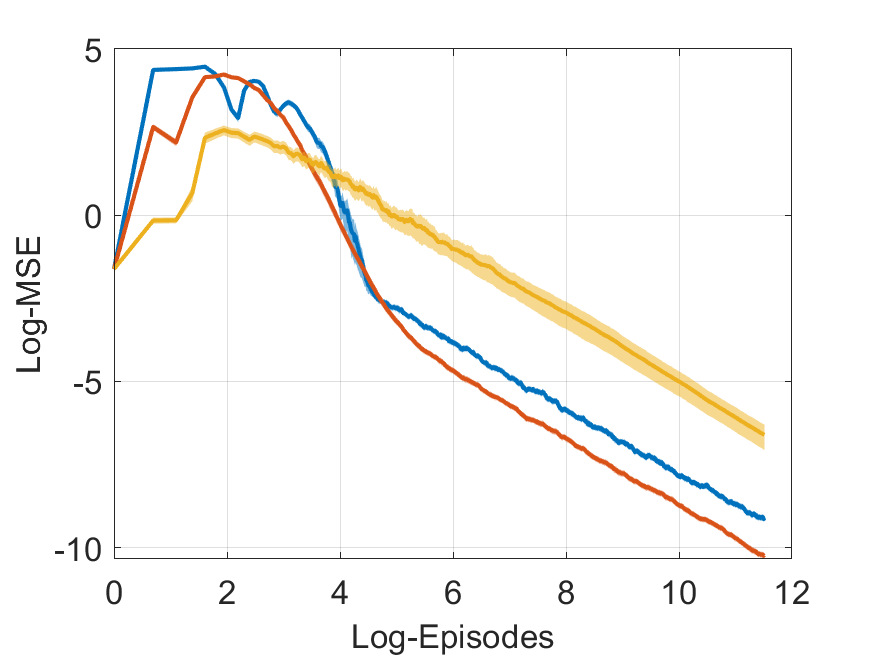}
		\caption{The mean squared error of the learned $\psi_1$.}
	\end{subfigure}
	\begin{subfigure}{0.32\textwidth}
		\includegraphics[width=1\textwidth]{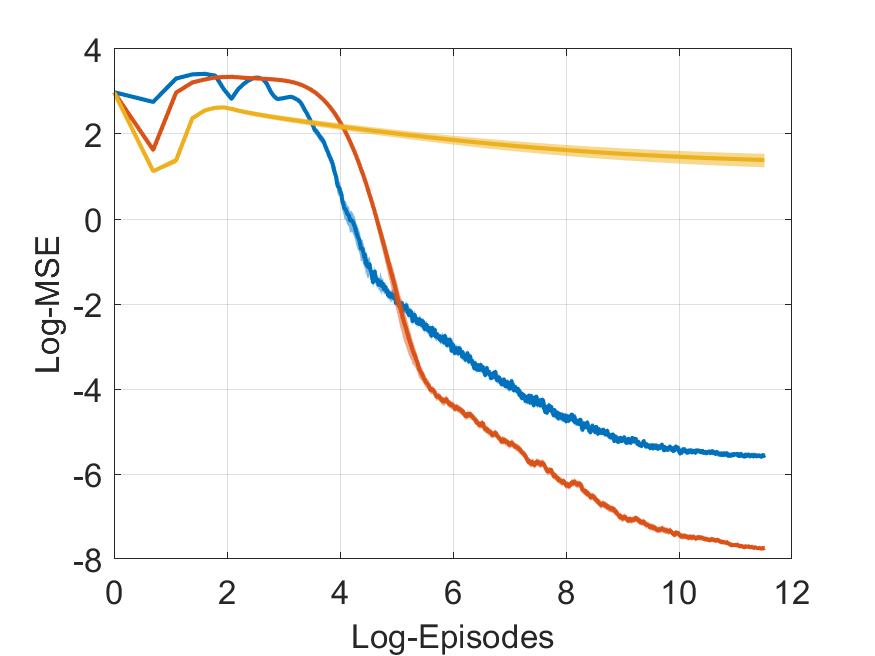}
		\caption{The mean squared error of the learned $\psi_2$.}
	\end{subfigure}
	\begin{subfigure}{0.32\textwidth}
		\includegraphics[width=1\textwidth]{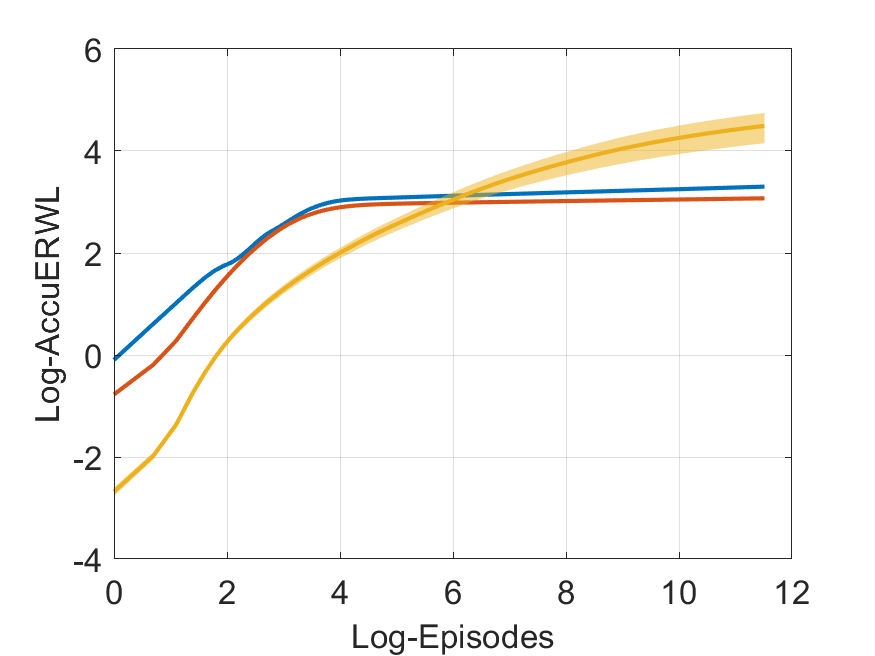}
		\caption{The mean accumulated ERWL of $\mathcal N(\psi_{1,n},0)$.}
	\end{subfigure}
	\caption{The illustration of the learned policy. The left two panels show the convergence of the mean squared error of the learned parameters in the policy, and the right panel shows the mean equivalent relative wealth loss of the learned deterministic policy. Both panels are in the log-scales. The results are based on simulated data with 1000 runs. The shaded area indicates twice the standard deviation of the estimated expectation. The temperature parameters are taken as $\lambda=3,1,0.3$, respectively, and the learning rate $a_{\psi, n} = 10/(1+n)$. The number of episodes within each simulation run is $10^5$.}
	\label{fig:merton 1}
\end{figure}

\begin{figure}[htbp]
	\centering
	\begin{subfigure}{0.9\textwidth}
		\centering
		\includegraphics[width=0.55\textwidth]{Legend_1.png}
	\end{subfigure}
	\begin{subfigure}{0.32\textwidth}
		\includegraphics[width=1\textwidth]{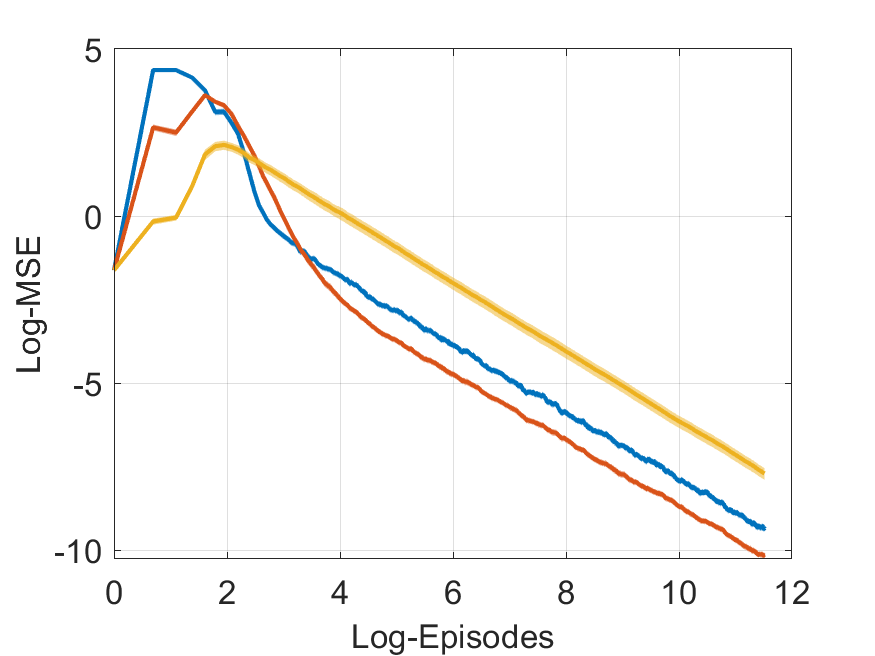}
		\caption{The mean squared error of the learned $\psi_1$.}
	\end{subfigure}
	\begin{subfigure}{0.32\textwidth}
		\includegraphics[width=1\textwidth]{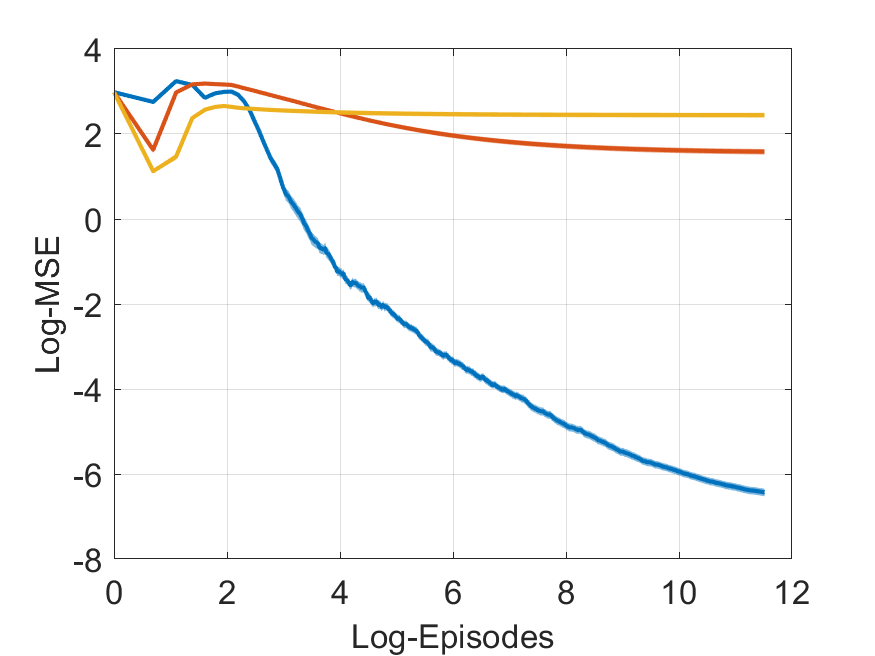}
		\caption{The mean squared error of the learned $\psi_2$.}
	\end{subfigure}
	\begin{subfigure}{0.32\textwidth}
		\includegraphics[width=1\textwidth]{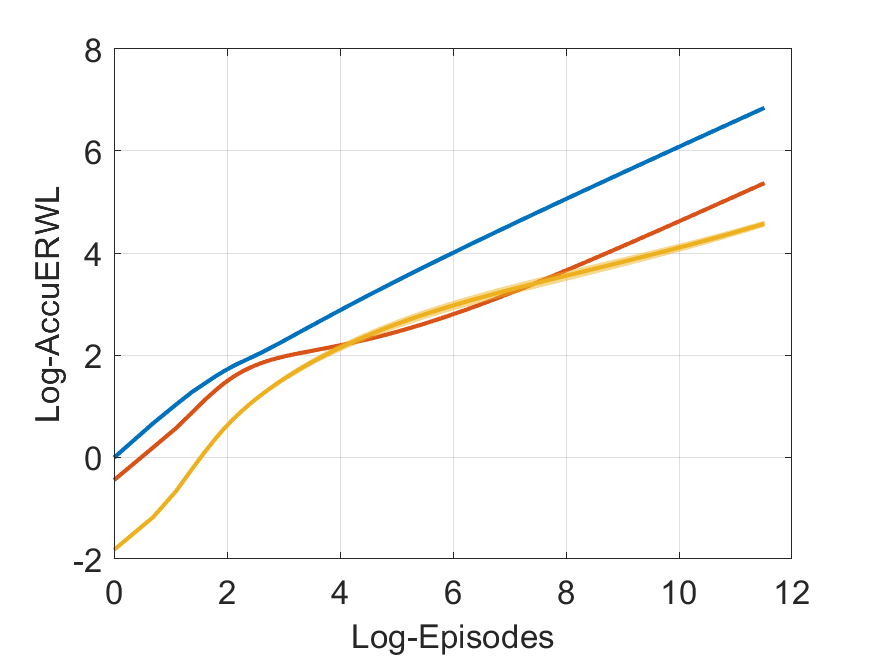}
		\caption{The mean accumulated ERWL of $\mathcal N(\psi_{1,n},\lambda_n\psi_{2,n})$.}
	\end{subfigure}
	\caption{The illustration of the learned policy. The left two panels show the convergence of the mean squared error of the learned parameters in the policy, and the right panel shows the mean equivalent relative wealth loss of the learned deterministic policy. Both panels are in the log-scales. The results are based on simulated data with 1000 runs. The shaded area indicates twice the standard deviation of the estimated expectation. The temperature parameter is taken as $\lambda_n=3/\sqrt{n+1},1/\sqrt{n+1},0.3/\sqrt{n+1}$, respectively, and the learning rate $a_{\psi, n} = 10/\sqrt{n+1}$. The number of episodes within each simulation run is $10^5$.}
	\label{fig:merton 2}
\end{figure}


\begin{algorithm}[htbp]
	\caption{Offline--Episodic On-policy q-Learning Algorithm for Merton's Investment Problem}
	\textbf{Inputs}: investment horizon $T$, time step size $\Delta t$, number of episodes $N$, number of time grids $K$, risk aversion coefficient $\gamma$, temperature parameter $\lambda$, learning rates schedules $a_{\theta,i},a_{\psi,i} \downarrow 0$, projection region schedules $b_i,c_i \uparrow \infty$, $j$ stands for the number of episodes.

	\textbf{Required program}: observations under the behavioral policy $ \{a(t_k), R(t_{k})\}_{k = 0,\cdots, K-1}\cup \{ R(t_K)\} = \textit{Obervation}(\Delta t)$ that returns the observed actions and log-returns trajectories under the behavioral policy at the sampling time grids with step size $\Delta t$.

	\textbf{Learning procedure}:
	\begin{algorithmic}
		\State Initialize $\theta_1,\bm\psi_1$.
		\For{episode $i=1$ to $N$} 
		
		\State{
			Use policy $\bm\pi(\cdot|\bm\psi_i)$ to obtain one observation $\{a_i(t_k), R_i(t_{k})\}_{k = 0,\cdots, K-1}\cup \{ R_i(t_K)\} = \textit{Obervation}(\Delta t)$.
			
			Compute
			\[ \delta_{t_k} = V\left( t_{k+1},R_i(t_{k+1}) ;\theta_i \right) - V\left(t_k, R_i(t_k);\theta_i \right) - q\left( t_{k},R_i(t_{k}),a_i(t_k);\bm\psi_i \right) \Delta t ,  \]
			for $k = 0,\cdots,K-1$.
			\[ \Delta \theta = \frac{2}{T^2}\sum_{k=0}^{K-1} \frac{\partial V}{\partial \theta}\left(t_k, R_i(t_k);\theta_i \right) \delta_{t_k}, \]
			\[
			\Delta \psi_1 = \frac{1}{T} \sum_{k=0}^{K-1} \frac{\partial q}{\partial \psi_1}\left(t_k, R_i(t_k),a_i(t_k);\bm\psi_i \right) \delta_{t_k}. 
			\]
			\[ \Delta \psi_2 =  -\frac{1}{T}\sum_{k=0}^{K-1} \frac{\partial q}{\partial \psi_2^{-1}}\left(t_k, R_i(t_k),a_i(t_k);\bm\psi_i \right) \delta_{t_k}. 
			\]
			
			Update $\theta$ and $\psi$ by
			\[ \theta_{i+1} = \Pi_{[-c_{i+1}, c_{i+1}]}\left( \theta_i + a_{\theta,i} \Delta \theta \right) .\]
			\[ \psi_{1,i+1} = \Pi_{[-c_{i+1}, c_{i+1}]}\left( \psi_{1,i} + a_{\psi,i} \Delta \psi_1 \right).  \]
			\[ \psi_{2,i+1} = \Pi_{[b_{i+1}^{-1}, c_{i+1}]}\left(  \psi_{2,i} + a_{\psi,i} \Delta \psi_2 \right).  \]
			
		}
		\EndFor
	\end{algorithmic}
	\label{algo:offline episodic}
\end{algorithm}

\subsection{Off-policy linear-quadratic control}
I consider the commonly adopted linear-quadratic (LQ) control problem:
\begin{equation}
	\label{eq:lq dynamics}
	\dd X(t) = (AX(t) + Ba(t))\dd t + (CX(t) + D a(t))\dd W(t),
\end{equation}
where $X(t)$ stands for the state variable, and $a(t)$ stands for the control taken at time $t$. The ultimate goal is to maximize the long term average quadratic payoff
\begin{equation}
	\label{eq:lq non risk sensitive}
	\liminf_{T\to \infty}\frac{1}{T}\E\left[\int_0^T r(X(t),a(t))\dd t | X_0 = x_0 \right],
\end{equation}
with $r(x,a) = -(\frac{M}{2}x^2 + Rxa + \frac{N}{2}a^2 + Px + Qa)$. The optimal solution to this classical problem can be found in Appendix \ref{sec:lq true solution}, where the optimal (feedback) control is a linear function of the state that is time-invariant. 

The risk-sensitive RL counterpart to this LQ problem is to consider
\begin{equation}
	\label{eq:lq risk sensitive rl}
	\max_{\bm\pi}\liminf_{T\to \infty} \frac{1}{\epsilon T}\log\E^{\p}\left[ e^{\epsilon  \left[ \int_0^T r\left( X^{\bm\pi}(s), a^{\bm\pi}(s) \right) - \lambda \log\bm\pi\left(a^{\bm\pi}(s) | X^{\bm\pi}(s)\right)    \dd s \right]  } \right] .
\end{equation}
Here $\lambda > 0$ is the temperature parameter that determines the strength of randomization, and typically $\epsilon < 0$ is the risk sensitivity coefficient. The optimal solution to the non-risk-sensitive RL problem ($\epsilon = 0$) can be found in \citet[Appendix B2]{jia2022policypg}, where the optimal policy is a Gaussian distribution, whose mean coincides with the classical solution ($\lambda = 0$), and whose variance is a constant. However, the risk-sensitive RL problem cannot be solved analytically.

In this subsection, I assume that learning can only be conducted in the off-policy manner. Specifically, the available data consists of observations of one trajectory of state, action and reward generated under a behavior policy that is usually not optimal, denoted by $\bm\pi^b(\cdot|x)$. That is, the RL agent has access to a finite sequence of observations: $\{t_k,X(t_k),a(t_k),\mathcal R(t_k)\}_{0\leq k\leq K}$ at certain sampling times $0 = t_0<t_1<\cdots<t_K = T$, where the state is generated by the behavior policy according to \eqref{eq:lq dynamics}, $a(t)|X(t) \sim \bm\pi^b(\cdot| X(t))$, and the observed reward $r(t_k)$ can be an unbiased but noisy realization of its expectation value, i.e., $\E\left[ \mathcal R(t_k)|a(t_k), X(t_k) \right] = r(X(t_k), a(t_k))$. Moreover, the RL agent does not have the knowledge about the environment (coefficients $A,B,C,D,M,N,R,P,Q$, and even the LQ structure), nor the knowledge of the behavior policy $\bm\pi^b$. 

\subsubsection{Description of the algorithm}
The (optimal) q-function and (optimal) value function are approximated by
\begin{equation}
	\label{eq:parameterization lq}
	\begin{aligned}
		q(X,a;\bm\psi) = & -\frac{\left( a - \psi_1 X - \psi_2  \right)^2}{2\psi_3} - \frac{\lambda}{2}\log2\pi\lambda - \frac{\lambda}{2}\log\psi_3, \\
		V(X;\bm\theta) = & \theta_2 x^2 + \theta_1 x .
	\end{aligned}
\end{equation}
The value function does not contain a constant term because it is only unique up to a constant. In addition, one also needs to learn another scalar parameter, denoted by $\theta_0$. To sum-up, $\bm\psi\in \mathbb{R}^2\times \mathbb R_+$ and $\bm\theta\in \mathbb R^3$ need to be learned. The form \eqref{eq:parameterization lq} corresponds to a policy $\bm\pi(\cdot|X;\bm\psi) = \mathcal N(\psi_1 X + \psi_2,\lambda\psi_3)$. 

I highlight several distinctions between the off-policy and on-policy learning, and to explain the motivation for the choice of parameterzation and our performance metrics. In the off-policy setting, the data is exogenously given by a behavioral policy $\bm\pi^b$. Hence, unlike the on-policy learning, the agent does not interact with the environment and generate data, instead, the agent can only use existing data. Therefore, in the off-policy learning is essentially a statistical problem to estimate the optimal policy. The ultimate goal is still to solve the original LQ problem \eqref{eq:lq non risk sensitive}, and the risk sensitive objective \eqref{eq:lq risk sensitive rl} is a formulation that is introduced to address the discrepancy between the limited data and the true environment distribution (e.g., recall the distributional robust interpretation for such formulation discussed in Section \ref{sec:classical formulation}). The form \eqref{eq:parameterization lq} is motivated by the analytical solution to the non-risk-sensitive RL problem, see in \citet[Appendix B2]{jia2022policypg}.

\subsubsection{Numerical results}
In the numerical experiments, the configurations are $A = -2$, $B=1$, $C=0.25$, $D = 1$, $M=N=Q=2$, $R=P=1$. The behavior policy is taken as state-independent, i.e., $\bm\pi^b(\cdot|X) = \mathcal{N}(0,1)$, and the initial state is $x_0=0$. The time discretization for the sample is $\Delta t=0.01$. The sample reward is generated as $\mathcal R | a, X \sim \mathcal N\left( r(X,a), \Delta t \right)$. I mainly examine the effect of the sample size and the risk sensitivity coefficient. Hence, I choose the temperature parameter $\lambda\in \{3,1,0.3\}$, and generate sample with length $T\in \{1,10,100\}$ and choose the risk sensitivity coefficient $\epsilon\in \{0,-0.1,-0.5,-1,-2,-5,-10,-20\}$. In each simulation run, I use the behavior policy to generate sample with length $T$, and then conduct learning algorithms to obtain the learned parameters $\hat{\bm\psi},\hat{\bm{\theta}}$, and finally compute its distance to the optimal policy to the problem \eqref{eq:lq non risk sensitive}. The simulation runs are repeated for $10^4$ times to evaluate the mean squared error (MSE). 

The results are presented in Table \ref{tab:lq offpolicy}. The top, middle, bottom panel stands for different sample sizes, and subtables (a), (b) and (c) correspond to different temperature parameters. In each panel, notice that solving the non-risk-sensitive objective does not lead to the most accurate estimate for the policy. With suitable values of $\epsilon$, solving the risk-sensitive objective reduces MSE. As $\epsilon$ approaches 0, the results get closer to the non-risk-sensitive objective, whereas overly large $\epsilon$ (in the absolute value) harms the accuracy. Overall, as the sample size increases, the performance get significantly better. The optimal $\epsilon$ is between $-5$ and $-2$ when the sample size is small, and is between $-2$ and $-1$ when the sample size is large. It confirms the intuition that large sample size provides a better approximation to the ground truth distribution, and hence the agent does not have to be too risk sensitive. In addition, a smaller temperature parameter $\lambda$ seems to induce an overall smaller estimation error in the off-policy, offline setting. This is probably because, unlike the on-policy, online learning, in this setting, a data set has been generated and fixed, and thus, the benefits of exploration are limited; instead, a problem with larger temperature deviates more from the original objective function. 
\begin{table}[htbp]
	\centering
	\begin{subtable}{0.9\textwidth}
		\centering
		\begin{tabular}{cccccccccc}
			\toprule
			&  $\epsilon$     & $0$     & $-0.1$  & $-0.5$  & $-1$    & $-2$    & $-5$    & $-10$   & $-20$ \\
			\midrule
			$T=1$   & MSE of $\hat\psi_1$   & 2.3731 & 2.3218 & 2.1051 & 1.7744 & 1.3750 & 0.9908 & 0.8914 & 1.0191 \\
			& MSE of $\hat\psi_2$   & 0.3960 & 0.4287 & 0.3055 & 0.2514 & 0.2383 & 0.3082 & 0.4971 & 0.7456 \\
			& MSE of $\hat\psi_3$   & 0.4760 & 0.4631 & 0.2284 & 0.1211 & 0.1007 & 0.1243 & 0.2338 & 0.3388 \\
			&       &       &       &       &       &       &       &       &  \\
			$T=10$  & MSE of $\hat\psi_1$   & 0.3077 & 0.2696 & 0.1473 & 0.1057 & 0.1272 & 0.4491 & 1.2589 & 2.9190 \\
			& MSE of $\hat\psi_2$   & 0.0550 & 0.0529 & 0.0294 & 0.0322 & 0.0790 & 0.3864 & 1.1444 & 2.6676 \\
			& MSE of $\hat\psi_3$   & 0.0385 & 0.0200 & 0.0043 & 0.0098 & 0.0305 & 0.1922 & 0.6222 & 1.4192 \\
			&       &       &       &       &       &       &       &       &  \\
			&       &       &       &       &       &       &       &       &  \\
			$T=100$ & MSE of $\hat\psi_1$   & 0.0186 & 0.0178 & 0.0156 & 0.0138 & 0.0116 & 0.0360 & 0.4659 & 3.2283 \\
			& MSE of $\hat\psi_2$   & 0.0067 & 0.0061 & 0.0047 & 0.0037 & 0.0031 & 0.0322 & 0.4406 & 2.9502 \\
			& MSE of $\hat\psi_3$  & 0.0004 & 0.0004 & 0.0003 & 0.0004 & 0.0005 & 0.0151 & 0.2280 & 1.3929 \\
			\bottomrule
		\end{tabular}%
		\caption{When the temperature parameter $\lambda = 3$.}
	\end{subtable}
	\begin{subtable}{0.9\textwidth}
		\centering
		\begin{tabular}{cccccccccc}
			\toprule
			&  $\epsilon$     & $0$     & $-0.1$  & $-0.5$  & $-1$    & $-2$    & $-5$    & $-10$   & $-20$ \\
			\midrule
			$T=1$   & MSE of $\hat\psi_1$   & 0.923 & 0.913 & 0.905 & 0.832 & 0.801 & 1.003 & 1.486 & 2.336 \\
			& MSE of $\hat\psi_2$   & 0.332 & 0.350 & 0.292 & 0.324 & 0.426 & 0.777 & 1.356 & 2.266 \\
			& MSE of $\hat\psi_3$   & 0.313 & 0.326 & 0.201 & 0.153 & 0.170 & 0.328 & 0.607 & 1.113 \\
			&       &       &       &       &       &       &       &       &  \\
			$T=10$  & MSE of $\hat\psi_1$   & 0.038 & 0.037 & 0.036 & 0.038 & 0.075 & 0.287 & 1.065 & 2.876 \\
			& MSE of $\hat\psi_2$   & 0.012 & 0.012 & 0.012 & 0.014 & 0.044 & 0.258 & 1.063 & 3.053 \\
			& MSE of $\hat\psi_3$   & 0.003 & 0.003 & 0.002 & 0.003 & 0.013 & 0.090 & 0.328 & 1.121 \\
			&       &       &       &       &       &       &       &       &  \\
			&       &       &       &       &       &       &       &       &  \\
			$T=100$ & MSE of $\hat\psi_1$   & 0.0042 & 0.0041 & 0.0040 & 0.0038 & 0.0039 & 0.2119 & 2.2732 & 10.8350 \\
			& MSE of $\hat\psi_2$   & 0.0027 & 0.0027 & 0.0027 & 0.0029 & 0.0045 & 0.2200 & 2.4270 & 12.3801 \\
			& MSE of $\hat\psi_3$   & 0.0002 & 0.0003 & 0.0003 & 0.0004 & 0.0008 & 0.0565 & 0.7789 & 4.5567 \\
			
			\bottomrule
		\end{tabular}%
		\caption{When the temperature parameter $\lambda = 1$.}
	\end{subtable}
	\begin{subtable}{0.9\textwidth}
		\centering
		\begin{tabular}{cccccccccc}
			\toprule
			&  $\epsilon$     & $0$     & $-0.1$  & $-0.5$  & $-1$    & $-2$    & $-5$    & $-10$   & $-20$ \\
			\midrule
			$T=1$   & MSE of $\hat\psi_1$   & 0.2989 & 0.2989 & 0.2988 & 0.2988 & 0.2987 & 0.3097 & 0.3788 & 0.5657 \\
			& MSE of $\hat\psi_2$    & 0.0491 & 0.0492 & 0.0494 & 0.0498 & 0.0505 & 0.0668 & 0.1329 & 0.3346 \\
			& MSE of $\hat\psi_3$   & 0.0556 & 0.0556 & 0.0556 & 0.0556 & 0.0556 & 0.0604 & 0.0853 & 0.1520 \\
			&       &       &       &       &       &       &       &       &  \\
			$T=10$  & MSE of $\hat\psi_1$   & 0.2071 & 0.2070 & 0.2066 & 0.2061 & 0.2051 & 0.2036 & 0.2299 & 0.3840 \\
			& MSE of $\hat\psi_2$    & 0.0172 & 0.0172 & 0.0172 & 0.0173 & 0.0175 & 0.0191 & 0.0439 & 0.2008 \\
			& MSE of $\hat\psi_3$   & 0.0299 & 0.0298 & 0.0297 & 0.0295 & 0.0292 & 0.0287 & 0.0289 & 0.0612 \\
			&       &       &       &       &       &       &       &       &  \\
			&       &       &       &       &       &       &       &       &  \\
			$T=100$ & MSE of $\hat\psi_1$   & 0.0031 & 0.0031 & 0.0030 & 0.0029 & 0.0050 & 0.3170 & 4.1180 & 15.1154 \\
			& MSE of $\hat\psi_2$    & 0.0023 & 0.0023 & 0.0023 & 0.0026 & 0.0066 & 0.4215 & 4.7278 & 17.6495 \\
			& MSE of $\hat\psi_3$   & 0.0002 & 0.0002 & 0.0003 & 0.0004 & 0.0015 & 0.0963 & 1.1864 & 6.2771 \\
			\bottomrule
		\end{tabular}%
		\caption{When the temperature parameter $\lambda = 0.3$.}
	\end{subtable}
	\caption{The mean squared error (MSE) of the estimated parameters in the policy. Each column contains the MSE of different values of the risk-sensitive coefficient. Larger $\epsilon$ (in the absolute value) means larger deviation from the sample distribution. $\epsilon=0$ means non-risk-sensitive objective. The top, middle, bottom panel stands for different sample sizes. The simulation runs are repeated for $10^4$ times to evaluate MSE. }
	\label{tab:lq offpolicy}%
\end{table}%

\section{Discussions}\label{sec:conclusion}
In this paper, I studied continuous-time risk-sensitive RL in the exponential form. By leveraging exponential martingale properties, I transformed the risk-sensitive objective into a standard additive form with an additional quadratic variation penalty. I introduced the q-function in this setting and showed that the optimal q-function and value function can still be jointly characterized through a linear martingale approach. However, unlike in the risk-neutral case, the usual connection between q-learning and policy gradient breaks down due to the inherent nonlinearity of the value function.

In contrast to discrete-time risk-sensitive problems, where an equivalent simplification is lacking and algorithms tend to be more complex, the continuous-time formulation offers significant advantages. The key difference lies in the treatment of transitions: in discrete time, the transition probabilities remain unknown and difficult to learn, leading to nonlinear recursions and the challenging multiplicative Bellman equation. In contrast, the continuous-time setting allows for a transformation that avoids these complications both theoretically and numerically. Theoretically, it replaces the multiplicative Bellman equation with a simpler additive structure, and numerically, it eliminates the need to compute exponentials of potentially large values. This advantage stems from the fact that in continuous time, a normal approximation holds incrementally, enabling a more tractable and efficient learning process.

I demonstrate the proposed algorithm on two applications: Merton's investment problem with power utility functions under the Black-Scholes model, and the off-policy learning for a linear-quadratic problem. These two serve as simple test cases where the desired solutions are known to take linear or even constant forms. However, real-world problems often involve complex, unknown q-functions and policies, requiring function approximation methods such as neural networks. This introduces approximation errors that can impact performance. Moreover, the complicated functional forms may lead to non-convex optimization landscape, increasing numerical instability. These challenges, which are not captured in my experiments, highlight the need for further research to assess the algorithm's robustness and scalability in more complex, high-dimensional settings.

The analysis of RL algorithm in this paper is restricted to simple cases, but it provides key insights into the role of the temperature $\lambda$ and the entropy regularization in the learning procedure. In the context of Merton's investment problem, I show that the product between the temperature parameter and the step size acts as an ``effective learning rate'' under the stochastic approximation framework. This reveals a fundamental tradeoff: while lower temperatures slow down learning and, in the extreme case of zero temperature, eliminate effective learning signals, they also reduce randomness in the policy, leading to less noisy updates. Moreover, when the temperature and step size decay at compatible rates, the algorithm achieves the desired convergence rate. This analysis serves as a first step in understanding the impact of the temperature parameter on learning algorithms while explicitly accounting for sampling error -- an aspect that appears to be overlooked in prior literature. The approach in this paper sets an example for proving the convergence of stochastic-approximation-based algorithms for continuous-time, continuous-space RL. However, extending this methodology to more complex settings remains challenging, because verifying the typical convergence conditions in general cases is nontrivial. These open questions highlight important directions for future research.

The numerical study on off-policy linear-quadratic problems provides valuable insights into the role of the risk sensitivity coefficient $\epsilon$. The results suggest that when the sample size is limited, incorporating a risk-sensitive objective can improve learning accuracy compared to a risk-neutral approach. Moreover, the optimal choice of $\epsilon$ depends on the sample size, as reflected in the connection between risk-sensitive RL and distributionally robust RL. Specifically, introducing risk sensitivity helps mitigate the discrepancy between a finite data set and the underlying population distribution. As a direction for future work, it would be valuable to establish theoretical guarantees on how risk-sensitive formulations address distributional shift and to develop data-driven methods for selecting the risk sensitivity coefficient.


There are alternative formulations of risk sensitivity other than the exponential form proposed in discrete-time systems, e.g., \cite{xu2023regret,wu2023risk}, however, it still seems unclear what the continuous-time counterpart is and if any form of penalty can be introduced to transform the nonlinear objective to a recursive utility maximization problem. How to address the distributional robustness other than the KL divergence in the continuous-time diffusion processes in a tractable way and to conduct RL still remains largely an open question for future research. 

\section*{Acknowledgement}
The author is supported by the Start-up Fund at The Chinese University of Hong Kong and Hong Kong Research Grants Council (RGC) - Early Career Scheme (ECS) 24211124. I thank the participants at the ETH-Hong Kong-Imperial Mathematical Finance Workshop, the 1st INFORMS Conference on Financial Engineering and FinTech, and 2025 SIAM Conference on Financial Mathematics and Engineering for their helpful discussions. I also thank two anonymous reviewers for their constructive and detailed comments that have led to an improved version of this paper.

\newpage
\bibliography{reference}

\newpage

\appendix
\section{Ground Truth Solutions in Examples}
\subsection{Merton's investment problem with power utility}
\label{sec:merton true solution}
In the classical Merton's investment problem (without consumption) with power utility \citep{merton1969lifetime}, it is well-known that the optimal portfolio choice is to maintain a constant proportion of wealth in the risky asset with $a^* = \frac{\mu-r}{\gamma \sigma^2}$. The proof of this statement is omitted.

I give the ground truth solution to the exploratory problem \eqref{eq:risk sensitive rl crra} as follows using a guess-and-verify approach. I claim the optimal policy is $\bm\pi^* = \mathcal{N}(\frac{\mu-r}{\gamma\sigma^2}, \frac{\lambda}{\gamma \sigma^2})$, optimal q-function is $q^*(t,x,a) = -\frac{\gamma\sigma^2}{2}(a - \frac{\mu-r}{\gamma\sigma^2})^2 - \frac{\lambda}{2}\log\frac{2\pi\lambda}{\gamma\sigma^2}$, and  the optimal value function is $J^*(t,x) = \log x + (T-t)[r + \frac{(\mu-r)^2}{2\gamma\sigma^2} + \frac{\lambda}{2}\log\frac{2\pi\lambda}{\gamma\sigma^2}]$. To see this, it suffices to verify the martingale conditions in Theorem \ref{thm:q optimal}. Applying It\^o's lemma, we obtain
\[\begin{aligned}
	& \dd J^*(t,X^{\bm\pi}(t)) -q^*(t,X^{\bm\pi}(t),a^{\bm\pi}(t))\dd t + \frac{1-\gamma}{2}\dd \langle \log X^{\bm\pi} \rangle(t) \\
	= & \bigg\{ -\left[ r + \frac{(\mu-r)^2}{2\gamma\sigma^2} + \frac{\lambda}{2}\log\frac{2\pi\lambda}{\gamma\sigma^2} \right] + r + (\mu-r) a_t^{\bm\pi} - \frac{1}{2}(a^{\bm\pi}(t))^2\sigma^2 + \frac{1-\gamma}{2}(a^{\bm\pi}(t))^2\sigma^2 \\
	& + \frac{\gamma\sigma^2}{2}(a^{\bm\pi}(t) - \frac{\mu-r}{\gamma\sigma^2})^2 + \frac{\lambda}{2}\log\frac{2\pi\lambda}{\gamma\sigma^2} \bigg\}\dd t + a^{\bm\pi}(t)\sigma\dd W(t) = a^{\bm\pi}(t)\sigma\dd W(t) .
\end{aligned} \]

\subsection{Ergodic linear-quadratic control}
\label{sec:lq true solution}
This solution has also been presented in \citet[Appendix B2]{jia2022policypg}. It is repeated here for the completeness.

Let the true model be given by \eqref{eq:lq dynamics} and one aims to maximize the long-term average reward \eqref{eq:lq non risk sensitive}.
Consider the associated HJB equation:
\[\begin{aligned}
	0 = & \sup_{a}[ \mathcal{L}^a \varphi(x) + r(x,a) -\beta] \\
	=&  \sup_{a}\left[(Ax + Ba)\varphi'(x) + \frac{1}{2}(Cx + Da)^2\varphi''(x) - (\frac{M}{2}x^2 + Rxa + \frac{N}{2}a^2 + Px + Qa)  -\beta \right] . 	
\end{aligned} \]
Conjecturing  $\varphi(x) = \frac{1}{2}k_2 x^2 + k_1 x$ and plugging it into the HJB equation, we get the first-order condition $a^* = \frac{[k_2(B+CD)-R]x + k_1B - Q}{N -k_2 D^2}$, assuming $N - k_2D^2 > 0$. The HJB equation now becomes
\[ 0 = \frac{1}{2}[k_2(2A+C^2) - M]x^2 + (k_1A - P)x - \beta + \frac{1}{2} \frac{\left\{ [k_2(B+CD)-R]x + k_1B - Q \right\}^2}{N - k_2D^2} .\]
This leads to three algebraic equations by matching the coefficients of $x^2$, $x$ and the constant term:
\begin{equation*}
	\label{eq:lq algebraic equations}
	\left\{ \begin{aligned}
		& k_2(2A+C^2) - M + \frac{[k_2(B+CD)-R]^2}{N-k_2 D^2}  = 0,\\
		& k_1A - P + \frac{[k_2(B+CD)-R](k_1B - Q)}{N-k_2D^2} = 0,\\
		& \beta = \frac{(k_1 B-Q)^2}{2(N - k_2D^2)}.
	\end{aligned} \right.  
\end{equation*}
Note $k_2$ is the negative solution to a quadratic equation, and after solving $k_2$, $k_1,\beta$ can be directly computed. 

\section{Proof of Statements}\label{appendix:proof}

\subsection{Proof Lemma \ref{lemma:classical martingale}}
We state a useful lemma regarding the continuous martingale in the exponential form. 
\begin{lemma}
	\label{lemma:exponential martingale}
	Suppose $\{Z_s,s\geq 0\}$ is an $(\{\f_s\}_{s\geq 0}, \p)$- adapted continuous semimartingale with $\E[e^{\frac{\epsilon^2}{2}\langle Z \rangle (s)}] < \infty$ for every $s\in [0,T]$. 
	\begin{enumerate}
		\item[(i)] If $\{Z_s + \frac{\epsilon}{2}\langle Z \rangle(s),s\geq 0\}$ is an $(\{\f_s\}_{s\geq 0}, \p)$-local martingale, then $\{e^{\epsilon Z_s}, \geq 0\}$ is also an $(\{\f_s\}_{s\geq 0}, \p)$ martingale.
		\item[(ii)] If $\epsilon > 0$, $\{Z_s + \frac{\epsilon}{2}\langle Z \rangle(s),s\geq 0\}$ is an $(\{\f_s\}_{s\geq 0}, \p)$-local submartingale (supermartingale), then $\{e^{\epsilon Z_s},s\geq 0\}$ is also an $(\{\f_s\}_{s\geq 0}, \p)$ submartingale (supermartingale).
		\item[(iii)] If $\epsilon < 0$, $\{Z_s + \frac{\epsilon}{2}\langle Z \rangle(s),s\geq 0\}$ is an $(\{\f_s\}_{s\geq 0}, \p)$-local submartingale (supermartingale), then $\{e^{\epsilon Z_s},s\geq 0\}$ is also an $(\{\f_s\}_{s\geq 0}, \p)$ supermartingale (submartingale).
	\end{enumerate}
\end{lemma}

\begin{proof}
	See \citet[Chapter VIII, Corollary 1.16, page 309]{revuz2013continuous} for its proof for part $(i)$. 
	
	We only show the ``submartingale'' part when $\epsilon>0$. The rest can be shown similarly. By Doob-Meyer decomposition, we can write $Z_s + \frac{\epsilon}{2}\langle Z \rangle(s) = M_s + A_s$, where $\{M_s,s\geq 0\}$ is a local martingale and $\{A_s,s\geq 0\}$ is a predictable, increasing process starting from zero. Moreover, $\langle M \rangle(s) = \langle Z \rangle(s)$ for all $s\in [0,T]$ because both $\langle Z \rangle$ and $A$ are increasing process and have QV zero. Then $\E[e^{\frac{\epsilon^2}{2}\langle M \rangle (s)}] = \E[e^{\frac{\epsilon^2}{2}\langle Z \rangle (s)}] < \infty$. By part $(i)$ of this statement, we conclude that $\{ e^{\epsilon M_s - \frac{\epsilon^2}{2} \langle M \rangle (s)},s\geq 0 \}$ is a martingale, that is, $\{ e^{\epsilon Z_s - \epsilon A_s},s\geq 0 \}$ is a martingale. Therefore, for any $0\leq s < s' \leq T$, it holds that $A_{s'} \geq A_s$ almost surely, hence,
	\[ 1 = \E\left[ e^{\epsilon Z_{s'} - \epsilon Z_s - \epsilon (A_{s'} - A_s)}  \Big|\f_s \right] \leq \E\left[ e^{\epsilon Z_{s'} - \epsilon Z_s}  \Big|\f_s \right] .\]
	That is, $\E\left[e^{\epsilon Z_{s'}} \Big| \f_s \right] \geq e^{\epsilon Z_{s}}$. This proves that $\{ e^{\epsilon Z_{s}},s\geq 0\}$ is also a submartingale.
\end{proof}

Now we are ready to prove Lemma \ref{lemma:classical martingale}.
\begin{proof}
	Let $Z_s^{\bm a} = \int_t^s r(u, X_{u}^{\bm a}, a_{u})\dd u+ V^*(s, X_s^{\bm a};\epsilon)$, it satisfies $\E\left[ e^{\frac{\epsilon^2}{2} \langle Z^{\bm a} \rangle(s) } \right] < \infty$. By Lemma \ref{lemma:exponential martingale} and condition $(i)$, we know
	\[
	\begin{aligned}
		& \E^{\p^W}\left[ e^{\epsilon \left[\int_t^T r(s,X_s^{\bm a^*},a_s^*) \dd s + h(X_T^{\bm a^*}) \right]}   \Big| X_t^{\bm a^*} = x \right] \\
		= & \E^{\p^W}\left[ e^{\epsilon \left[\int_t^T r(s,X_s^{\bm a^*},a_s^*) \dd s + V^*(T,X_T^{\bm a^*}) \right]}   \Big| X_t^{\bm a^*} = x \right] = e^{\epsilon V^*(t,x)}.
	\end{aligned}
	\]
	Moreover, when $\epsilon > (<) 0$, for any $\bm a$, By Lemma \ref{lemma:exponential martingale} and condition $(ii)$, we have 
	\[
	\begin{aligned}
		& \E^{\p^W}\left[ e^{\epsilon \left[\int_t^T r(s,X_s^{\bm a},a_s) \dd s + h(X_T^{\bm a}) \right]}   \Big| X_t^{\bm a} = x \right] \\
		= & \E^{\p^W}\left[ e^{\epsilon \left[\int_t^T r(s,X_s^{\bm a},a_s) \dd s + V^*(T,X_T^{\bm a}) \right]}   \Big| X_t^{\bm a} = x \right] \leq (\geq) e^{\epsilon V^*(t,x)}.
	\end{aligned}
	\]
	Hence, as long as $\epsilon\neq 0$, we have
	\[ \frac{1}{\epsilon}\log \E^{\p^W}\left[ e^{\epsilon \left[\int_t^T r(s,X_s^{\bm a},a_s) \dd s + h(X_T^{\bm a}) \right]}   \Big| X_t^{\bm a} = x \right] \leq V^*(t,x) . \]
	This verifies that $V^*$ is the optimal value function and $\bm a^*$ is the optimal control.
\end{proof}

\subsection{Proof of Theorem \ref{thm:exploratory martingale}}
The next lemma about the entropy-maximizing distribution is useful to our proof. It is the same lemma appeared in \citet[Lemma 13]{jia2022q}.
\begin{lemma}
	\label{lemma:entropy max}
	Let $\gamma>0$ and a measurable function $q:\mathcal{A}\to\mathbb{R}$ with $\int_{{\cal A}} \exp\{\frac{1}{\lambda}q(a)\} \dd a < \infty$ be given. Then  $\bm\pi^*(a) = \frac{\exp\{\frac{1}{\lambda}q(a)\}}{\int_{{\cal A}} \exp\{\frac{1}{\lambda}q(a)\} \dd a} \in \mathcal{P}(\mathcal{A})$ is the unique maximizer of the following problem
	\begin{equation}
		\label{eq:entropy max}
		\max_{\pi(\cdot)\in
			\mathcal{P}(\mathcal{A})}\int_{\mathcal{A}} [q(a) - \lambda\log\bm\pi(a)]\bm\pi(a)\dd a .
	\end{equation}	
\end{lemma}
\begin{proof}
	See the proof of \citet[Lemma 13]{jia2022q}.
\end{proof}

We now turn to the proof of Theorem \ref{thm:exploratory martingale}. To ease our notation, we use $q^*$ introduced in \eqref{eq:q rate} in Definition \ref{def:def q function}. Even though it is defined after Theorem \ref{thm:exploratory martingale} in the main text, we simply use its notation here. 
\begin{proof}
	Since for any initial condition $(t,x)$,
	\[ \int_t^s \left\{ \left[ r(u, X_{u}^{\bm \pi^*}, a_{u}^{\bm \pi^*}) - \lambda\log\bm\pi^*(a_{u}^{\bm\pi^*}|u, X_{u}^{\bm\pi^*}) \right] \dd u + \frac{\epsilon}{2} \dd \langle  J^{*^{\bm \pi^*}} \rangle(u)\right\}+ J^*(s, X_s^{\bm \pi^*};\epsilon) \]
	is an $(\{\f_s^X \}_{s\geq 0},\p)$- (local) martingale, we have 
	\begin{equation}
		\label{eq:hjb exploratory}
		\begin{aligned}
			0 = & \lim_{s\to t^+} \frac{1}{s-t}\E^{\p}\Bigg[ \int_t^s \left\{ \left[ r(u, X_{u}^{\bm \pi^*}, a_{u}^{\bm \pi^*}) - \lambda\log\bm\pi^*(a_{u}^{\bm\pi^*}|u, X_{u}^{\bm\pi^*}) \right] \dd u + \frac{\epsilon}{2} \dd \langle  J^{*^{\bm \pi^*}} \rangle(u)\right\} \\
			& + J^*(s, X_s^{\bm \pi^*};\epsilon) - J^*(t,x;\epsilon) \Big| X_t^{\bm\pi^*} = x\Bigg] \\
			= & \int_{{\cal A}} \left\{ \mathcal{L}^{a} J^*(t,x;\epsilon) + r(t,x,a) - \lambda \log\bm\pi^*(a|t,x) + \frac{1}{2}\epsilon \left| \sigma(t,x,a)^\top \frac{\partial J^*}{\partial x}(t,x;\epsilon)\right|^2 \right\} \\
			& \times \bm\pi^*(a|t,x)\dd a .
		\end{aligned}
	\end{equation}	
	
	By the proof of Lemma \ref{lemma:classical martingale}, it suffices to show that for any admissible policy $\bm\pi$, and any initial condition $(t,x)$,
	\[ \int_t^s \left\{ \left[ r(u, X_{u}^{\bm \pi}, a_{u}^{\bm \pi}) - \lambda\log\bm\pi(a_{u}^{\bm\pi}|u, X_{u}^{\bm\pi}) \right] \dd u + \frac{\epsilon}{2} \dd \langle  J^{*^{\bm \pi}} \rangle(u)\right\}+ J^*(s, X_s^{\bm \pi};\epsilon) \]
	is a supermartingale.
	
	For any $t\leq s < s' \leq T$, consider the conditional expectation:
	\[
	\begin{aligned}
		& \E^{\p}\Bigg[  \int_s^{s'} \left\{ \left[ r(u, X_{u}^{\bm \pi}, a_{u}^{\bm \pi}) - \lambda\log\bm\pi(a_{u}^{\bm\pi}|u, X_{u}^{\bm\pi}) \right] \dd u + \frac{\epsilon}{2} \dd \langle  J^{*^{\bm \pi}} \rangle(u)\right\}\\
		& + J^*(s', X_{s'}^{\bm \pi};\epsilon) - J^*(s, X_{s}^{\bm \pi};\epsilon)   \Big| X_t^{\bm\pi} = x \Bigg] \\
		= & \E^{\p^W}\Bigg[\int_s^{s'} \int_{\mathcal{A}}\bigg[ r(u,\tilde X_u^{\bm\pi}, a) - \lambda\log\bm\pi(a|u,\tilde X_u^{\bm\pi}) \\
		& + \frac{\epsilon}{2}|\sigma(u,\tilde X_u^{\bm\pi},a)^\top \frac{\partial J^*}{\partial x}(u,\tilde X_u^{\bm\pi};\epsilon)|^2 + \mathcal{L}^a J^*(u,\tilde X_u^{\bm\pi};\epsilon)     \bigg] \bm\pi(a|u, \tilde X_u^{\bm\pi})\dd a\dd u \Big| \tilde X_t^{\bm\pi} = x \Bigg] \\
		\leq & \E^{\p^W}\Bigg[\int_s^{s'} \int_{\mathcal{A}}\bigg[ r(u,\tilde X_u^{\bm\pi}, a) - \lambda\log\bm\pi^*(a|u,\tilde X_u^{\bm\pi}) \\
		& + \frac{\epsilon}{2}|\sigma(u,\tilde X_u^{\bm\pi},a)^\top \frac{\partial J^*}{\partial x}(u,\tilde X_u^{\bm\pi};\epsilon)|^2 + \mathcal{L}^a J^*(u,\tilde X_u^{\bm\pi};\epsilon)     \bigg] \bm\pi^*(a|u, \tilde X_u^{\bm\pi})\dd a\dd u \Big| \tilde X_t^{\bm\pi} = x \Bigg] = 0,
	\end{aligned}
	\]
	where the last inequality is due to Lemma \ref{lemma:entropy max} and the definition \eqref{eq:pi star in v star}, i.e., $\bm\pi^*(a|t,x) = \frac{\exp\{\frac{1}{\lambda}q^*(t,x,a;\epsilon)\}}{\int_{{\cal A}} \exp\{\frac{1}{\lambda}q^*(t,x,a;\epsilon)\} \dd a} \in \mathcal{P}(\mathcal{A}) $, and the last equality is by \eqref{eq:hjb exploratory}.
\end{proof}
\subsection{Proof of Proposition \ref{prop:qstar1}}
\begin{proof}
	As a side product, in the proof of Theorem \ref{thm:exploratory martingale}, we have recovered the associated exploratory HJB equation \eqref{eq:hjb exploratory}. Together with the notation $q^*$ introduced in \eqref{eq:q rate} in Definition \ref{def:def q function}, and the form of the optimal policy \eqref{eq:pi star in v star}, we can rewrite \eqref{eq:hjb exploratory} as 
	\[\begin{aligned}
		0 = & \int_{{\cal A}} \left[ q^*(t,x,a;\epsilon) - \lambda\log\frac{\exp\{\frac{1}{\lambda}q^*(t,x,a;\epsilon)\}}{\int_{{\cal A}} \exp\{\frac{1}{\lambda}q^*(t,x,a;\epsilon)\} \dd a} \right]  \frac{\exp\{\frac{1}{\lambda}q^*(t,x,a;\epsilon)\}}{\int_{{\cal A}} \exp\{\frac{1}{\lambda}q^*(t,x,a;\epsilon)\} \dd a} \dd a \\
		= & \lambda\log \int_{{\cal A}} \exp\{\frac{1}{\lambda}q^*(t,x,a;\epsilon)\} \dd a .
	\end{aligned}  \]
	That is the desired relation \eqref{eq:optimal q hjb}. Furthermore, plugging \eqref{eq:optimal q hjb} in  \eqref{eq:pi star in v star} gives \eqref{eq:optimal pi and q}.
\end{proof}

\subsection{Proof of Theorem \ref{thm:q optimal}}
\begin{proof}
	\begin{enumerate}
		\item[(i)] Let $\widehat{J^*} = J^*$ and $\widehat{q^*} = q^*$ be the optimal value function and optimal q-function respectively. For any ${\bm\pi}\in {\bm\Pi}$, applying It\^o's lemma to the process ${J^*}(s,X_s^{\bm\pi})$, we obtain for $0\leq t <s\leq T$:
		\[ \begin{aligned}
			& {J^*}(s,{X}_{s}^{\bm\pi};\epsilon) - J^*(t,x;\epsilon) + \int_t^{s} \left\{ [r(u,{X}_{u}^{\bm\pi},a^{\bm\pi}_{u}) - {q^*}(u,{X}_{u}^{\bm\pi},a^{\bm\pi}_{u};\epsilon) ]\dd u + \frac{\epsilon}{2}\dd \langle {J^*}^{\bm\pi} \rangle(u) \right\} \\
			= & \int_t^{s}\Bigg[ \mathcal{L}^{a_u^{\bm\pi}} J^*(u,{X}_{u}^{\bm\pi},a^{\bm\pi}_{u};\epsilon) + r(u,{X}_{u}^{\bm\pi},a^{\bm\pi}_{u}) + \frac{\epsilon}{2}\left| \sigma(u,{X}_{u}^{\bm\pi},a^{\bm\pi}_{u})^\top \frac{\partial }{\partial x}{J^*}(u,{X}_{u}^{\bm\pi};\epsilon) \right|^2 \\
			&  - {q^*}(u,{X}_{u}^{\bm\pi},a^{\bm\pi}_{u};\epsilon) \Bigg]\dd u  + \int_t^{s}\frac{\partial }{\partial x}{J^*}(u,{X}_{u}^{\bm\pi};\epsilon) \circ \dd W_{u} \\
			=&  \int_t^{s} \frac{\partial }{\partial x}{J^*}(u,{X}_{u}^{\bm\pi};\epsilon) \circ \sigma(u,{X}_{u}^{\bm\pi},a^{\bm\pi}_{u}) \dd W_{u} ,
		\end{aligned}\]
		where the $\int \cdots \dd u$ term vanishes due to the definition of $q^*$ in \eqref{eq:optimal q}. Hence \eqref{eq:martingale with q function2 optimal} is a local martingale. To show it is a martingale, it suffices to show 
		\[ \E\left[ \int_0^{T}\left| \sigma(u,{X}_{u}^{\bm\pi},a^{\bm\pi}_{u})^\top \frac{\partial }{\partial x}{J^*}(u,{X}_{u}^{\bm\pi};\epsilon) \right|^2\dd u \right] < \infty .\]
		The finiteness of the above expectation follows from the polynomial growth in $J^*$'s all derivatives, Assumption \ref{ass:dynamic}-(iv), Definition \ref{ass:admissible}-(iii) and the moment estimate in \citet[Lemma 2]{jia2022policypg}.
		
		The form of the optimal policy follows from Proposition \ref{prop:qstar1}.

		\item[(ii)] Define 
		\[\begin{aligned}
			\hat r(t,x,a): = & \frac{\partial \widehat{J^*}}{\partial t}(t,x) + b(t,x,a)\circ \frac{\partial \widehat{J^*}}{\partial x}(t,x) + \frac{1}{2}\sigma\sigma^\top(t,x,a)\circ \frac{\partial^2 \widehat{J^*}}{\partial x^2}(t,x)\\
			&  + \frac{\epsilon}{2}\left| \sigma(t,x,a)^\top \frac{\partial \widehat J^*}{\partial x}(t,x)\right|^2.
		\end{aligned}\]
		Then
		\[ \widehat{J^*}(s,{X}_s^{\bm\pi}) + \int_t^s \left\{ -\hat r(u,{X}_{u}^{\bm\pi},a^{\bm\pi}_{u}) \dd u + \frac{\epsilon}{2}\dd \langle \widehat{J^*}^{\bm\pi} \rangle(u) \right\}\]
		is an $(\{\f_s\}_{s\geq 0}, \p)$-local martingale, which follows from applying It\^o's lemma to the above process. As a result,
		$\int_t^s [r(u,{X}_{u}^{\bm\pi},a^{\bm\pi}_{u}) - \widehat{q^*}(u,{X}_{u}^{\bm\pi},a^{\bm\pi}_{u}) + \hat r(u,{X}_{u}^{\bm\pi},a^{\bm\pi}_{u})  ]\dd u$ is an $(\{\f_s\}_{s\geq 0}, \p)$-local martingale. 
		
		By the same argument as in the proof of \citet[Theorem 6]{jia2022q}, we conclude
		\[\begin{aligned}
			\widehat{q^*}(t,x,a) = & r(t,x,a) + \hat r(t,x,a)  \\
			= & r(t,x,a) + \mathcal{L}^a \widehat{J^*}(t,x) + \frac{\epsilon}{2}\left| \sigma(t,x,a)^\top \frac{\partial \widehat J^*}{\partial x}(t,x)\right|^2,
		\end{aligned}\]
		for every $(t,x,a)$.

		The second constraint in \eqref{eq:q hjb2 optimal} implies that $\widehat{\bm\pi^*}(a|t,x): = \exp\{  \frac{1}{\lambda}\widehat{q^*}(t,x,a) \}$ is a probability density function, and $\widehat{q^*}(t,x,a) = \lambda \log\widehat{\bm\pi^*}(a|t,x)$. Therefore,
		\[\begin{aligned}
			& \int_t^s \left\{ \left[ r(u, X_{u}^{\widehat{\bm\pi^*}}, a_{u}^{\widehat{\bm\pi^*}}) - \lambda\log\widehat{\bm\pi^*}(a_{u}^{\widehat{\bm\pi^*}}|u, X_{u}^{\widehat{\bm\pi^*}}) \right] \dd u + \frac{\epsilon}{2} \dd \langle  \widehat{J}^{*^{\widehat{\bm\pi^*}}} \rangle(u)\right\}+ \widehat{J}^*(s, X_s^{\widehat{\bm\pi^*}}) \\
			= & \int_t^s \left\{ \left[ r(u, X_{u}^{\widehat{\bm\pi^*}}, a_{u}^{\widehat{\bm\pi^*}}) - \widehat{q^*}(u, X_{u}^{\widehat{\bm\pi^*}},a_{u}^{\widehat{\bm\pi^*}}) \right] \dd u + \frac{\epsilon}{2} \dd \langle  \widehat{J}^{*^{\widehat{\bm\pi^*}}} \rangle(u)\right\}+ \widehat{J}^*(s, X_s^{\widehat{\bm\pi^*}})
		\end{aligned}  \]
		is an $(\{\f_s \}_{s\geq 0},\p)$- (local) martingale, which follows from applying It\^o's lemma to the above process. Hence it is also an $(\{\f_s^X \}_{s\geq 0},\p)$- (local) martingale because $\f_s^X$ is a sub-sigma field of $\f_s$.
		
		By Theorem \ref{thm:exploratory martingale}, we conclude that $\widehat{J^*}$ is the optimal value function and $\widehat{\bm\pi^*}$ is the optimal policy.
	\end{enumerate}
\end{proof}

\subsection{Proof of Theorem \ref{thm:q optimal ergodic}}
\begin{proof}
	By the same argument as in the proof of Theorem \ref{thm:q optimal}, we conclude that
	\[ \widehat{q^*}(x,a) =  r(x,a) + \mathcal{L}^a \widehat{J^*}(x) - \widehat{\beta^*} + \frac{\epsilon}{2}\left| \sigma(x,a)^\top \frac{\partial \widehat J^*}{\partial x}(x)\right|^2. \]
	The second constraint in \eqref{eq:q hjb2 optimal ergodic} implies that $\widehat{\bm\pi^*}(a|x) = \exp\{  \frac{1}{\lambda}\widehat{q^*}(x,a) \}$ is a probability density function, and $\widehat{q^*}(x,a) = \lambda \log\widehat{\bm\pi^*}(a|x)$. Applying It\^o's lemma, we obtain 
	\[\begin{aligned}
		& \exp\left\{ \epsilon\left[ \widehat{J^*}(X_T^{\widehat{\bm\pi^*}}) + \int_0^T \left(  r(X_t^{\widehat{\bm\pi^*}}, a_t^{\widehat{\bm\pi^*}}) -\lambda \log\widehat{\bm\pi^*}(a_t^{\widehat{\bm\pi^*}} | X_t^{\widehat{\bm\pi^*}})\right)\dd t \right]  \right\} \\
		= & e^{\epsilon \widehat{J^*}(x)} \exp\left\{  \int_0^T \epsilon\left( \mathcal{L}^{a_t^{\widehat{\bm\pi^*}}} \widehat{J^*}(X_t^{\widehat{\bm\pi^*}}) +  r(X_t^{\widehat{\bm\pi^*}}, a_t^{\widehat{\bm\pi^*}}) -\lambda \log\widehat{\bm\pi^*}(a_t^{\widehat{\bm\pi^*}} | X_t^{\widehat{\bm\pi^*}})\right)\dd t +  Z_t^{\widehat{\bm\pi^*}} \circ \dd W_t  \right\} \\
		= & e^{\epsilon \widehat{J^*}(x) + \epsilon \widehat{\beta^*} T } \exp\left\{  \int_0^T - \frac{1}{2}|Z_t^{\widehat{\bm\pi^*}}|^2  \dd t + Z_t^{\widehat{\bm\pi^*}} \circ \dd W_t  \right\},
	\end{aligned}   \]
	where $Z_t^{\widehat{\bm\pi^*}} =  \epsilon \sigma(X_t^{\widehat{\bm\pi^*}},a_t^{\widehat{\bm\pi^*}}) \frac{\partial }{\partial x}\widehat{J^*}(X_t^{\widehat{\bm\pi^*}})$.
	
	Based on the first condition in \eqref{eq:q hjb2 optimal ergodic}, and Girsanov theorem, we can define a change-of-measure $\frac{\dd \hat{\p}}{\dd \p}|_{\f_T} = \exp\left\{  \int_0^T - \frac{1}{2}|Z_t^{\widehat{\bm\pi^*}}|^2  \dd t + Z_t^{\widehat{\bm\pi^*}} \circ \dd W_t  \right\}$ and $\hat\p$ is a new probability measure that is equivalent to the original probability $\p$. Therefore,
	\[\begin{aligned}
		& \frac{1}{\epsilon T} \log \E\left[ \exp\left\{ \epsilon\left[  \int_0^T \left(  r(X_t^{\widehat{\bm\pi^*}}, a_t^{\widehat{\bm\pi^*}}) -\lambda \log\widehat{\bm\pi^*}(a_t^{\widehat{\bm\pi^*}} | X_t^{\widehat{\bm\pi^*}})\right)\dd t \right]  \right\} \right] \\
		= & \frac{1}{T} \widehat{J^*}(x) + \widehat{\beta^*} + \frac{1}{\epsilon T} \log \E\left[ \exp\left\{ -\epsilon \widehat{J^*}(X_T^{\widehat{\bm\pi^*}}) +  \int_0^T - \frac{1}{2}|Z_t^{\widehat{\bm\pi^*}}|^2  \dd t + Z_t^{\widehat{\bm\pi^*}} \circ \dd W_t   \right\}  \right] \\
		= & \frac{1}{T} \widehat{J^*}(x) + \widehat{\beta^*} + \frac{1}{\epsilon T} \log\E^{\hat\p}\left[ e^{-\epsilon \widehat{J^*}(X_T^{\widehat{\bm\pi^*}}) } \right]. 
	\end{aligned}  \]
	
	Under the extra regularity conditions in the statement, we obtain 
	\[ \lim_{T\to \infty}\frac{1}{\epsilon T} \log \E\left[ \exp\left\{ \epsilon\left[  \int_0^T \left(  r(X_t^{\widehat{\bm\pi^*}}, a_t^{\widehat{\bm\pi^*}}) -\lambda \log\widehat{\bm\pi^*}(a_t^{\widehat{\bm\pi^*}} | X_t^{\widehat{\bm\pi^*}})\right)\dd t \right]  \right\} \right] = \widehat{\beta^*} . \] 
	
	For any $\bm\pi\in \bm\Pi$, repeat the above calculation, we obtain 
	\[ \begin{aligned}
		& \frac{1}{\epsilon T} \log \E\left[ \exp\left\{ \epsilon\left[  \int_0^T \left(  r(X_t^{\bm\pi}, a_t^{\bm\pi}) -\lambda \log\bm\pi(a_t^{\bm\pi} | X_t^{\bm\pi})\right)\dd t \right]  \right\} \right] \\
		=& \frac{1}{T} \widehat{J^*}(x) + \widehat{\beta^*} + \frac{1}{\epsilon T} \log \E^{\p^{\bm\pi}}\left[ \exp\left\{ -\epsilon \widehat{J^*}(X_T^{\bm\pi}) + \epsilon \int_0^T \left[ \widehat{q^*}(X_t^{\bm\pi}, a_t^{\bm\pi}) - \lambda\log\bm\pi(a_t^{\bm\pi} | X_t^{\bm\pi}) \right] \dd t \right\}  \right],
	\end{aligned}  \]
	where $\frac{\dd\p^{\bm\pi}}{\dd \p}|_{\f_T} = \exp\left\{  \int_0^T - \frac{1}{2}|Z_t^{\bm\pi}|^2  \dd t + Z_t^{\bm\pi} \circ \dd W_t  \right\}$, and $Z_t^{\bm\pi} =  \epsilon \sigma(X_t^{\bm\pi},a_t^{\bm\pi}) \frac{\partial }{\partial x}\widehat{J^*}(X_t^{\bm\pi})$.
	
	Recall that 
	\[ \int_{\mathcal A} \left[ \widehat{q^*}(x,a) - \lambda\log\bm\pi(a|x) \right]\dd \bm\pi(a|x)\dd a \leq \int_{\mathcal A} \left[ \widehat{q^*}(x,a) - \lambda\log\widehat{\bm\pi^*}(a|x) \right]\widehat{\bm\pi^*}(a|x) \dd  a = 0 \]
	by Lemma \ref{lemma:entropy max} and the fact that $\widehat{\bm\pi^*}(a|x) = \exp\{  \frac{1}{\lambda}\widehat{q^*}(x,a) \}$, we know if $\epsilon > (<) 0$,
	\[ \E^{\p^{\bm\pi}}\left[ \exp\left\{  \frac{\epsilon( 1+\delta)}{\delta} \int_0^T \left[ \widehat{q^*}(X_t^{\bm\pi}, a_t^{\bm\pi}) - \lambda\log\bm\pi(a_t^{\bm\pi} | X_t^{\bm\pi}) \right] \dd t \right\}  \right] \leq (\geq) 0 ,\]
	by Lemma \ref{lemma:exponential martingale}.
	
	By H\"older's inequality, we obtain 
	\[
	\begin{aligned}
		& \frac{1}{\epsilon T}\log\E^{\p^{\bm\pi}}\left[ \exp\left\{ -\epsilon \widehat{J^*}(X_T^{\bm\pi}) + \epsilon \int_0^T \left[ \widehat{q^*}(X_t^{\bm\pi}, a_t^{\bm\pi}) - \lambda\log\bm\pi(a_t^{\bm\pi} | X_t^{\bm\pi}) \right] \dd t \right\}  \right] \\
		\leq & \frac{1}{(1+\delta)\epsilon T }\log \E^{\p^{\bm\pi}}\left[ \exp\left\{ -\epsilon(1+\delta) \widehat{J^*}(X_T^{\bm\pi}) \right\}  \right] \\
		& + \frac{\delta}{(1+\delta) \epsilon T} \log \E^{\p^{\bm\pi}}\left[ \exp\left\{  \frac{\epsilon( 1+\delta)}{\delta} \int_0^T \left[ \widehat{q^*}(X_t^{\bm\pi}, a_t^{\bm\pi}) - \lambda\log\bm\pi(a_t^{\bm\pi} | X_t^{\bm\pi}) \right] \dd t \right\}  \right] \\
		\leq & \frac{1}{(1+\delta)\epsilon T }\log \E^{\p^{\bm\pi}}\left[ \exp\left\{ -\epsilon(1+\delta) \widehat{J^*}(X_T^{\bm\pi}) \right\}  \right] \to 0,
	\end{aligned}
	\]
	as $T\to\infty$. Hence, we have shown that under any $\bm\pi\in \bm\Pi$, we have 
	\[ \limsup_{T\to \infty} \frac{1}{\epsilon T} \log \E\left[ \exp\left\{ \epsilon\left[  \int_0^T \left(  r(X_t^{\bm\pi}, a_t^{\bm\pi}) -\lambda \log\bm\pi(a_t^{\bm\pi} | X_t^{\bm\pi})\right)\dd t \right]  \right\} \right] \leq \widehat{\beta^*} .  \]
	This proves the optimality of the policy $\widehat{\bm\pi^*}$. Consequently, $\widehat{J^*}$, $\widehat{q^*}$, $\widehat{\beta^*}$ are respectively the optimal value function, the optimal q-function, and the optimal value.
\end{proof}

\subsection{Proof of Theorem \ref{thm:convergence of merton}}
We first state a result about a recursive sequence.
	
	\begin{lemma}
		\label{lemma:recursive mse}
		Suppose there are sequences of non-negative numbers $\{e_i\}_{i\geq 0}, \{ a_i \}_{i\geq 0},\{\eta_i\}_{i\geq 0}$ such that $a_i \in (0,1)$, $e_{i+1} \leq (1 -  a_i) e_i + a_i^2 \eta_i$ for all $i\geq 0$. If for all $i\geq 0$, with $a_i \leq a_{i+1}(1 + A a_{i+1})$ for some constant $0 < A < 1$, satisfying $\frac{1}{1-A} \geq \frac{(1-a_0)e_0}{a_0\eta_0} + a_0$, and $\eta_i$ is non-decreasing in $i$, then $e_{i+1} \leq \frac{1}{1-A} a_i \eta_i$ for all $i\geq 0$. In particular, for any $p\in(\frac{1}{2}, 1]$, a sequence in the form of $a_i = \frac{\alpha}{(i + \beta)^p}$ with suitable choice of $\alpha,\beta$ satisfies $a_i \leq a_{i+1}(1 + A a_{i+1})$ for all $i\geq 0$.
	\end{lemma}
	\begin{proof}
		We prove by induction. For $i=0$, we have 
		\[ e_1 \leq (1-a_0)e_0 + a_0^2\eta_0 \leq \frac{1}{1-A} a_0 \eta_0 . \]
		
		Suppose the conclusion holds for all $i=0,\cdots,I-1$. Consider the case for $I$. By the recursive relation and the induction, we have
		\[\begin{aligned}
			e_{I+1} \leq & (1 -  a_I) e_I + a_I^2 \eta_I \leq \frac{1}{1-A} (1 -  a_I) a_{I-1} \eta_{I-1} + a_I^2 \eta_I \\
			\leq & \frac{1}{1-A} (1 -  a_I) a_{I}(1 + A a_{I}) \eta_I + a_I^2 \eta_I \\
			= &  \frac{1}{1-A}a_I\eta_I(1 + Aa_I - a_I - Aa_I^2) + a_I^2\eta_I \\
			= & \frac{1}{1-A}a_I\eta_I - \frac{A}{1-A} a_I^3\eta_I < \frac{1}{1-A}a_I\eta_I .
		\end{aligned}  \]
		
		Finally, consider a sequence in the form of $a_i = \frac{\alpha}{(i + \beta)^p}$, then $a_i \leq a_{i+1}(1 + A a_{i+1})$ is equivalent to 
		\[ (i+\beta+1)^p - (i + \beta)^p \leq A \frac{\alpha (i+\beta)^p}{(i+\beta+1)^p} .\]
		Note that the left-hand side is decreasing in $i$ when $p\leq 1$ and the right-hand side is increasing in $i$, it suffices to have 
		\[ (\beta+1)^p - \beta^p \leq A \frac{\alpha \beta^p}{(\beta+1)^p} . \]
		The above relation holds as long as $\alpha$ is sufficiently large, and $\beta$ can be arbitrary.
	\end{proof}
	
To get the big picture of how the proof is conducted, I first illustrate the ideas by assuming ideal continuous sampling and computation of integral are feasible in the following Stages 1 and 2. Then I give the proof by considering the discretization error in Stage 3.
	\subsubsection{Stage 1: Properties of the increment}
	We first calculate the conditional mean and variance of \eqref{eq:update theta}, \eqref{eq:update psi 1}, and \eqref{eq:update psi 2}. To ease notations, we suppress the subscript $i$ in this stage.  
	
	Among \eqref{eq:update theta}, \eqref{eq:update psi 1}, and \eqref{eq:update psi 2}, the common term is $\dd V\left(t , X(t);\theta  \right) - q\left(a(t) ;\bm\psi\right)\dd t + \frac{1-\gamma}{2}\dd \langle V\left( \cdot,X(\cdot);\theta \right)\rangle(t)$, which is the \textit{temporal difference (TD)} error in the continuous-time setting. By It\^o's lemma, we have
	\begin{equation}
		\label{eq:td}
		\begin{aligned}
			& \dd V\left(t , X(t);\theta  \right) - q\left( a(t) ;\bm\psi\right)\dd t + \frac{1-\gamma}{2}\dd \langle V\left( \cdot,X(\cdot);\theta \right)\rangle(t) \\
			=&  \dd \log X(t) - \theta\dd t + \left[ \frac{\left( a(t) - \psi_{1} \right)^2}{2\psi_{2}} + \frac{\lambda}{2}\log2\pi\lambda + \frac{\lambda}{2}\log\psi_{2} \right] \dd t + \frac{1-\gamma}{2} \dd \langle \log X(\cdot) \rangle(t) \\
			= & \left[ r + (\mu - r)a(t) - \frac{\gamma}{2}a(t)^2\sigma^2 \right]\dd t + a(t)\sigma \dd W(t) - \theta\dd t \\
			& + \left[ \frac{\left( a(t) - \psi_{1} \right)^2}{2\psi_{2}} + \frac{\lambda}{2}\log2\pi\lambda + \frac{\lambda}{2}\log\psi_{2} \right] \dd t . 
		\end{aligned}
	\end{equation}
	
	Recall here $a(t)\sim \mathcal{N}(\psi_1,\lambda\psi_2)$, hence, for a given $\theta,\bm\psi$,
	\[ \begin{aligned}
		& \E\left[  \int_0^T \frac{\partial V}{\partial \theta}\left(t, X(t);\theta \right) \left[ \dd V\left(t , X(t);\theta  \right) - q\left( a(t) ;\bm\psi\right)\dd t + \frac{1-\gamma}{2}\dd \langle V\left( \cdot,X(\cdot);\theta \right)\rangle(t)  \right]  \right] \\
		= & \E\left[ \int_0^T (T-t)   \left( r + (\mu - r)a(t) - \frac{\gamma}{2}a(t)^2\sigma^2  -\theta +  \frac{\left( a(t) - \psi_{1} \right)^2}{2\psi_{2}} + \frac{\lambda}{2}\log2\pi\lambda + \frac{\lambda}{2}\log\psi_{2} \right)\dd t \right] \\
		= & \int_0^T (T-t) \left(  r + (\mu - r) \psi_{1} - \frac{\gamma}{2}\left( \psi_{1}^2 + \lambda \psi_{2} \right) \sigma^2  -\theta +  \frac{\lambda}{2} + \frac{\lambda}{2}\log2\pi\lambda + \frac{\lambda}{2}\log\psi_{2} \right) \dd t \\
		= & \frac{T^2}{2} h_{\theta}( \psi_{1},\psi_{2},\theta),
	\end{aligned} \]
	where 
	\begin{equation}
		\label{eq:mean update theta}
		h_{\theta}( \psi_{1},\psi_{2},\theta) \\
		=  r-\theta+\frac{\lambda}{2}\log2\pi\lambda+\frac{\lambda}{2}\log\psi_2 +  (\mu - r) \psi_{1}  - \frac{\gamma}{2}\left( \psi_{1}^2 + \lambda \psi_{2} \right) \sigma^2  +  \frac{\lambda}{2} .   
	\end{equation}
	\[ \begin{aligned}
		& \E\left[  \int_0^T \frac{\partial q}{\partial \psi_1}\left( a(t);\bm\psi \right) \left[ \dd V\left(t , X(t);\theta  \right) - q\left( a(t) ;\bm\psi\right)\dd t + \frac{1-\gamma}{2}\dd \langle V\left( \cdot,X(\cdot);\theta \right)\rangle(t)  \right]   \right] \\
		= & \E\left[ \int_0^T \frac{a(t) - \psi_1}{2\psi_2}  \left( r + (\mu - r)a(t) - \frac{\gamma}{2}a(t)^2\sigma^2  -\theta +  \frac{\left( a(t) - \psi_{1} \right)^2}{2\psi_{2}} + \frac{\lambda}{2}\log2\pi\lambda + \frac{\lambda}{2}\log\psi_{2} \right)\dd t \right] \\
		= & T h_{\psi,1}( \psi_{1},\psi_{2},\theta),
	\end{aligned} \]
	where 
	\begin{equation}
		\label{eq:mean update psi 1}
		h_{\psi, 1}( \psi_{1},\psi_{2},\theta) = \frac{\lambda \gamma \sigma^2}{2} \left( \frac{\mu-r}{\gamma\sigma^2} - \psi_1 \right) = \frac{\lambda \gamma \sigma^2}{2} \left( \psi_1^* - \psi_1 \right).
	\end{equation}
	\[ \begin{aligned}
		& -\E\left[  \int_0^T \frac{\partial q}{\partial \psi_2^{-1}}\left(t, R(t), a(t);\bm\psi \right) \left[ \dd V\left(t , X(t);\theta  \right) - q\left( a(t) ;\bm\psi\right)\dd t + \frac{1-\gamma}{2}\dd \langle V\left( \cdot,X(\cdot);\theta \right)\rangle(t)  \right]  \right] \\
		= & -\E\Bigg[ \int_0^T \left( -\frac{\left( a(t) - \psi_1 \right)^2}{2} + \frac{\lambda}{2}\psi_2 \right) \bigg( r + (\mu - r)a(t) - \frac{\gamma}{2}a(t)^2\sigma^2  -\theta \\
		& +  \frac{\left( a(t) - \psi_{1} \right)^2}{2\psi_{2}} + \frac{\lambda}{2}\log2\pi\lambda + \frac{\lambda}{2}\log\psi_{2} \bigg)\dd t \Bigg] \\
		= & T h_{\psi,2}( \psi_{1},\psi_{2},\theta ),
	\end{aligned} \]
	where 
	\begin{equation}
		\label{eq:mean update psi 2}
		h_{\psi, 2}( \psi_{1},\psi_{2},\theta) \\
		= \frac{\lambda^2}{2}\left( \psi_2 - \gamma \sigma^2 \psi_2^2\right) = \frac{\lambda^2 \gamma \sigma^2}{2}\psi_2 (\frac{1}{\gamma\sigma^2} - \psi_2) =  \frac{\lambda^2 \gamma \sigma^2}{2}\psi_2 (\psi_2^* - \psi_2) .
	\end{equation}

	Next, we estimate the variance of the increment. For a given $\theta,\bm\psi$, by Cauchy-Schwarz inequality and Burkholder-Davis-Gundy inequality, we have
	\[ \begin{aligned}
		& \operatorname{Var}\left[  \int_0^T \frac{\partial V}{\partial \theta}\left(t, X(t);\theta \right) \left[ \dd V\left(t , X(t);\theta  \right) - q\left( a(t) ;\bm\psi\right)\dd t + \frac{1-\gamma}{2}\dd \langle V\left( \cdot,X(\cdot);\theta \right)\rangle(t)  \right]  \right] \\
		\leq & \E\left[ \left(  \int_0^T \frac{\partial V}{\partial \theta}\left(t, X(t);\theta \right) \left[ \dd V\left(t , X(t);\theta  \right) - q\left( a(t) ;\bm\psi\right)\dd t + \frac{1-\gamma}{2}\dd \langle V\left( \cdot,X(\cdot);\theta \right)\rangle(t)  \right] \right)^2 \right] \\
		\leq & 2\E\left[ \int_0^T (T-t)^2   \left( r + (\mu - r)a(t) - \frac{\gamma}{2}a(t)^2\sigma^2  -\theta +  \frac{\left( a(t) - \psi_{1} \right)^2}{2\psi_{2}} + \frac{\lambda}{2}\log2\pi\lambda + \frac{\lambda}{2}\log\psi_{2} \right)^2\dd t \right] \\
		& + 2\E\left[ \int_0^T (T-t)^2 a(t)^2\sigma^2 \dd t \right] \\
		\leq & C(\mu,r,\sigma,T,\gamma) \left[1+ \psi_1^4 + \lambda^2 \psi_2^2 + \theta^2 + \lambda^2(\log\lambda)^2 + \lambda^2 + \lambda^2(\log\psi_2)^2 \right].
	\end{aligned} \]
	Similarly,
	\[ \begin{aligned}
		& \operatorname{Var}\left[  \int_0^T \frac{\partial q}{\partial \psi_1}\left( a(t);\bm\psi \right) \left[ \dd V\left(t , X(t);\theta  \right) - q\left( a(t) ;\bm\psi\right)\dd t + \frac{1-\gamma}{2}\dd \langle V\left( \cdot,X(\cdot);\theta \right)\rangle(t)  \right]   \right] \\
		\leq & \E\left[ \left( \int_0^T \frac{\partial q}{\partial \psi_1}\left( a(t);\bm\psi \right) \left[ \dd V\left(t , X(t);\theta  \right) - q\left( a(t) ;\bm\psi\right)\dd t + \frac{1-\gamma}{2}\dd \langle V\left( \cdot,X(\cdot);\theta \right)\rangle(t)  \right] \right)^2  \right] \\
		\leq & 2\E\left[ \int_0^T \left(\frac{a(t) - \psi_1}{2\psi_2} \right)^2  \left( r + (\mu - r)a(t) - \frac{\gamma}{2}a(t)^2\sigma^2  -\theta +  \frac{\left( a(t) - \psi_{1} \right)^2}{2\psi_{2}} + \frac{\lambda}{2}\log2\pi\lambda + \frac{\lambda}{2}\log\psi_{2} \right)^2 \dd t \right] \\
		& + 2\E\left[ \int_0^T \left(\frac{a(t) - \psi_1}{2\psi_2} \right)^2 a(t)^2\sigma^2\dd t  \right] \\
		\leq & \lambda C(\mu,r,\sigma,T,\gamma)\left[1 + \lambda^2 + \lambda^3\psi_2 + \psi_2^{-1}(1 + \theta^2 + \psi_1^2 + \lambda^2(\log\lambda)^2 + \lambda^2 + \lambda^2 (\log\psi_2)^2)  \right] .
	\end{aligned} \]
	\[ \begin{aligned}
		& \operatorname{Var}\left[  \int_0^T \frac{\partial q}{\partial \psi_2^{-1}}\left(t, R(t), a(t);\bm\psi \right) \left[ \dd V\left(t , X(t);\theta  \right) - q\left( a(t) ;\bm\psi\right)\dd t + \frac{1-\gamma}{2}\dd \langle V\left( \cdot,X(\cdot);\theta \right)\rangle(t)  \right]  \right] \\
		\leq & \E\left[ \left( \int_0^T \frac{\partial q}{\partial \psi_2^{-1}}\left(t, R(t), a(t);\bm\psi \right) \left[ \dd V\left(t , X(t);\theta  \right) - q\left( a(t) ;\bm\psi\right)\dd t + \frac{1-\gamma}{2}\dd \langle V\left( \cdot,X(\cdot);\theta \right)\rangle(t)  \right] \right)^2 \right] \\
		\leq & 2\E\Bigg[ \int_0^T \left( -\frac{\left( a(t) - \psi_1 \right)^2}{2} + \frac{\lambda}{2}\psi_2 \right)^2 \bigg( r + (\mu - r)a(t) - \frac{\gamma}{2}a(t)^2\sigma^2  -\theta \\
		& +  \frac{\left( a(t) - \psi_{1} \right)^2}{2\psi_{2}} + \frac{\lambda}{2}\log2\pi\lambda + \frac{\lambda}{2}\log\psi_{2} \bigg)^2\dd t \Bigg] + 2\E\left[\int_0^T \left( -\frac{\left( a(t) - \psi_1 \right)^2}{2} + \frac{\lambda}{2}\psi_2  \right)^2  a(t)^2\sigma^2\dd t \right]\\
		\leq & \lambda^2 C(\mu,r,\sigma,T,\gamma)\left[ 1 + \lambda^2\psi_2^2 \left( 1 + \theta^2 + \psi_1^2 + \lambda^2(\log\lambda)^2 + \lambda^2 + \lambda^2(\log\psi_2)^2\right)  +  \lambda^3\psi_2^3  \right] .
	\end{aligned} \]
	\subsubsection{Stage 2: Recursive relations}
	From the above calculation, and ignoring the error of approximating the integral \eqref{eq:update theta}, \eqref{eq:update psi 1}, and \eqref{eq:update psi 2} using finite sums, the recursive relation of the sequence in the learning procedure can be written as 
	\begin{equation}
		\label{eq:projection recusive}
		\begin{aligned}
			\theta_{i+1} = & \Pi_{[-c_{i+1}, c_{i+1}]}\left( \theta_i + a_{\theta,i} \left( h_{\theta}(\psi_{1,i},\psi_{2,i},\theta_i) + \zeta_{i+1}  \right) \right),\\
			\psi_{1,i+1} = & \Pi_{[-c_{i+1}, c_{i+1}]}\left( \psi_{1,i} + a_{\psi,i} \left( h_{\psi,1}(\psi_{1,i},\psi_{2,i},\theta_i) + \xi_{1,i+1}  \right) \right),\\
			\psi_{2,i+1} = & \Pi_{[b_{i+1}^{-1}, c_{i+1}]}\left(  \psi_{2,i} + a_{\psi,i} \left( h_{\psi,2}(\psi_{1,i},\psi_{2,i},\theta_i) + \xi_{2,i+1} \right) \right) ,
		\end{aligned}
	\end{equation}
	where $\E[\zeta_{i+1}|\theta_i,\bm\psi_i] = 0$, $\E[\xi_{1,i+1}|\theta_i,\bm\psi_i] = 0$, $\E[\xi_{2,i+1}|\theta_i,\bm\psi_i] = 0$, and $\operatorname{Var}[\zeta_{i+1}|\theta_i,\bm\psi_i]$, $\operatorname{Var}[\xi_{1,i+1}|\theta_i,\bm\psi_i]$, and $\operatorname{Var}[\xi_{1,i+1}|\theta_i,\bm\psi_i]$ satisfy the upper bounds provided in Stage 1.
	
	We first show the convergence of $\bm\psi_i$. Since $\{b_i\}_{i\geq 0},\{c_i\}_{i\geq 0}$ are two increasing, divergent sequence (condition $(i)$ in the statement), there exists an $I$ that depends on $(\mu,r,\sigma,T,\gamma,\lambda)$, such that $\psi_1^* \in [-c_i,c_i]$ and $\psi_2^* \in [b_i^{-1}, c_i]$ for all $i\geq I$. In addition, by conditions $(iii)$, we know $a_{\psi,i}c_i^2\to 0$, so does $a_{\psi,i}c_i\to 0$. Hence, without loss of generality, we can take this $I(\mu,r,\sigma,T,\gamma,\lambda)$ large enough such that $a_{\psi,i}b_i^{-1} \leq a_{\psi,i}c_i \leq \frac{1}{\lambda^2\gamma\sigma^2} $ for all $i\geq I$. In addition, by the conditions $(i),(ii),(iii)$ in the statement, $b_i\uparrow \infty$, $\sum_{i=1}^{\infty}a_{\psi,i}b_i^{-1} = \infty$, and $\sum_{i=1}^{\infty}a_{\psi,i}^2 b_j^2 < \infty$, we know $\sum_{i=1}^{\infty}a_{\psi,i} = \infty$ and $\sum_{i=1}^{\infty}a_{\psi,i}^2 < \infty$. 
	
	For any $i\geq I$, by the property of the projection mapping as in the proof of \citet[Theorem 2]{andradottir1995stochastic}, we have 
	\[ \begin{aligned}
		& \E\left[ \left( \psi_{1,i+1} - \psi_1^*\right)^2  | \theta_i,\bm\psi_i\right] \\
		\leq &  \E\left[ \left(\psi_{1,i} + a_{\psi,i} \left( h_{\psi,1}(\psi_{1,i},\psi_{2,i},\theta_i) + \xi_{1,i+1} \right) - \psi_1^* \right)^2  | \theta_i,\bm\psi_i\right] \\
		\leq & (\psi_{1,i} - \psi_1^*)^2 + 2a_{\psi,i}(\psi_{1,i} - \psi_1^*)h_{\psi,1}(\psi_{1,i},\psi_{2,i},\theta_i)  + a_{\psi,i}^2  h_{\psi,1}(\psi_{1,i},\psi_{2,i},\theta_i)^2 + a_{\psi,i}^2 \operatorname{Var}[\xi_{1,i+1}|\theta_i,\bm\psi_i] \\
		= & (1 - a_{\psi,i}\lambda\gamma\sigma^2 ) (\psi_{1,i} - \psi_1^*)^2 + \frac{1}{4}a_{\psi,i}^2\lambda^2\gamma^2\sigma^4(\psi_{1,i} - \psi_1^*)^2 +  a_{\psi,i}^2 \operatorname{Var}[\xi_{1,i+1}|\theta_i,\bm\psi_i]\\
		\leq & (1 - \frac{1}{2}a_{\psi,i}\lambda\gamma\sigma^2 )^2 (\psi_{1,i} - \psi_1^*)^2 + a_{\psi,i}^2\lambda C \left( 1 + \lambda^2 + \lambda^3 c_i + b_i(1 + c_i^2 + \lambda^2(\log\lambda)^2 + \lambda^2 + \lambda^2 (\log c_i)^2)  \right) .
	\end{aligned}\]
	The almost sure convergence of $\{ (\psi_{1,i} - \psi_1^*)^2\}_{i\geq 0}$ follows from \citet[Theorem 1]{robbins1971convergence} and by the same argument as in \citet[Theorem 2]{andradottir1995stochastic}, we can show $\psi_{1,i} - \psi_1^* \to 0$ almost surely.

	Taking the expectation on both sides yield the recursive relation between two mean squared error:
	\begin{equation}
		\label{eq:recursive error psi 1}
		\begin{aligned}
			& \epsilon_{1,i+1} \\
			\leq & (1 - \frac{1}{2}a_{\psi,i}\lambda\gamma\sigma^2 )^2 \epsilon_{1,i} + a_{\psi,i}^2 \lambda C \left( 1 + \lambda^2 + \lambda^3 c_i + b_i(1 + c_i^2 + \lambda^2(\log\lambda)^2 + \lambda^2 + \lambda^2 (\log c_i)^2)  \right) \\
			\leq & (1 - \frac{1}{2}a_{\psi,i}\lambda\gamma\sigma^2 ) \epsilon_{1,i} + a_{\psi,i}^2 \lambda C \left( 1 + \lambda^2 + \lambda^3 c_i + b_i(1 + c_i^2 + \lambda^2(\log\lambda)^2 + \lambda^2 + \lambda^2 (\log c_i)^2)  \right),
		\end{aligned}
	\end{equation}
	where $\epsilon_{1,i} = \E\left[ (\psi_{1,i} - \psi_{1}^*)^2  \right]$.
	
	By the similar calculation, we have
	\[ \begin{aligned}
		& \E\left[ \left( \psi_{2,i+1} - \psi_2^*\right)^2  | \theta_i,\bm\psi_i\right] \\
		\leq &  \E\left[ \left(\psi_{2,i} + a_{\psi,i} \left( h_{\psi,2}(\psi_{1,i},\psi_{2,i},\theta_i) + \xi_{2,i+1} \right) - \psi_2^* \right)^2  | \theta_i,\bm\psi_i\right] \\
		\leq & (\psi_{2,i} - \psi_2^*)^2 + 2a_{\psi,i}(\psi_{2,i} - \psi_2^*)h_{\psi,2}(\psi_{1,i},\psi_{2,i},\theta_i)  + a_{\psi,i}^2  h_{\psi,2}(\psi_{1,i},\psi_{2,i},\theta_i)^2 + a_{\psi,i}^2 \operatorname{Var}[\xi_{2,i+1}|\theta_i,\bm\psi_i] \\
		= & (1 - a_{\psi,i}\lambda^2\gamma\sigma^2\psi_{2,i} ) (\psi_{2,i} - \psi_2^*)^2 + \frac{1}{4}a_{\psi,i}^2\lambda^4\gamma^2\sigma^4\psi_{2,i}^2(\psi_{2,i} - \psi_2^*)^2 +  a_{\psi,i}^2 \operatorname{Var}[\xi_{2,i+1}|\theta_i,\bm\psi_i]\\
		\leq & (1 - \frac{1}{2}a_{\psi,i}\lambda^2\gamma\sigma^2 \psi_{2,i})^2 (\psi_{2,i} - \psi_2^*)^2 + a_{\psi,i}^2\lambda^2 C \left( 1 + \lambda^2c_i^2 \left( 1 + c_i^2 + \lambda^2(\log\lambda)^2 + \lambda^2 + \lambda^2(\log b_i)^2\right)  +  \lambda^3c_i^3 \right) \\
		\leq & (1 - \frac{1}{2}a_{\psi,i}\lambda^2\gamma\sigma^2 b_i^{-1})^2 (\psi_{2,i} - \psi_2^*)^2 + a_{\psi,i}^2\lambda^2 C \left( 1 + \lambda^2c_i^2 \left( 1 + c_i^2 + \lambda^2(\log\lambda)^2 + \lambda^2 + \lambda^2(\log b_i)^2\right)  +  \lambda^3c_i^3 \right) ,
	\end{aligned}\]
	where the last inequality holds when $a_{\psi,i}b_i^{-1} \leq a_{\psi,i}c_i \leq \frac{1}{\lambda^2\gamma\sigma^2} $. 
	
	The almost sure convergence of $\{ (\psi_{2,i} - \psi_2^*)^2\}_{i\geq 0}$ follows from \citet[Theorem 1]{robbins1971convergence} and by the same argument as in \citet[Theorem 2]{andradottir1995stochastic}, we can show $\psi_{2,i} - \psi_2^* \to 0$ almost surely. 
	
	Taking the expectation on both sides yield the recursive relation between two mean squared error:
	\begin{equation}
		\label{eq:recursive error psi 2}
		\begin{aligned}
			&	\epsilon_{2,i+1} \\
			\leq & (1 - \frac{1}{2}a_{\psi,i}b_i^{-1} \lambda^2\gamma\sigma^2 )^2 \epsilon_{2,i} + a_{\psi,i}^2\lambda^2 C \left(  1 + \lambda^2c_i^2 \left( 1 + c_i^2 + \lambda^2(\log\lambda)^2 + \lambda^2 + \lambda^2(\log c_i)^2\right)  +  \lambda^3c_i^3  \right) \\	
			\leq & (1 - \frac{1}{2}a_{\psi,i}b_i^{-1} \lambda^2\gamma\sigma^2 ) \epsilon_{2,i} + a_{\psi,i}^2\lambda^2 C \left( 1 + \lambda^2c_i^2 \left( 1 + c_i^2 + \lambda^2(\log\lambda)^2 + \lambda^2 + \lambda^2(\log b_i)^2\right)  +  \lambda^3c_i^3 \right),
		\end{aligned}
	\end{equation}
	where $\epsilon_{2,i} = \E\left[ (\psi_{2,i} - \psi_{2}^*)^2  \right]$.
	
	Finally, by the recursive relations \eqref{eq:recursive error psi 1} and \eqref{eq:recursive error psi 2}, and Lemma \ref{lemma:recursive mse}, we obtain that under suitable choices of sequences such that $a_{\psi,n} \sim \frac{1}{n}$, $b_n,c_n \sim \log n$, there are constants $C_1,C_2$ such that 
	\[\epsilon_{1,i+1} \leq C_1 a_{\psi,i}b_ic_i^2 = \tilde{O}(\frac{1}{n}),\  \epsilon_{2,i+1} \leq C_1 a_{\psi,i}b_i^2c_i^4 = \tilde{O}(\frac{1}{n}). \]
	
		\subsubsection{Stage 3: Impact of Discretization}
		In the above analysis, I ignored the error caused by discretization. Here, the errors caused by time discretization are twofold. First, continuously sampling i.i.d. actions $a(t)$ is practically infeasible, and also has measure-theoretical issues (see Remark 2.1 in \cite{szpruch2022optimal} for more details). Second, the integral can only be approximated by suitable finite sums in numerical computation. 
		
		Therefore, let us consider the dynamics of the wealth equation with discretized time grid: $0=t_0<t_1<\cdots<t_K=T$ such that for $t\in [t_{k-1}, t_{k}]$ with $k=1,\cdots,K$, 
		\[ \dd X(t) = [r + (\mu-r)a(t_{k-1})]X(t)\dd t + X(t)a(t_{k-1})\sigma\dd W(t) .\]
		For simplicity, here, the time grid is equally spaced, i.e., $\Delta t = \frac{T}{K}$ and $t_k = k\Delta t$.
		Note that this is the same wealth equation as \eqref{eq:state dynamics X} with a piecewise constant action sequence $a(t_0),a(t_1),\cdots,a(t_{K-1})$. On the interval $[t_{k-1},t_k]$, we can still apply It\^o's lemma to the TD error \eqref{eq:td}, i.e.,
		\[\begin{aligned}
			& V\left(t_k , X(t_k);\theta  \right) - V\left(t_{k-1} , X(t_{k-1});\theta  \right) - q\left( a(t_{k-1}) ;\bm\psi\right)\Delta t \\
			& + \frac{1-\gamma}{2} \left( V\left(t_k , X(t_k);\theta  \right) - V\left(t_{k-1} , X(t_{k-1});\theta  \right) \right)^2 \\
			= & -\theta\Delta t + \log X(t_k) - \log X(t_{k-1}) + \left[ \frac{(a(t_{k-1})-\psi_1)^2}{2\psi_2} + \frac{\lambda}{2}\log2\pi\lambda + \frac{\lambda}{2}\log\psi_2  \right]\Delta t \\
			& + \frac{1-\gamma}{2}\left(-\theta\Delta t + \log X(t_k) - \log X(t_{k-1}) \right)^2 \\
			= & \left[ -\theta + r + (\mu-r)a(t_{k-1}) - \frac{1}{2}\sigma^2 a(t_{k-1})^2  \right]\Delta t + \sigma a(t_{k-1})\left[ W(t_k) - W(t_{k-1})\right] \\
			& + \left[ \frac{(a(t_{k-1})-\psi_1)^2}{2\psi_2} + \frac{\lambda}{2}\log2\pi\lambda + \frac{\lambda}{2}\log\psi_2  \right]\Delta t \\
			& + \frac{1-\gamma}{2} \bigg\{ \left[ -\theta + r + (\mu-r)a(t_{k-1}) - \frac{1}{2}\sigma^2 a(t_{k-1})^2  \right]^2 (\Delta t)^2 + \sigma^2a(t_{k-1})^2 \left[ W(t_k) - W(t_{k-1})\right]^2 \\
			& \quad \quad \quad \quad + 2\left[ -\theta + r + (\mu-r)a(t_{k-1}) - \frac{1}{2}\sigma^2 a(t_{k-1})^2  \right]\sigma a(t_{k-1}) \left[ W(t_k) - W(t_{k-1})\right]\Delta t  \bigg\} \\
			=&: \delta(t_k)
		\end{aligned}   \]
		
		Moreover, the increments \eqref{eq:update theta}, \eqref{eq:update psi 1}, and \eqref{eq:update psi 2} in the algorithm will be approximated by
		\[ \frac{2}{T^2}\sum_{k=1}^K (T - t_{k-1})\delta(t_k),\ \frac{1}{T}\sum_{k=1}^K \frac{a(t_{k-1}) - \psi_1}{2\psi_2}\delta(t_k),\ -\frac{1}{T}\sum_{k=1}^K \left( -\frac{\left( a(t_{k-1}) - \psi_1 \right)^2}{2} + \frac{\lambda}{2}\psi_2 \right)\delta(t_k), \]
		respectively.
		
		Motivated by the calculation in Stage 1, let us mainly focus on the latter two terms. In particular, recall that $a(t_0),\cdots,a(t_{K-1}) \sim \mathcal N(\psi_1,\lambda\psi_2)$ i.i.d., hence,
		\[\begin{aligned}
			& \left| \E\left[  \frac{1}{T}\sum_{k=1}^K \frac{a(t_{k-1}) - \psi_1}{2\psi_2}\delta(t_k) \right] - h_{\psi,1}(\psi_1,\psi_2,\theta)\right| \\
			= &\Bigg| \E\bigg[  \frac{1}{T}\sum_{k=1}^K \frac{a(t_{k-1}) - \psi_1}{2\psi_2} \Big[ (\mu-r-\psi_1\sigma^2)(a(t_{k-1}) - \psi_1)\Delta t \\
			& \quad \quad \quad \quad + \frac{1-\gamma}{2}\left( \cdots (\Delta t)^2 + 2\sigma^2 \psi_1 (a(t_{k-1}) - \psi_1)  \Delta t  \right)  \Big]   \bigg]  - h_{\psi,1}(\psi_1,\psi_2,\theta)\Bigg| \\
			= & \frac{|1-\gamma|}{4} \lambda \left( \mu-r-\psi_1\sigma^2  \right) \left[2 \left( -\theta+r+(\mu-r)\psi_1 - \frac{1}{2}\sigma^2\psi_1^2\right)  + 3\lambda\sigma^2\psi_2 \right] \Delta t  ,
		\end{aligned}  \]
		and
		\[\begin{aligned}
			& \left| -\E\left[  \frac{1}{T}\sum_{k=1}^K \left( -\frac{\left( a(t_{k-1}) - \psi_1 \right)^2}{2} + \frac{\lambda}{2}\psi_2 \right)\delta(t_k) \right] - h_{\psi,2}(\psi_1,\psi_2,\theta)\right| \\
			= &\Bigg| -\E\bigg[  \frac{1}{T}\sum_{k=1}^K \left( -\frac{\left( a(t_{k-1}) - \psi_1 \right)^2}{2} + \frac{\lambda}{2}\psi_2 \right) \Big[ -\frac{1}{2}\sigma^2 \left( a(t_{k-1}) - \psi_1 \right)^2\Delta t + \frac{(a(t_{k-1})-\psi_1)^2}{2\psi_2} \Delta t \\
			& \quad \quad \quad \quad + \frac{1-\gamma}{2}\left( \cdots (\Delta t)^2 + \sigma^2 \left( a(t_{k-1}) - \psi_1 \right)^2  \Delta t\right)  \Big]  \bigg]   - h_{\psi,2}(\psi_1,\psi_2,\theta) \Bigg| \\
			= & \frac{|1-\gamma|}{2}\sigma^2\lambda^2\psi_2^2 \Delta t .
		\end{aligned}  \]
		
		Moreover, we may find an upper bound for the variance of these terms. Note that across each time interval $[t_{k-1},t_k]$, the pair $(a(t_{k-1}), \delta(t_k))$ are identical and independently distributed, therefore,
		\[ \begin{aligned}
			& \operatorname{Var}\left[  \frac{1}{T}\sum_{k=1}^K \frac{a(t_{k-1}) - \psi_1}{2\psi_2}\delta(t_k) \right]  = \frac{K}{T^2} \operatorname{Var}\left[ \frac{a(t_{k-1}) - \psi_1}{2\psi_2}\delta(t_k)  \right] \\
			\leq & \frac{1}{4 T \psi_2^2 \Delta t} \E\left[ \left( a(t_{k-1}) - \psi_1 \right)^2  \delta(t_k)^2  \right] \leq \frac{\lambda \sqrt{3}}{4 T \psi_2 \Delta t} \left( \E\left[ \delta(t_k) ^4 \right]  \right)^{1/2} .
		\end{aligned} \]
		Similarly,
		\[ \begin{aligned}
			& \operatorname{Var}\left[  \frac{1}{T}\sum_{k=1}^K  \left( -\frac{\left( a(t_{k-1}) - \psi_1 \right)^2}{2} + \frac{\lambda}{2}\psi_2 \right) \delta(t_k) \right]  = \frac{K}{T^2} \operatorname{Var}\left[ \left( -\frac{\left( a(t_{k-1}) - \psi_1 \right)^2}{2} + \frac{\lambda}{2}\psi_2 \right)\delta(t_k)  \right] \\
			\leq & \frac{1}{4T\Delta t} \E\left[ \left( \left( a(t_{k-1}) - \psi_1 \right)^2 - \lambda\psi_2 \right)^2 \delta(t_k)^2 \right] \leq \frac{\lambda^2 \psi_2^2 \sqrt{15}}{2 T \Delta t}  \left( \E\left[ \delta(t_k) ^4 \right]  \right)^{1/2} .
		\end{aligned} \]
		
		By the definition of $\delta(t_k)$, we may have
		\[ \begin{aligned}
			& \E\left[ \delta(t_k)^4  \right] \\
			\leq & C \E\Bigg[\left(1 + |\theta|^4 + a(t_{k-1})^8  + (\lambda\log\lambda)^4 + (\psi_2^{-1} + \log \psi_2)^4 + \psi_1^8  \right)  (\Delta t)^2 + \sigma^4 a(t_{k-1})^4( W(t_k) - W(t_{k-1}))^4  \\
			&\quad \quad  +  \left(1 + |\theta|^8 + a(t_{k-1})^{16}  \right) (\Delta t)^8 + \sigma^8 a(t_{k-1})^8 ( W(t_k) - W(t_{k-1}))^8 \\
			& \quad \quad + (1 + |\theta|^4 + a(t_{k-1})^8) a(t_{k-1})^4 (\Delta t)^4  ( W(t_k) - W(t_{k-1}))^4  \Bigg] \\
			\leq & C \E\Bigg[\left(1 + |\theta|^4 + a(t_{k-1})^8 + (\lambda\log\lambda)^4 + (\psi_2^{-1} + \log \psi_2)^4 + \psi_1^8  \right)  (\Delta t)^2  +  \left(1 + |\theta|^8 + a(t_{k-1})^{16}  \right) (\Delta t)^8\\
			& \quad \quad   + (1 + |\theta|^4 + a(t_{k-1})^8) a(t_{k-1})^4 (\Delta t)^6  \Bigg] \\
			\leq &  C\left( 1 + |\theta|^4 + (\lambda\log\lambda)^4 + (\psi_2^{-1} + \log \psi_2)^4 + \psi_1^8 + \lambda^4\psi_2^4 \right)(\Delta t)^2 \\
			& + C\left( 1 +|\theta|^8 +  \psi_1^{16} + \lambda^8\psi_2^{8} \right) (\Delta t)^8 + C\left( 1 +|\theta|^4 +  \psi_1^{8} + \lambda^4\psi_2^{4} \right) \lambda^2\psi_2^2 (\Delta t)^6
		\end{aligned} \]
		where $C$ is a constant that only depends on model primitives $\mu,r,\sigma,\gamma$.
		
		From the above calculation, the recursive relation of the sequence in the learning procedure can still be written as the same as \eqref{eq:projection recusive}, except that the ``error" term $\xi_{1,i+1}$ and $\xi_{2,i+1}$ satisfy
		\[\begin{aligned}
			\beta_{1,i} := & \left| \E\left[\xi_{1,i+1}| \theta_i,\bm\psi_i  \right]  \right| \leq C\lambda \left( 1 + |\psi_{1,i}|^3 + 3\lambda \psi_{2,i} + 3\lambda |\psi_{1,i}|\psi_{2,i}  \right) \Delta t_i \\
			\beta_{2,i} := & \left| \E\left[\xi_{2,i+1}| \theta_i,\bm\psi_i  \right]  \right| \leq C\lambda^2 \psi_{2,i}^2 \Delta t_i \\
			\zeta_{1,i} := & \operatorname{Var}\left[ \xi_{1,i+1}| \theta_i,\bm\psi_i  \right]\leq C\lambda \psi_{2,i}^{-1} \sqrt{d_{1,i} + d_{2,i}(\Delta t_i)^4 + d_{3,i} (\Delta t_i)^6} \\
			\zeta_{2,i} := & \operatorname{Var}\left[ \xi_{2,i+1}| \theta_i,\bm\psi_i  \right]\leq  C \lambda^2 \psi_{2,i}^2\sqrt{d_{1,i} + d_{2,i}(\Delta t_i)^4 + d_{3,i} (\Delta t_i)^6} ,
		\end{aligned}  \]
		where
		\[ \begin{aligned}
			& d_{1,i} = 1+ (\lambda\log\lambda)^4 + b_i^4 + c_i^8 + \lambda^4 c_i^4 \\
			& d_{2,i} = 1+c_i^{16} +\lambda^8 c_i^8 \\
			& d_{3,i} = \lambda c_i^2(1+c_i^8 + \lambda c_i^4) .
		\end{aligned} \]
		
		Following the parallel calculation in the Stage 2, for $i\geq I$ that is sufficiently large, we have
		\[ \begin{aligned}
			& \E\left[ \left( \psi_{1,i+1} - \psi_1^*\right)^2  | \theta_i,\bm\psi_i\right] \\
			\leq &  \E\left[ \left(\psi_{1,i} + a_{\psi,i} \left( h_{\psi,1}(\psi_{1,i},\psi_{2,i},\theta_i) + \xi_{1,i+1} \right) - \psi_1^* \right)^2  | \theta_i,\bm\psi_i\right] \\
			\leq & (\psi_{1,i} - \psi_1^*)^2 + 2a_{\psi,i}(\psi_{1,i} - \psi_1^*)\left[ h_{\psi,1}(\psi_{1,i},\psi_{2,i},\theta_i) + \beta_{1,i}\right]  + a_{\psi,i}^2  h_{\psi,1}(\psi_{1,i},\psi_{2,i},\theta_i)^2 + a_{\psi,i}^2 (\beta_{1,i}^2 + \zeta_{1,i}) \\
			& + 2a_{\psi,i}^2 h_{\psi,1}(\psi_{1,i},\psi_{2,i},\theta_i) \beta_{1,i} \\
			= & (1 - a_{\psi,i}\lambda\gamma\sigma^2 ) (\psi_{1,i} - \psi_1^*)^2 + \frac{1}{4}a_{\psi,i}^2\lambda^2\gamma^2\sigma^4(\psi_{1,i} - \psi_1^*)^2 +  a_{\psi,i}^2(\beta_{1,i}^2 + \zeta_{1,i}) \\
			& + a_{\psi,i} (1 - \frac{\lambda\gamma\sigma^2}{2} a_{\psi,i}) \left[ (\psi_{1,i} - \psi_1^*)^2 + 1 \right] \beta_{1,i} \\
			\leq & (1 - \frac{1}{2}a_{\psi,i}\lambda\gamma\sigma^2 ) (1 - \frac{1}{2}a_{\psi,i}\lambda\gamma\sigma^2 + a_{\psi,i} \beta_{1,i}) (\psi_{1,i} - \psi_1^*)^2 +  a_{\psi,i} (1 - \frac{\lambda\gamma\sigma^2}{2} a_{\psi,i})  \beta_{1,i} +  a_{\psi,i}^2(\beta_{1,i}^2 + \zeta_{1,i}) .
		\end{aligned}\]
		
		Without loss of generality, we may take $I$ large enough such that $1 - \frac{1}{2}a_{\psi,i}\lambda\sigma^2 + a_{\psi,i}\beta_{1,i} \in (0,1)$, and $1 - \frac{1}{2}a_{\psi,i}\lambda\sigma^2\in (0,1)$. Moreover, given our condition, $a_{\psi,n} \sim \frac{1}{n}$, $b_n,c_n \sim \log n$, and $\Delta t_i\sim \frac{1}{n}$, we can also ensure $d_{2,i}(\Delta t_i)^4 \leq d_{1,i} $, $d_{3,i}(\Delta t_i)^6\leq d_{1,i}$ for all $i\geq I$. Then, we obtain
		\[ \begin{aligned}
			& \E\left[ \left( \psi_{1,i+1} - \psi_1^*\right)^2  | \theta_i,\bm\psi_i\right] \\
			\leq & (1 - \frac{1}{2}a_{\psi,i}\lambda\sigma^2) (\psi_{1,i} - \psi_1^*)^2 + a_{\psi,i}\beta_{1,i} + a_{\psi,i}^2(\beta_{1,i}^2 + \zeta_{1,i}) \\
			\leq & (1 - \frac{1}{2}a_{\psi,i}\lambda\sigma^2) (\psi_{1,i} - \psi_1^*)^2 + a_{\psi,i}C\lambda \left(1 +  c_i^3 + 3\lambda c_i^2 \right)\Delta t_i \\
			& + a_{\psi,i}^2 C\left[ \lambda^2 \left(1 + c_i^3 + 3\lambda c_i^2 \right)^2 (\Delta t_i)^2 + \lambda b_i \sqrt{d_{1,i}} \right] \\
			\leq & (1 - \frac{1}{2}a_{\psi,i}\lambda\sigma^2) (\psi_{1,i} - \psi_1^*)^2 + a_{\psi,i}^2 C \lambda \left[ (1+c_i^3 + 3\lambda c_i^2)\frac{\Delta t_i}{a_{\psi,i}} + \lambda \left(1 + c_i^3 + 3\lambda c_i^2 \right)^2 (\Delta t_i)^2 + b_i \sqrt{d_{1,i}} \right] \\
			\leq &  (1 - \frac{1}{2}a_{\psi,i}\lambda\sigma^2) (\psi_{1,i} - \psi_1^*)^2 + a_{\psi,i}^2 C \lambda (\log i)^5 .
		\end{aligned} \]
		
		By the similar argument in the end of Stage 2 and Lemma \ref{lemma:recursive mse}, our desired conclusion still holds.

	\subsection{Proof of Theorem \ref{thm:regret merton}}
	From the definition of ERWL, we can solve \[\Delta(\psi_1,\lambda \psi_2) = 1 - \exp\left\{ - \left[ J^*(x) -  J(x,\psi_1,\lambda \psi_2) \right] \right\} \leq J^*(x) -  J(x,\psi_1,\lambda \psi_2) , \]
	where the inequality is due to the fact that $e^{-x} \geq 1 - x$. Moreover, from the expression \eqref{eq:risk sensitive objective merton explicit}, we have
	\[J^*(x) -  J(x,\psi_1,\lambda \psi_2) = \frac{\gamma \sigma^2 T}{2}\left[ (\psi_1 - \psi_1^*)^2 + \lambda\psi_2\right] .  \]
	\subsubsection{Executing deterministic policies}
	In this case, our choice of tuning sequence $\{a_{\psi,i}\}_{i\geq 0}$, $\{b_i\}_{i\geq 0}$, $\{c_i\}_{i\geq 0}$, and $\lambda_i = \lambda$ is identical to that in Theorem \ref{thm:convergence of merton}. Hence, we have 
	\[ \begin{aligned}
		\E\left[ \sum_{i=1}^N  \Delta(\psi_{1,i},0) \right] \leq & \E\left[ \sum_{i=1}^N  \left(  J^*(x) -  J(x,\psi_{1,i},0) \right)\right] \\
		= & \frac{\gamma \sigma^2 T}{2}\sum_{i=1}^N \E[(\psi_{1,i} - \psi_1^*)^2] \leq  \frac{C_1 \gamma \sigma^2 T}{2}\sum_{i=1}^N \frac{(\log i)^3}{i} = O\left( (\log N)^4 \right) .
	\end{aligned} \]
	\subsubsection{Executing stochastic policies}
	Suppose in the $i$-th episode, we use $\lambda_i$ in \eqref{eq:parameterization form}. The derivation about the recursive relation in the proof of Theorem \ref{thm:convergence of merton} is still applicable by replacing constant $\lambda$ by $\lambda_i$. Hence, with the choice of tuning sequences $a_{\psi,n} \sim \frac{1}{\sqrt{n}}$, $b_n = O(1)$, $c_n \sim \log n$, $\lambda_n \sim \frac{1}{\sqrt{n}}$, \eqref{eq:recursive error psi 1} becomes
	\[\begin{aligned}
		\epsilon_{1,i+1} \leq & (1 - \frac{1}{2}a_{\psi,i}\lambda_i\gamma\sigma^2 ) \epsilon_{1,i} + a_{\psi,i}^2\lambda_i C \left( 1 + \lambda_i^2 + \lambda_i^3 c_i + b_i(1 + c_i^2 + \lambda_i^2(\log\lambda_i)^2 + \lambda_i^2 + \lambda_i^2 (\log c_i)^2)  \right) \\
		= & (1 - \frac{\gamma\sigma^2}{2}a_{\psi,i}\lambda_i ) \epsilon_{1,i} + a_{\psi,i}^2 \lambda_i C (1 + c_i^2) .
	\end{aligned}  \]
	The convergence of $\epsilon_{1,i}$ is guaranteed since $\sum  a_{\psi,i}\lambda_i = \sum \frac{1}{i} = \infty$, and $\sum a_{\psi,i}^2 \lambda_i c_i^2 = \sum \frac{\log i}{i^{3/2} } < \infty $. Moreover, by Lemma \ref{lemma:recursive mse}, we obtain $\epsilon_{1,i+1}  = O\left( \frac{(\log i)^2}{\sqrt{i}}\right)$.
	
	Hence, 
	\[ \begin{aligned}
		\E\left[ \sum_{i=1}^N  \Delta(\psi_{1,i},\lambda_i \psi_{2,i}) \right] \leq & \E\left[ \sum_{i=1}^N  \left(  J^*(x) -  J(x,\psi_{1,i},\lambda_i \psi_{2,i}) \right)\right] \\
		\leq & \frac{\gamma \sigma^2 T}{2}\sum_{i=1}^N \E[(\psi_{1,i} - \psi_1^*)^2] + \frac{\gamma \sigma^2 T}{2}\sum_{i=1}^N \lambda_i c_i \\
		\leq & \frac{C \gamma \sigma^2 T}{2}\sum_{i=1}^N \left( \frac{(\log i)^2 }{\sqrt{i} } + \frac{\log i}{\sqrt{i} } \right)  = O(\sqrt{N} (\log N)^2).
	\end{aligned} \]
	
	\subsubsection{Impact of Discretization}
	From the Stage 3 in the proof of Theorem \ref{thm:convergence of merton}, we can see the same argument also applies in our previous analysis and will only affect the logarithmic factor. Therefore, the desired conclusion still holds.

\end{document}